\documentclass{article}
\usepackage{microtype}
\usepackage{graphicx}
\usepackage{subfigure}
\usepackage{booktabs} % for professional tables
\usepackage{hyperref}

\usepackage[accepted]{icml2024}
\usepackage{amsmath}
\usepackage{amssymb}
\usepackage{mathtools}
\usepackage{amsthm}
\usepackage[capitalize,noabbrev]{cleveref}
\pdfoutput=1
\theoremstyle{plain}
\newtheorem{theorem}{Theorem}[section]
\newtheorem{proposition}[theorem]{Proposition}
\newtheorem{lemma}[theorem]{Lemma}

\theoremstyle{definition}
\newtheorem{definition}[theorem]{Definition}
\newtheorem{assumption}[theorem]{Assumption}
\theoremstyle{remark}

\usepackage[textsize=tiny]{todonotes}
\usepackage{amsmath,amsfonts,bm}

\def\eqref#1{equation~\ref{#1}}
\def\1{\bm{1}}

\def\vk{{\bm{k}}}

\def\vm{{\bm{m}}}

\def\vq{{\bm{q}}}

\def\vu{{\bm{u}}}
\def\vv{{\bm{v}}}

\DeclareMathAlphabet{\mathsfit}{\encodingdefault}{\sfdefault}{m}{sl}
\SetMathAlphabet{\mathsfit}{bold}{\encodingdefault}{\sfdefault}{bx}{n}

\DeclareMathOperator*{\argmin}{arg\,min}

\usepackage{url}
\usepackage[utf8]{inputenc}
\usepackage{amsmath,dsfont}
\usepackage{cite}
\usepackage{amsfonts} 
\usepackage{mathtools , amssymb , amsthm } % imports amsmath
\usepackage{amsmath}
\usepackage{graphicx}
\usepackage{booktabs}
\usepackage{caption}
\usepackage{array}
\usepackage{pifont}
\usepackage{algorithm}
\usepackage{algorithmic}
\usepackage{custom}
\usepackage{multirow}
\usepackage{bbm}
\usepackage{bm}
\newcommand{\bP}{\mathbb{P}}

\renewcommand{\1}{\mathbbm{1}}

\renewcommand{\algorithmiccomment}[1]{\hfill\textcolor{black}{\\ $\triangleright$ \textit{#1}}}

\begin{document}

\twocolumn[
\icmltitle{Distribution-Free Fair Federated Learning with Small Samples}

\icmlsetsymbol{equal}{*}
\begin{icmlauthorlist}
\icmlauthor{Qichuan Yin}{equal,rm1,rm3}
\icmlauthor{Zexian Wang}{rm4}
\icmlauthor{Junzhou Huang}{rm2}
\icmlauthor{Huaxiu Yao}{rm1}
\icmlauthor{Linjun Zhang}{rm3}
\end{icmlauthorlist}
\icmlaffiliation{rm1}{UNC-Chapel Hill}
\icmlaffiliation{rm2}{University of Texas at Arlington}
\icmlaffiliation{rm3}{Rutgers University}
\icmlaffiliation{rm4}{University of Michigan, Ann Arbor}

\icmlcorrespondingauthor{Linjun Zhang}{linjun.zhang@rutgers.edu}
\icmlkeywords{Machine Learning, ICML}

\vskip 0.3in
]

\printAffiliationsAndNotice{\icmlEqualContribution} % otherwise use the standard text.

\begin{abstract}
As federated learning gains increasing importance in real-world applications due to its capacity for decentralized data training, addressing fairness concerns across demographic groups becomes critically important. However, most existing machine learning algorithms for ensuring fairness are designed for centralized data environments and generally require large-sample and distributional assumptions, underscoring the urgent need for fairness techniques adapted for decentralized and heterogeneous systems with finite-sample and distribution-free guarantees. To address this issue, this paper introduces FedFaiREE, a post-processing algorithm developed specifically for distribution-free fair learning in decentralized settings with small samples. Our approach accounts for unique challenges in decentralized environments, such as client heterogeneity, communication costs, and small sample sizes. We provide rigorous theoretical guarantees for both fairness and accuracy, and our experimental results further provide robust empirical validation for our proposed method.
\end{abstract}

\section{Introduction}
Federated learning (FL) is a machine learning technique that harnesses data from multiple clients to enhance performance. Notably, it accomplishes this without the need to centralize all the data on a single server \citep{Fedavg}. With the growing integration of FL in practical applications, \textit{fairness} is gaining prominence, especially in domains like healthcare \citep{healthcare,healthcare2} and smartphone technology \citep{smartphone,smartphone2}. 
However, applying existing fairness methods directly can be challenging, primarily because many of these methods were originally designed within a centralized framework. This can lead to poor performance or high communication costs when implementing them in real-world scenarios. 

To tackle the fairness challenges in the context of federated learning, recent research has introduced several techniques, including FairFed \citep{FairFed}, FedFB \citep{FedFB}, FCFL \citep{FCFL}, and AgnosticFair \citep{AgnosticFair}. These methods aim to enhance fairness by implementing debiasing at the local client level and fine-tuning aggregation weights on the server. However, despite their promise, these approaches face certain challenges. Firstly, as highlighted by \citet{tradeoff}, achieving global fairness by solely ensuring local fairness can prove elusive. In other words, ensuring fairness for all clients individually may not necessarily result in overall fairness across the federated system. Secondly, many existing methods assume an ideal scenario of infinite samples or struggle to guarantee fairness constraints in a \textit{distribution-free} manner, that is, without making any distributional assumptions. These drawbacks limit the wide use of the existing methods in real-world applications.
%, where practical constraints like privacy considerations may restrict access to only finite datasets. 
For example, when developing decision models across multiple hospitals or medical institutions, stringent privacy regulations and data access limitations often mean that only limited data can be utilized.

To address these concerns, this paper introduces FedFaiREE, a post-processing algorithm to achieve finite-sample and distribution-free fairness in federated learning. FedFaiREE provides a flexible framework that accounts for unique challenges presented by decentralized settings, including communication costs, client heterogeneity, client correlation, and small sample sizes. The core concept behind FedFaiREE involves the distributed utilization of order statistics to conform to fairness constraints and the selection of the classifier with the best accuracy among classifiers that meet the fairness constraints. 

\begin{figure}
    \centering
    \includegraphics[width=0.4\textwidth]{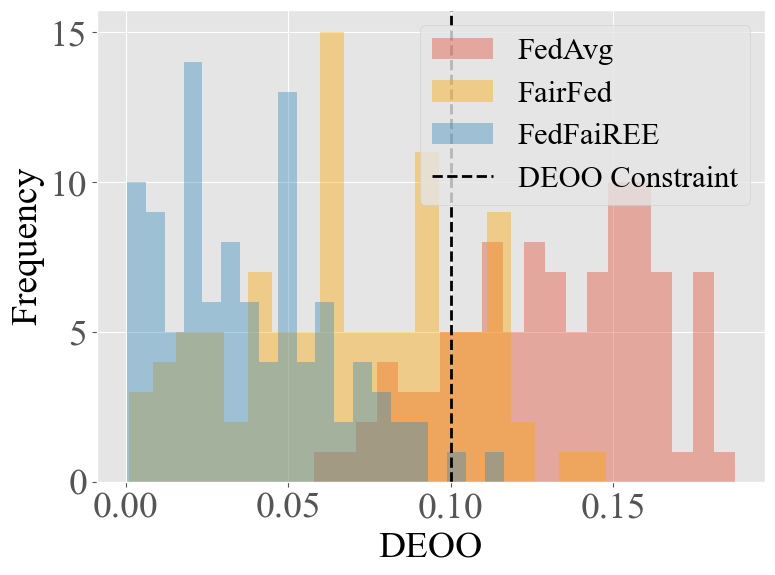}
    \caption{The distribution of $|DEOO|$ (a fairness metric, defined in Equation \ref{deoo}) for FedAvg \citep{Fedavg} and FairFed \citep{FairFed}, both with and without FedFaiREE, evaluated on the Adult dataset \citep{dua2017uci}. See Section \ref{sec:experiment} for experiment details.}\label{fig:teaser}
\end{figure}
 %add more details. difference
Our primary contributions are three-fold: first, we introduce FedFaiREE, a simple yet highly effective approach to ensuring fairness constraints in scenarios with limited samples without any distributional assumptions; second, we provide theoretical guarantees that our method can achieve nearly optimal accuracy under fairness constraints when the input prediction function is suitable; third, empirically, as demonstrated in Figure \ref{fig:teaser}, we applied existing methods like FairFed \citep{FairFed} and FedAvg \citep{Fedavg} with and without FedFaiREE to the Adult dataset \citep{dua2017uci}. We found that existing algorithms are unable to effectively control fairness in real-world applications due to the small sample size in each client, while FedFaiREE shows promising performance by strictly satisfying the pre-defined fairness requirement.

\subsection{Additional Related Work}
Existing fairness methodologies in federated learning predominantly address two key aspects of fairness: fairness among clients and fairness among groups. 
%\linjun{let's move the discussion of client fairness to the related work section, and only mention the group fairness here} 
The former aspect aims to ensure that the global model's performance across individual clients is equitable in terms of equality or contribution \citep{ditto,Lyu2020,yu2020fairness,huang2020efficiency}. In contrast, our primary focus in this paper revolves around the latter facet — fairness among groups \citep{Group}, also referred to as \textit{group fairness}, where the objective is to ensure equitable treatment across different sensitive labels, such as race and gender. 

\textbf{Existing Group Fairness Techniques.} Conventional approaches can be approximately divided into three categories \citep{groupfairness}: pre-processing methods that directly perform debiasing on input data \citep{zemel2013learning,johndrow2019algorithm}; in-processing methods that incorporate fairness metrics into model training as part of the objective function \citep{goh2016satisfying, cho2020fair}; post-processing methods that adjust model outputs to enhance fairness \citep{fairee,zeng2022bayes,fish2016confidence}. FaiREE \citep{fairee} is the first approach in the literature that achieves group fairness in a finite-sample and distribution-free manner. However, FaiREE is restricted to handling i.i.d. centralized data, while our proposed method is designed to address the challenges presented by decentralized settings, such as communication costs associated with updating local data and client heterogeneity. Under the setting of client heterogeneity, even if all training data are centralized, FaiREE will still encounter bias due to variations among different clients. In addition, our proposed method allows client correlation, while FaiREE requires independence among training samples. See a more detailed discussion in Section~\ref{sec:compare} of the Appendix. %Finally, client correlation may also violate the assumption of FaiREE. In the context of training decision models across multiple hospitals, it's common for a patient to visit several hospitals, which can influence the assumption of data independence.

\textbf{Group Fairness Approaches in Federated Learning.}
In recent years, there has been a growing amount of work focusing on group fairness in the context of Federated Learning \citep{FairFed,FCFL,FedFB,AgnosticFair,rodriguez2021enforcing, chu2021FedFair, liang2020think,hu2022fair,papadaki2022minimax}. Most of these studies aim to either introduce fairness principles into the local updates, adapt conventional fairness methods, or perform reweighting during aggregation, or a combination of these strategies. Specifically, \citet{AgnosticFair} proposed AgnosticFair, a framework that utilizes kernel reweighing functions to adjust items in local objective functions, including both loss terms and fairness constraints. \citet{FedFB} introduced FedFB, a method that adapts Fair Batch, a centralized technique designed to improve fairness among groups by reweighting loss terms for different subgroups, for the FL setting. \citet{FairFed} proposed FairFed, an approach that adjusts aggregate weights by considering the disparities between local fairness metrics and the global fairness metric in each training round.
\section{Preliminaries}
%\subsection{Fairness notions and Notations}
In this paper, we address the problem of predicting a binary label, denoted by $Y$, using a set of features. The features are divided into two categories: $X$ and $A$. Here, $X \in \mathcal{X}$ represents non-sensitive features, while $A \in \mathcal{A} = \{0, 1,\cdots, A_0\}$ corresponds to sensitive features. A data point includes $(x, y, a)$, which corresponds to $(X,Y,A)$. For simplicity, we first introduce the concept of \textit{Score-based classifier} \citep{chen2018my,zafar2019fairness}. %{\color{cyan} change the citation; this definition is used in other papers}
\begin{definition}(Score-based classifier)
        A score-based classifier is an indication function $\hat Y=\phi(x,a) = \1\{f(x,a) > c\}$ for a measurable score function $f: \mathcal{X} \times \mathcal{A} \rightarrow[0,1]$ and a constant threshold $c>0$. 
\end{definition}
To assess the fairness of the classifier, we introduce a fairness notion, Equality of Opportunity, which has been extensively utilized in the fairness literature.

\begin{definition}(Equality of Opportunity\citep{hardt2016equality})
A classifier satisfies Equality of Opportunity if it satisfies the same true positive rate among protected groups:
$\bP_{X \mid A=a, Y=1}(\widehat{Y}=1)=\bP_{X \mid A=0, Y=1}(\widehat{Y}=1),$
where $a \in \{1, \cdots, A_0\}$.
\label{def-eq}
\end{definition}

Equality of Opportunity focuses on ensuring an equal opportunity to be predicted as a true positive across different groups. However, in practice, achieving strict Equality of Opportunity is often too hard. Therefore, a tolerance parameter, denoted as $\alpha$, is commonly introduced in Equality of Opportunity, as discussed in prior works \citep{zeng2022bayes, fairee}. To be more specific, given a classifier $\phi$, the $\alpha$ difference tolerance in Equality of Opportunity within a binary group label can be defined as:
\begin{equation}
|\bP_{X \mid A=1, Y=1}(\widehat{Y}=1)-\bP_{X \mid A=0, Y=1}(\widehat{Y}=1)|\leq \alpha.
\end{equation}
To be concise, in later sections, we use $DEOO$ to represent the left side of the inequality, i.e.,
\begin{equation}
DEOO=\bP_{X \mid A=1, Y=1}(\widehat{Y}=1)-\bP_{X \mid A=0, Y=1}(\widehat{Y}=1).
\label{deoo}
\end{equation}
We are going to talk about the multi-group,multi-label vision of Equality of Opportunity tolerance in later sections.

\textbf{Notation.} To further simplify the formula in the article, we provide notations as follows:
% $D, D_i$ represent the dataset in all clients and in client $i$, respectively. The index $i$ belongs to the set ${1, 2, \ldots, S}$.
% $n$ signifies the size of dataset $D$.
% $T$ represents the scores of elements in dataset $D$, arranged in non-decreasing order, i.e. $T=\{t_{(1)},t_{(2)},\ldots,t_{(n)}\}$.
% %$k$ represents the selected score rank used as the classifier threshold for predicting true.
% $D_i^{y,a}$ denotes the subset of dataset $D_i$ that satisfies $Y=y$ and $A=a$. Similar explanations apply to $T^{y,a}$, $n^{y,a}$.
% %$\pi_i^{1,a}$ denotes the probability of data point x from Client i given that x with label $Y=1$ and attribute $A=a$.
% %$\alpha$ represents the constraint or limit of the fairness indicator.
% $p_a$ signifies the probability of the sensitive attribute $A=a$, i.e., $P(A=a)$.
% $p_{Y,a}$ represents the probability of label $Y=1$ given the sensitive attribute $A=a$, i.e., $P(Y=1 \mid A=a)$ and $q_{Y,a}=1-p_{Y,a}$.
% $[S]$ denotes the set of integers from $1$ to $S$, i.e., $\{1,2,\cdots,S\}$.
% $\Delta_S$ represents the set of $S$-dimensional vectors $\vv=(v_1,\cdots,v_S)$ that satisfy the conditions $v_i \geq 0$ and $\sum_{i=1}^S v_i=1$.
$p_a$ signifies the probability of the sensitive attribute $A=a$, i.e., $P(A=a)$.
$p_{Y,a}$ represents the probability of label $Y=1$ given the sensitive attribute $A=a$, i.e., $P(Y=1 \mid A=a)$, and $q_{Y,a}$ is defined as $1-p_{Y,a}$.
$D$ and $D_i$ represent the datasets for all clients and client $i$, respectively, where $i$ belongs to the set $\{1, 2, \ldots, S\}$.
$n$ denotes the size of dataset $D$.
$T$ represents the ordered scores of elements in dataset $D$.
$D_i^{y,a}$ is used to denote the subset of dataset $D_i$ where $Y=y$ and $A=a$. Similar notations apply to $T^{y,a}$ and $n^{y,a}$.
% $[S]$ denotes the set of integers from $1$ to $S$.
% $\Delta_S$ represents the set of $S$-dimensional vectors $\vv=(v_1, v_2, \ldots, v_S)$ satisfying the conditions $v_i \geq 0$ and $\sum_{i=1}^S v_i=1$.

\begin{figure*}[ht]
    \centering
    \includegraphics[width=0.8\textwidth]{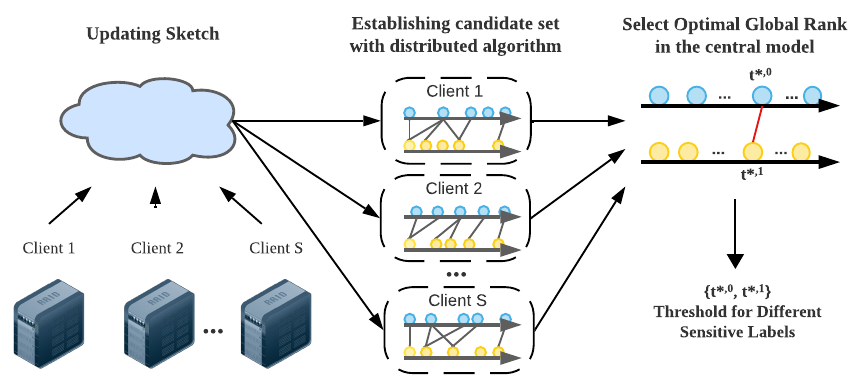}
    \caption{\textbf{Overview of FedFaiREE. } With S clients and a pre-trained model in consideration, each circle in the image symbolizes a datapoint score in the training set. The color of the circles represents different sensitive labels, while the gray edges depict local ranks of threshold pairs (each global classifier's threshold pair corresponds to S local ranks). Notably, the red edge signifies the chosen global classifier with thresholds ${t^{*,0}, t^{*,1}}$ for sensitive labels $A=0$ and $A=1$, respectively.} %\linjun{${t_0, t_1}$?}
    \label{fig:overview}

\end{figure*}

\section{Enabling Fair Federated Learning}\label{sec:meth}
In this section, we introduce FedFaiREE, a \textbf{Fed}erated Learning, \textbf{Fai}r, distribution-f\textbf{REE} algorithm. FedFaiREE has the capability to ensure fairness in scenarios involving finite samples, distribution-free cases, and heterogeneity among clients. To incorporate heterogeneity among clients into our model, we make the following assumption.

\begin{assumption}\label{ass:mix}
The training data points within the client $i$ are drawn independently and identically (i.i.d) from distribution $P_i$, while the test data points are sampled from a global distribution that represents a mixture of $P_1,\cdots, P_S$ with weight $\{ \pi_i \}_{i \in [S]} \in \Delta_S$. Specifically, we assume that
$$\small
%\begin{matrix}
\left(X_{k}^{i}, Y_{k}^{i}\right) \sim P_i\text{, }\left(X^{\text{test}},Y^{\text{test}}\right) \sim P^{mix}= \sum_{i=1}^S \pi_i P_i.
%\end{matrix}
$$
\end{assumption}
%\linjun{add explanations of this assumption}
This implies that each client $i$ has its own distribution $P_i$, and test data points are randomly sampled from client $i$ with a probability of $\pi_i$.
%Similarly, we define $\pi_i^{1,a}$ denotes the probability of data point x from Client i given that x with label $Y=1$ and attribute $A=a$.

\subsection{Problem formulation}
Consider a scenario with $S$ clients, each equipped with a locally available dataset $D_i = \cup_{y \in \mathcal{Y}, a\in \mathcal{A} }D_i^{y,a} $ and a pre-trained score-based classifier $\phi_0(x,a)=\1 \{ f(x,a) > c \}$. Here, $i \in [S]$, representing each client, and $D_i^{y,a}$ denotes a subset of data points in $D_i$ with labels $Y=y$ and sensitive attributes $A=a$. Considering certain fairness constraint $|DEOO|<\alpha$, we aim to determine optimal thresholds $\lambda_0$ and $\lambda_1$ for constructing the output classifier $\phi(x,a) = \1 \{ f(x,a) > \lambda_a \}$. 

Our inspiration stems from \citet{zeng2022bayes}, highlighting that the classifier with optimal misclassification performance while adhering to specific fairness constraints requires different thresholds for different groups. 
%Drawing further inspiration from FaiREE \citep{fairee}, 
Furthermore,
we extend our consideration to scores $t_{i,j}^{y,a}=f(x_{i,j}^{y,a})$ and $T^{y,a}$, where $T_i^{y,a}=\{t_{i,1}^{y,a},t_{i,2}^{y,a}, \cdots, t_{i,n_i^{y,a}}^{y,a}\}$ represents the corresponding sorted score set. If we limit the problem on client i, this naturally leads us to the idea of transforming the problem of selecting optimal thresholds $\lambda_a$ into determining the optimal ``local ranks" (i.e. ranks on the client) of the score $k_{i}^{1,a}$. However, as we concern about global fairness and misclassification error, we opt to seek the global rank $k^{1,a}$ (i.e., the rank in the sorted score set $T^{1,a}$ consisting of all client scores with $Y=1$ and $A=a$  $t^{1,a}$), and $\phi(x,a) = \1 \{ f(x,a) > t_{(k^{1,a})}^{1,a} \}$. By mapping this to its corresponding ``local ranks" $k_i^{1,a}$, we can leverage the properties of order statistics to ensure fairness under client heterogeneity. We will delve into the details of our approach and observations in the next subsection.

To this end, we present an overview of our algorithm in Figure \ref{fig:overview}, consisting of two main parts — 1). establishing a candidate set with a distributed algorithm that meets the fairness constraint with high probability, and 2). selecting the optimal rank pair with the smallest misclassification error. 
In this section, we first discuss the simplest case: a binary-group and binary-label scenario, i.e., $\mathcal{Y}=\mathcal{A}=\{0,1\}$. 
However, it is important to note that FedFaiREE is adaptable to various fairness notions and has the additional capacity to accommodate even more diverse situations. Subsequent sections will discuss more fairness concepts like Equalized Odds and further scenarios involving label shift, multi-group fairness, and multi-label classification.

\subsection{Candidate set construction with distributed quantile algorithm}\label{sec:establish}
To select rank pairs whose corresponding classifiers satisfy fairness constraints, we leverage the properties of order statistics. Specifically, we consider score sets that $k^{1,a}$ represents the rank in the sorted $T^{1,a}$. 
To account for heterogeneity among clients, we further introduce the notation $k_i^{1,a}$ to denote the corresponding rank of $t_{k^{1,a}}^{1,a}$ within the sorted set $T_i^{1,a}$, where $i \in [S]$ and $k_i^{1,a}$ satisfies $t_{i,(k_i^{1,a})}^{1,a} \leq t_{(k^{1,a})}^{1,a} < t_{i,(k_i^{1,a}+1)}^{1,a}$. For simplicity, we further define $\vk^{1,a}=(k_1^{1,a}, \cdots, k_S^{1,a})$, and $Q(\alpha,\beta)$ represents independent variable following a $\text{Beta}(\alpha, \beta)$ distribution. We present the following observation regarding fairness control. %, as stated in the following proposition.

\begin{proposition}
Under Assumption \ref{ass:mix}, for $a \in \{0,1\}$, consider $k^{1, a} \in \{1, \ldots, n^{1, a}\}$, the corresponding $k_i^{1,a}$ for $i \in [S]$ and the score-based classifier $\phi(x, a) = \1\{f(x, a) > t_{(k^{1, a})}^{1, a}\}$. Define 
\begin{small}
\begin{equation}
\begin{aligned}
    h_{y,a}(\vu,\vv)=&\mathbb{P}\Big(\sum_{i=1}^S \pi_i^{y,a} Q\left(u_i, n_i^{y,a}+1-u_i\right) \\
    &- \sum_{i=1}^S \pi_i^{y,1-a} Q\left(v_i, n_i^{y,1-a}+1-v_i\right) \geq \alpha\Big).
\end{aligned}
\end{equation}
\end{small}
Then we have:
\begin{equation}
\begin{aligned}
    \mathbb{P}(|DEOO(\phi)| > \alpha) \leq& h_{1,0}(\vk^{1,0}+\mathbf{1},\vk^{1,1}) \\
    &+ h_{1,1}(\vk^{1,1}+\mathbf{1},\vk^{1,0}),
    % &\mathbb{P}\left(\sum_{i=1}^S \pi_i^{1,1} Q\left(k_i^{1,1}+1, n_i^{1,1}-k_i^{1,1}\right) - \sum_{i=1}^S \pi_i^{1,0} Q\left(k^{1,0}, n_i^{1,0}+1-k^{1,0}\right) > \alpha\right) \\
    % &+ \mathbb{P}\left(\sum_{i=1}^S \pi_i^{1,1} Q\left(k_i^{1,1}, n_i^{1,1}+1-k_i^{1,1}\right) - \sum_{i=1}^S \pi_i^{1,0} Q\left(k^{1,0}+1, n_i^{1,0}-k^{1,0}\right) < -\alpha\right).
\end{aligned}
\label{eq:loss}
\end{equation}
where $\pi_i^{1,a} = \mathbb{P}( x \text{ from client } i \mid x \text{ with } Y=1, A=a)$.
\label{prop:deoo}
\end{proposition}

%The proof of Proposition \ref{prop:deoo} can be found in Appendix 
This proposition enables us to select classifiers that satisfy fairness constraints with arbitrary finite sample and no distributional assumption. 
Moreover, $Q(\alpha,\beta)$ can be efficiently estimated by Monte Carlo simulations in applications. Specifically, we approximated $Q(\alpha, \beta)$ by conducting random sampling 1000 times in our experiment, yielding a highly satisfactory approximation.

Due to the need of computing local ranks to make use of Proposition \ref{prop:deoo}, %\linjun{can we say ``Due to the need of computing local ranks as required by Proposition \ref{prop:deoo}''?}
it is crucial to consider the tradeoff between accuracy and communication cost in real applications. We can adopt distributed quantile algorithms to reduce communication costs while controlling errors in calculating local ranks. Therefore, we present an alternative formulation of Proposition \ref{prop:deoo} to allow errors in the local rank calculation. To begin with, we introduce the concept of approximate quantiles and ranks \citep{luo2016quantiles,lu2023federated}.
%\qichuan{I'm not sure whether here we need to cite some other paper for similar definition}

\begin{definition}\label{def:approximate}
    ($\varepsilon$-approximate $\beta$-quantile and rank of a given set)
For an error $\varepsilon \in(0,1)$, the $\varepsilon$-approximate $\beta$-quantile of a given set is any element with rank between $(\beta-\varepsilon) N$ and $(\beta+\varepsilon) N$, where $N$ is the total number of elements in set. Further, the $\varepsilon$-approximate rank of an element in a given set is any rank between $(\beta-\varepsilon) N$ and $(\beta+\varepsilon) N$ where $\beta N$ represents the real rank.
\label{def:app-rank}
\end{definition}

Under Definition \ref{def:app-rank}, if the rank estimation method produces $\varepsilon$-approximate ranks, it is possible to correspondingly modify Proposition \ref{prop:deoo}. 
%\linjun{what do you mean by "the changes in Proposition \ref{prop:deoo}"?}

%ref waits for Appendix writing

\begin{proposition}
Under Assumption \ref{ass:mix}, for $a \in \{0,1\}$, consider $k^{1, a} \in \{1, \ldots, n^{1, a}\}$, the corresponding $\hat{k}_i^{1,a}$ for $i \in [S]$ which are $\varepsilon$-approximate ranks and the score-based classifier $\phi(x, a) = \1\{f(x, a) > t_{(k^{1, a})}^{1, a}\}$ . Define 
\begin{small}
\begin{equation}
\begin{aligned}
    h_{y,a}(\vu,\vv)=&\mathbb{P}\Big(\sum_{i=1}^S \pi_i^{y,a} Q\left(u_i, n_i^{y,a}+1-u_i\right) \\
    &- \sum_{i=1}^S \pi_i^{y,1-a} Q\left(v_i, n_i^{y,1-a}+1-v_i\right) \geq \alpha\Big).    
\end{aligned}
\end{equation}

\end{small} Then we have:
\begin{equation}
\small
\begin{aligned}
    \mathbb{P}(|DEOO(\phi)| > \alpha) \leq h_{1,0}(\bm{M}^{1,0},\vm^{1,1})+h_{1,1}(\bm{M}^{1,1},\vm^{1,0}),
    %&\mathbb{P}\left(\sum_{i=1}^S \pi_i^{1,1} Q\left(M_i^{1,1}, n_i^{1,1}+1-M_i^{1,1}\right) - \sum_{i=1}^S \pi_i^{1,0} Q\left(m_i^{1,0}, n_i^{1,0}+1-m_i^{1,0}\right) > \alpha\right) \\
    %&+ \mathbb{P}\left(\sum_{i=1}^S \pi_i^{1,1} Q\left(m_i^{1,1}, n_i^{1,1}+1-m_i^{1,1} \right) - \sum_{i=1}^S \pi_i^{1,0} Q\left(M_i^{1,0}, n_i^{1,0}+1-M_i^{1,0}\right) < -\alpha\right).
\end{aligned}
\label{eq:adloss}
\end{equation}
where $\pi_i^{1,a}$ is defined in Proposition \ref{prop:deoo}, $\bm{M}^{1,a}=(M_1^{1,a}, \cdots, M_S^{1,a})$, $\vm^{1,a}=(m_1^{1,a}, \cdots, m_S^{1,a})$, $M_i^{1,a}=max\big(\lceil \hat{k}_i^{1,a}+\varepsilon n_i^{1,a}\rceil, n_i^{1,a}+1\big)$, $m_i^{1,a}=min\big(\lceil \hat{k}_i^{1,a}-\varepsilon n_i^{1,a}\rceil,0\big)$. Especially, $Q(0,\beta)=0$ and $Q(\alpha,0)=1$ for $\alpha,\beta \neq 0$. 
\label{prop:adloss}
\end{proposition}

In practical distributed settings, calculating the exact local rank in Proposition \ref{prop:adloss} is generally hard due to communication constraints. By adopting approximate $\varepsilon$ and related parameters in a distributed quantile algorithm, we strike a balance between accuracy and communication cost, enabling the effective implementation of our algorithm in distributed environments.

In our experiments, we implemented the Q-digest ~\citep{q-digest}, a tree-based sketching distributed quantile algorithm commonly used for efficiently approximating quantiles and ranks computation with rigorous theory controlling the error. Due to the inherent characteristics of the Q-digest algorithm, it only yields approximate quantiles and ranks that tend to be greater than their true values. However, considering the adaptability of other distributed quantile algorithms and aiming to reduce the absolute value of $\varepsilon$, we take into account both upward and downward estimation deviations as described in Definition \ref{def:approximate}. %\linjun{what are upward and downward deviations?}%It's worth noting that in our experiments, we specifically focus on the upward direction; additional details can be found in Appendix \ref{?}.

By Proposition \ref{prop:adloss}, we construct the candidate set $K$ as 
\begin{small}
 \begin{equation}\label{eq:construct K}
     K=\{(k^{1,0}, k^{1,1})| L(\vk^{1,0}, \vk^{1,1})< 1-\beta\},
 \end{equation}
 \end{small}
 where $\vk^{1,a}=(\hat{k}_1^{1,a}, \cdots, \hat{k}_S^{1,a})$ are estimated corresponding “local ranks” of $k^{1,a}$, and $L(\vk^{1,0}, \vk^{1,1})$ represents the right-hand side of the inequality \ref{eq:adloss}. 

\subsection{Selection for the optimal threshold }
In this subsection, we elaborate on our method for selecting the optimal threshold. For a given pair $\left(k^{1,0}, k^{1,1}\right)$ from the candidate set, we exploit the properties of order statistics to compute the estimated misclassification error and then select the pair minimizing the estimated error. 

\begin{algorithm}[t]
\small
\caption{FedFaiREE for DEOO}\label{alg:fed}
\textbf{Input:} Train dataset $D_i = D_i^{0,0} \cup D_i^{0,1} \cup D_i^{1,0} \cup D_i^{1,1}$; pre-trained classifier $\phi_0$ with function f; fairness constraint parameter $\alpha$ ; Confidence level parameter $\beta$; Weights of different clients $\pi$ \\
\textbf{Output:} classifier $\hat{\phi}(x,a)= \1 
\{ f(x,a) >  t_{(k^{1,a})}^{1,a}\}$ \\
\begin{algorithmic}[1]
\STATE{\bf Client Side:}
\algorithmiccomment{Calculate scores and update sketches}
\FOR{i=1,2,..,$S$} 
\STATE{Score on train data points in $D_i$ and get $T_i^{y,a}=\{t_{i,1}^{y,a},t_{i,2}^{y,a}, \cdots, t_{i,n_i^{y,a}}^{y,a}\}$ }
\STATE{Sort $T_i^{y,a}$ and calculate q-digest of $T_i^{y,a}$ on client $i$}
\STATE{Update digest to server}
\ENDFOR
\STATE{\bf Server Side:}
%\STATE{Combine q-digests }%and get the sketch of sorted $T^{y,a}$
\STATE{Construct $K$ by $K=\{(k^{1,0}, k^{1,1})| L(\vk^{1,0}, \vk^{1,1})< 1-\beta \}$}
\algorithmiccomment{Establishing a set that satisfies fairness constraints and confidence requirements using order statistics. The search for $(k^{1,0}, k^{1,1})$ can be simplified using technique in Appendix \ref{appendix:search}.}
\STATE{Select optimal $(k_0,k_1)$ by minimizing Equation \ref{eq:error_p} using estimated values $\hat{p}^i_{a}$, $\hat{p}^i_{Y,a}$ and $\hat{q}^i_{Y,a}$}
\algorithmiccomment{Searching for the classifier that minimizes the misclassification error.}
%$\hat{p}^i_{a} = \frac{n_i^{0,a}+n_i^{1,a}}{n_i^{0,0}+n_i^{0,1}+n_i^{1,0}+n_i^{1,1}}$ and $\hat{p}^i_{Y,a} = \frac{n_i^{1,a}}{n_i^{0,a}+n_i^{1,a}}$}
\end{algorithmic}
\end{algorithm}

To facilitate this, %diffierent from FaiREE, 
we need to compute the approximate ranks of $t_{\left(k^{1,0}\right)}^{1,0}$ and $t_{\left(k^{1,1}\right)}^{1,1}$ in the sorted sets $T_i^{0,0}$ and $T_i^{0,1}$, where $i \in [S]$, respectively. Specifically, we determine $k_{i}^{0, a}$ such that $t_{i,\left(k_{i}^{0, a}\right)}^{0, a} \leq t_{\left(k^{1, a}\right)}^{1, a} < t_{i,\left(k_{i}^{0, a}+1\right)}^{0, a}$ for $a \in \{0,1\}$. To simplify, in the following sections, we assume the corresponding $\hat{k}_i^{1,a}$ for $i \in [S]$ are $\varepsilon$-approximate ranks and the estimated quantiles presented by distributed quantile algorithm are $\varepsilon$-approximate quantiles. Then, we commence by presenting our observation on the estimation of misclassification error through the following proposition.
\begin{proposition}\label{prop:error}
    Under Assumption \ref{ass:mix}, the misclassification error can be estimated by
    \begin{small}
    \begin{equation}\label{eq:error_p}
    \begin{aligned}    &\hat{\mathbb{P}}\left(\hat{\phi}(x, a) \neq Y\right)= \sum_{i=1}^S \pi_i \Big[\frac{\hat{k}_{i}^{1,0}+0.5}{n_i^{1,0}+1} p^i_{0} p^i_{Y, 0}+\frac{\hat{k}_{i}^{1,1}+0.5}{n_i^{1,1}+1} p^i_{1} p^i_{Y, 1} \\
    &+ \frac{n_i^{0,0}+0.5-\hat{k}_{i}^{0,0}}{n_i^{0,0}+1} p^i_{0}q^i_{Y, 0}+\frac{n_i^{0,1}+0.5-\hat{k}_{i}^{0,1}}{n_i^{0,1}+1} p^i_{1}q^i_{Y, 1}
    \Big].
    \end{aligned}  
    \end{equation}
    \end{small}
    Further, the discrepancy between empirical error and true error is upper bounded by the following:
    \begin{equation}\small
    \begin{aligned}
        \left| \mathbb{P}\left(\hat{\phi}(x, a) \neq Y\right)-\hat{\mathbb{P}}\left(\hat{\phi}(x, a) \neq Y\right) \right| \leq \theta,
    \end{aligned}
    \end{equation} 
    where \begin{small}
        $\theta=\sum_{i=1}^S \pi_i [e_i^{0,0} p^i_{0}q^i_{Y, 0}
    +e_i^{0,1} p^i_{0}p^i_{Y, 0}+e_i^{1,0} p^i_{1}q^i_{Y, 1} $
    $+e_i^{1,1} p^i_{1}p^i_{Y, 1}], e_i^{y,a}=\frac{2\lfloor \varepsilon n_i^{y,a} \rfloor +1}{2\left(n_i^{y,a}+1\right)}$.
    \end{small}
    \end{proposition}%prop:error
Proposition \ref{prop:error} provides a method for estimating the overall misclassification error using data from the training set with Equation \ref{eq:error_p}. However, we may not have exact knowledge of the probabilities $p^i_{a}$ and $p^i_{Y,a}$. In such cases, we can use the estimated values $\hat{p}^i_{a} = \frac{n_i^{0,a}+n_i^{1,a}}{n_i^{0,0}+n_i^{0,1}+n_i^{1,0}+n_i^{1,1}}$, $\hat{p}^i_{Y,a} = \frac{n_i^{1,a}}{n_i^{0,a}+n_i^{1,a}}$, $\hat{q}^i_{Y,a}=1-\hat{p}^i_{Y,a}$%\linjun{the right hand side does not depend on $Y$?} 
to calculate the empirical error. We will further present a theorem to show that we can achieve a desirable accuracy using the estimated values in Section \ref{sec:gua}. 

At the end of this section, we provide a concise summary of our algorithm in Algorithm \ref{alg:fed}. 
%It's worth noting that we have set the weights of distributions of different clients, denoted as $\pi$, as a given parameter. 
It is worth noting that while in our experiment, we assume that $\pi_i$ is proportional to $n_i$, we may not know the exact values of $\pi_i$ in real applications. To enhance the robustness of our approach in such real-world scenarios, one can consider introducing a hypothesis space denoted as $H(\pi)$ to model the range of $\pi$ and incorporate $\max_{\pi \in H(\pi)}$ into equations \ref{eq:construct K} and \ref{eq:error_p}.

\section{Theoretical Guarantees}\label{sec:gua}
In this section, we provide the accuracy analysis for FedFaiREE. To mitigate situations where there might be an extreme initial pre-trained classifier, we introduce the following assumption.
\begin{assumption}
    The distribution of $f(x, a)$ exhibits the following property. When conducting $N$ independent samplings to form a sample set, let $q_0$ be the $\beta$-quantile of the sample set. There exist function $\delta: \mathbb{N} \to \mathbb{R}$, constant $\gamma>0$, such that $\lim_{N \to \infty} \delta(N) = 0$ and with a probability of at least $1-\delta(N)$, for any $q$ considered as an $\varepsilon$-approximate $\beta$-quantile of the sample set, it satisfies that , $q$ lies within the $\gamma \varepsilon$-neighborhood of $q_0$.
    \label{ass:quan}
\end{assumption}
In simpler terms, Assumption 1 is a property akin to Lipschitz continuity, ensuring that the approximated quantile and the actual quantile do not exhibit extreme discrepancies.
Moreover, in the following theorem, we establish a theoretical basis for the accuracy of FedFaiREE. 
To facilitate accurate comparisons, we introduce the notion of the \textit{fair Bayes-optimal classifier}, denoting the classifier with the optimal accuracy under fairness constraints. The precise definition of the fair Bayes-optimal classifier under DEOO can be found in Lemma~\ref{bayes-optimal}. To be concise, we denote the standard Bayes-optimal classifier without fairness constraints by $\phi^*(x,a)=\1\{f^*(x,a)>1/2\}$, where $f^{*} \in {\arg\min}_{f} [\bP(Y \neq \1\{f(x,a)>1/2\})]$.
\begin{theorem} \label{thm:acc}
Under Assumptions \ref{ass:mix} and \ref{ass:quan}, given $\alpha^{\prime}<\alpha$. Suppose $\hat{\phi}$ is the final output of FedFaiREE. We then have:

(1) $|D E O O(\hat{\phi})|<\alpha$ with probability $(1-\delta)^{N}$, where $N$ is the size of the candidate set.

(2) Suppose the density distribution functions of $f^{*}$ under $A=a, Y=1$ are continuous. When the input classifier $f$ satisfies $\left|f(x, a)-f^{*}(x, a)\right| \leq \epsilon_{0}$, for any $\epsilon>0$ such that $F_{(+)}^{*}(\epsilon+\gamma \varepsilon) \leq$ $\frac{\alpha-\alpha^{\prime}}{2}-F_{(+)}^{*}\left(2 \epsilon_{0}\right)$, we have
\begin{equation}\small
\begin{aligned}
        &\mathbb{P}(\hat{\phi}(x, a) \neq Y)-\mathbb{P}\left(\phi_{\alpha^{\prime}}^{*}(x, a) \neq Y\right) \\
        &\leq  2 F_{(+)}^{*}\left(2 \epsilon_{0}\right)+2 F_{(+)}^{*}(\epsilon+\gamma \varepsilon)+8 \epsilon^{2}+20 \epsilon+ 2\theta
\end{aligned}
    %     &+\sum_{i=1}^S 2\pi_i \left[e_i^{0,0} p^i_{0}\left(1-p^i_{Y, 0}\right)
    % +e_i^{0,1} p^i_{0}p^i_{Y, 0} +e_i^{1,0} p^i_{1}\left(1-p^i_{Y, 1}\right)
    % +e_i^{1,1} p^i_{1}p^i_{Y, 1}\right] 
\end{equation}
with probability \begin{small}
    $1-4\sum_{a=0}^{1}\sum_{i=1}^{S} e^{-2 n_i^{0,a} \epsilon^{2}}-\prod_{i=1}^S \big(1-F_{i(-)}^{1,0}(2 \epsilon)\big)^{n_i^{1,0}}-\prod_{i=1}^S \big(1-F_{i(-)}^{1,1}(2 \epsilon)\big)^{n_i^{1,1}}-\delta, $
\end{small}
where $\delta=\delta^{1,0}(n^{1,0})+\delta^{1,1}(n^{1,1})$, $\theta$ is defined in Proposition \ref{prop:error} and the definition of $F_{(+)}$ and $F_{(-)}$ are shown in Lemma \ref{le:dis}.
\end{theorem}
%The proof is presented in the Appendix \ref{appendix:thm_acc}. 
This theorem provides assurance that our method can achieve almost the optimal misclassification error with DEOO constraints, provided that the input classifier is chosen appropriately, i.e., is close enough to the Bayes-optimal one. This theorem underscores the effectiveness of our approach in minimizing errors when ensuring fairness in a distribution-free and finite-sample manner.

\section{Extension to Different Scenarios}\label{sec:app} 
\subsection{Label Shift in Test Set}
In this section, we explore the application of our algorithm in various scenarios. First, we assume the presence of a label shift in the test set, a situation that is frequently encountered in real-world applications \citep{pmlr-v202-plassier23a,pmlr-v202-tian23a}. To do so, we first need to revise Assumption \ref{ass:mix} to adapt extension settings.
\begin{assumption}\label{ass:shift}
The training data points on client $i$ are i.i.d drawn from the distribution $P_i$, and we further assume the global distribution $P$ is a mixture of $P_1,\cdots, P_S$ with weight $\{ \pi_i \}_{i \in [S]} \in \Delta_S$, while the test data points are sampled from another distribution $P_{i}$, heterogeneity between $P$ and which induced due to label shift, that is, we assume that
\begin{equation}\small
\begin{aligned}
    \left(X_{k}^{i}, Y_{k}^{i}\right) \sim P_i\text{, } P^{mix}= \sum_{i=1}^S \pi_i P_i=P(X,A|Y)*P^{mix}(Y)\text{, } \\
\left(X^{\text{test}},Y^{\text{test}}\right) \sim P_{i}=P(X,A|Y)*P_{i}(Y).
\end{aligned}
\end{equation}
\end{assumption}
We note that FedFaiREE can be adapted to Assumption \ref{ass:shift} by modifying the target function for the optimal rank selection from Equation \ref{eq:error_p} to the following equation:
\begin{equation}\small
\begin{aligned}
&\hat{\mathbb{P}}\left(\hat{\phi}(x, a) \neq Y\right)= \sum_{i=1}^S \pi_i \big[\frac{\hat{k}_{i}^{1,0}+0.5}{n_i^{1,0}+1} p^i_{0} p^i_{Y, 0}w^{1,0} \\
&+\frac{\hat{k}_{i}^{1,1}+0.5}{n_i^{1,1}+1} p^i_{1} p^i_{Y, 1}w^{1,1} 
+ \frac{n_i^{0,0}+0.5-\hat{k}_{i}^{0,0}}{n_i^{0,0}+1} p^i_{0}q^i_{Y, 0}w^{0,0} \\
&+\frac{n_i^{0,1}+0.5-\hat{k}_{i}^{0,1}}{n_i^{0,1}+1} p^i_{1}q^i_{Y, 1}w^{0,1}
\big],    
\end{aligned} \label{eq:error_shift}
\end{equation}
where $w^{y,a}= \frac{p^{S+1}_a p^{S+1}_{Y,a}}{p_a p_{Y,a}}$.
In Appendix \ref{appendix:shift}, we provide a detailed proposition to ensure the accuracy of our estimations and present a concise algorithm. Furthermore, to account for label shift scenarios, we offer a theorem guarantee as a revised version of \ref{thm:acc} at the end of this subsection.%\qichuan{I'm not sure whether we should put the revised version of the target function at here.} 
\begin{theorem} 
\label{thm:label}
Under Assumptions \ref{ass:quan} and \ref{ass:shift}, given $\alpha^{\prime}<\alpha$. Suppose $\hat{\phi}$ is the final output of FedFaiREE. We then have:

(1) $|D E O O(\hat{\phi})|<\alpha$ with probability $(1-\delta)^{N}$, where $N$ is the size of the candidate set.

(2) Suppose the density distribution functions of $f^{*}$ under $A=a, Y=1$ are continuous. When the input classifier $f$ satisfies $\left|f(x, a)-f^{*}(x, a)\right| \leq \epsilon_{0}$, for any $\epsilon>0$ such that $F_{(+)}^{*}(\epsilon+\gamma \varepsilon) \leq$ $\frac{\alpha-\alpha^{\prime}}{2}-F_{(+)}^{*}\left(2 \epsilon_{0}\right)$, we have
\begin{equation}\small
    \begin{aligned}
        &\mathbb{P}(\hat{\phi}(x, a) \neq Y)-\mathbb{P}\left(\phi_{\alpha^{\prime}}^{*}(x, a) \neq Y\right)  \\
        &\leq 2 F_{(+)}^{*}\left(2 \epsilon_{0}\right)+2 F_{(+)}^{*}(\epsilon+\gamma \varepsilon)+2\theta^{\prime}+O(\epsilon),
    %     &+\sum_{i=1}^S 2\pi_i \left[e_i^{0,0}\hat{w}^{0,0} p^i_{0}\left(1-p^i_{Y, 0}\right)
    % +e_i^{0,1}\hat{w}^{0,1} p^i_{0}p^i_{Y, 0} +e_i^{1,0}\hat{w}^{1,0} p^i_{1}\left(1-p^i_{Y, 1}\right)
    % +e_i^{1,1}\hat{w}^{1,1} p^i_{1}p^i_{Y, 1}\right]
    \end{aligned}
\end{equation}
with probability $
1-4\sum_{a=0}^{1}\sum_{i=1}^{S} e^{-2 n_i^{0,a} \epsilon^{2}}-\prod_{i=1}^S\big(1-F_{i(-)}^{1,0}(2 \epsilon)\big)^{n_i^{1,0}}-\prod_{i=1}^S\big(1-F_{i(-)}^{1,1}(2 \epsilon)\big)^{n_i^{1,1}}-\delta, $ 
where the definitions of $\delta$, $F_{(+)}$, $F_{(-)}$ are same with Theorem \ref{thm:acc}, $\theta^{\prime}$ is defined in Proposition \ref{prop:error_shift}.
%$w^{y,a}= \frac{p^{S+1}_a p^{S+1}_{Y,a}}{p_a p_{Y,a}}$ and
%$\theta^{\prime}=\sum_{i=1}^S 2\pi_i \left[e_i^{0,0}w^{0,0} p^i_{0}\left(1-p^i_{Y, 0}\right)
    % +e_i^{0,1}w^{0,1} p^i_{0}p^i_{Y, 0} +e_i^{1,0}w^{1,0} p^i_{1}\left(1-p^i_{Y, 1}\right)
    % +e_i^{1,1}w^{1,1} p^i_{1}p^i_{Y, 1}\right].$
\end{theorem}
%\linjun{add one or two sentences to explain the theoretical results}
In summary, Theorem \ref{thm:label} assures that our FedFaiREE algorithm can effectively control fairness and maintain accuracy in situations where label shift is present in the test data. These guarantees are essential for deploying fair and accurate machine learning models in practical applications.

\subsection{Equalized Odds}\label{sec:deo}
We have also explored the potential extension of our algorithm to fairness indicators beyond DEOO. In this subsection, we will discuss its application to Equalized Odds. More fairness notions are presented in Appendix \ref{appendix:application}.
\begin{definition}(Equalized Odds \citep{hardt2016equality})
        A classifier satisfies Equalized Odds if it satisfies the following equality:
          \begin{small}
              $\bP_{X \mid A=1, Y=1}(\widehat{Y}=1) = \bP_{X \mid A=0, Y=1}(\widehat{Y}=1)$
          \end{small}  and 
          \begin{small}
              $\bP_{X \mid A=1, Y=0}(\widehat{Y}=1) = \bP_{X \mid A=0, Y=0}(\widehat{Y}=1)$. 
          \end{small}       
\end{definition}
%\linjun{we may move Definition 2.2 to Section 5.2 there if it is only discussed there?}
Similarly, we can express the fairness constraints under Equalized Odds as
% we denote DEO as follows:
% \begin{small}$$DEO=(\bP_{X \mid A=1, Y=1}(\widehat{Y}=1) - \bP_{X \mid A=0, Y=1}(\widehat{Y}=1), \bP_{X \mid A=1, Y=0}(\widehat{Y}=1) - \bP_{X \mid A=0, Y=0}(\widehat{Y}=1)).$$\end{small}
%For brevity, we use the notation $\preceq$ to represent element-wise comparison between vectors. 
%Therefore, $(\alpha_1, \alpha_2)$ difference tolerance can be written as 
$|DEO| \preceq (\alpha_1, \alpha_2)$, which is equivalent to \begin{small}
   $|\bP_{X \mid A=1, Y=1}(\widehat{Y}=1) - \bP_{X \mid A=0, Y=1}(\widehat{Y}=1)| \leq \alpha_1$ 
\end{small} and \begin{small}
$|\bP_{X \mid A=1, Y=0}(\widehat{Y}=1) - \bP_{X \mid A=0, Y=0}(\widehat{Y}=1)| \leq \alpha_2$\end{small}. Hence, in order to consider two fairness constraints simultaneously, we modify Equation \ref{eq:construct K} as follows.
\begin{equation}\label{eq:construct K deo} \small
     K=\{(k^{*,0}, k^{*,1})| h^*_{1,1}+h^*_{1,0}+h^*_{0,1}+h^*_{0,0}< 1-\beta \},
 \end{equation}
where $h^*_{y,a}$ are functions of $\vk^{*,a}$ defined in Proposition \ref{prop:deo}. Additional details and propositions can be found in Appendix \ref{appendix:deo}.
% \begin{proposition}\label{prop:deo}
% Under Assumption \ref{ass:mix}, for $a \in \{0,1\}$, consider $k^{1, a} \in \{1, \ldots, n^{1, a}\}$, the corresponding $\hat{k}_i^{1,a}$ for $i \in [S]$ which are $\varepsilon$-approximate ranks and the score-based classifier $\phi(x, a) = \1\{f(x, a) > t_{(k^{1, a})}^{1, a}\}$ . 
% Define 
% \begin{small}$$h_{y,a}=\mathbb{P}\left(\sum_{i=1}^S \pi_i^{y,a} Q\left(M_i^{y,a}, n_i^{y,a}+1-M_i^{y,a}\right) - \sum_{i=1}^S \pi_i^{y,1-a} Q\left(m_i^{y,1-a}, n_i^{y,1-a}+1-m_i^{y,1-a}\right) \geq \alpha\right).$$\end{small}
% Then we have:
% \begin{equation}\small
% \begin{aligned}
%     \mathbb{P}(|DEO(\phi)| \preceq (\alpha,\alpha)) \geq 1-h_{1,1}-h_{1,0}-h_{0,1}-h_{0,0},
% \end{aligned}
% \label{eq:adloss_DEO}
% \end{equation}
% where the definitions of $M_i^{y,a}$, $m_i^{y,a}$, $\pi_i^{y,a}$, $Q(A,B)$ are similar to Proposition \ref{prop:deoo}
% %where $M_i^{y,a}=\lceil \hat{k}_i^{y,a}+\varepsilon n_i^{y,a}\rceil$, $m_i^{y,a}=\lceil \hat{k}_i^{y,a}-\varepsilon n_i^{y,a}\rceil$, $\pi_i^{y,a} = \mathbb{P}(\text{sampling } x \text{ from client } i \mid \text{ sampling } x \text{ with label } Y=y \text{ and } A=a)$, and $Q(A,B)$ are independent random variables following Beta distribution, $Q(A,B) \sim \text{Beta}(A, B)$. 
% \end{proposition}
This equation allows us to construct a candidate set under DEO fairness constraints, enabling us to apply our algorithm to achieve Equalized Odds. Furthermore, we provide theoretical guarantees for DEO fairness.
\begin{theorem} \label{thm:acc_eo}
Under Assumptions \ref{ass:mix} and \ref{ass:quan}, given $\alpha^{\prime}<\alpha$. Suppose $\hat{\phi}$ is the final output of FedFaiREE with target DEO constraint. We then have:

(1) $|D E O (\hat{\phi})|<\alpha$ with probability $(1-\delta)^{N}$, where $N$ is the size of the candidate set.

(2) Suppose the density distribution functions of $f^{*}$ under $A=a, Y=1$ are continuous. When the input classifier $f$ satisfies $\left|f(x, a)-f^{*}(x, a)\right| \leq \epsilon_{0}$, for any $\epsilon>0$ such that $F_{(+)}^{*}(\epsilon+\gamma \varepsilon) \leq$ $\frac{\alpha-\alpha^{\prime}}{2}-F_{(+)}^{*}\left(2 \epsilon_{0}\right)$, we have
\begin{equation}\small
    \begin{aligned}
        &\mathbb{P}(\hat{\phi}(x, a) \neq Y)-\mathbb{P}\left(\phi_{\alpha^{\prime}}^{*}(x, a) \neq Y\right) \\
        &\leq  2 F_{(+)}^{*}\left(2 \epsilon_{0}\right)+2 F_{(+)}^{*}(\epsilon+\gamma \varepsilon)+2\theta+O( \epsilon)
    %     &+\sum_{i=1}^S 2\pi_i \left[e_i^{0,0} p^i_{0}\left(1-p^i_{Y, 0}\right)
    % +e_i^{0,1} p^i_{0}p^i_{Y, 0} +e_i^{1,0} p^i_{1}\left(1-p^i_{Y, 1}\right)
    % +e_i^{1,1} p^i_{1}p^i_{Y, 1}\right]+O( \epsilon)
    \end{aligned}
\end{equation}
with probability $
1-4\sum_{a=0}^{1} \sum_{i=1}^{S} e^{-2 n_i^{0,a} \epsilon^{2}}-\prod_{i=1}^S\big(1-F_{i(-)}^{1,0}(2 \epsilon)\big)^{n_i^{1,0}}-\prod_{i=1}^S\big(1-F_{i(-)}^{1,1}(2 \epsilon)\big)^{n_i^{1,1}}-\delta, $
where the definitions of $\delta$, $\theta$, $F_{(+)}$, $F_{(-)}$ are same as Theorem \ref{thm:acc}.
\end{theorem}

\subsection{Extension to Multi-Groups}
Recalling the definition of $DEOO$, we define a metric for Equality of Opportunity under Multiple Groups as:
$$\small
\begin{aligned}
DEOOM = &\max_{a}\{|\bP_{X \mid A=a, Y=1}(\widehat{Y}=1)  \\
&- \bP_{X \mid A=0, Y=1}(\widehat{Y}=1)|\}.
\end{aligned}
$$

Here $A=0$ is the group relative advantages and thus we consider the probability difference between $A=0$ and others. %Inspired by Proposition \ref{prop:adloss}, we have 
To control DEOOM, we modify Equation \ref{eq:construct K} as:
\begin{equation}\label{eq:construct K deoom} \small
     K=\{(k^{*,0}, k^{*,1},\cdots, k^{*,a})| \sum_{a=1}^{A_0} h^*_{1,a}< 1-\beta \},
 \end{equation}
where $h^*_{y,a}$ are functions of $\vk^{*,a}$ defined in Proposition \ref{prop:deoom}. Additional details and propositions can be found in Appendix \ref{appendix:multigroup}.
% \begin{proposition}
% \label{prop:deoom}
% Under Assumption \ref{ass:mix}, for $a \in \{0,1,\cdots,A_0\}$, consider $k^{1, a} \in \{1, \ldots, n^{1, a}\}$, the corresponding $\hat{k}_i^{1,a}$ for $i \in [S]$ which are $\varepsilon$-approximate ranks and the score-based classifier $\phi(x, a) = \1\{f(x, a) > t_{(k^{1, a})}^{1, a}\}$ . Define 
% \begin{small}
% $$\begin{aligned}
%     h^*_{y,a}=&\mathbb{P}\left(\sum_{i=1}^S \pi_i^{y,a} Q\left(M^{1,a}_i, n_i^{y,a}+1-M^{1,a}_i\right) \right.\\
%     &\left. -\sum_{i=1}^S \pi_i^{y,0} Q\left(m^{1,0}_i, n_i^{y,0}+1-m^{1,0}_i\right) \geq \alpha\right) \\
%     &+\mathbb{P}\left(\sum_{i=1}^S \pi_i^{y,0} Q\left(M^{1,0}_i, n_i^{y,0}+1-M^{1,0}_i\right) \right.\\
%     &\left.- \sum_{i=1}^S \pi_i^{y,a} Q\left(m^{1,a}_i, n_i^{y,a}+1-m^{1,a}_i\right) \geq \alpha\right)
% \end{aligned}.$$
% \end{small} Then we have:
% \begin{equation}
% \small
% \begin{aligned}
%     \mathbb{P}(|DEOOM(\phi)| > \alpha) \leq \sum_{a=1}^{A_0} h^*_{1,a}
% \end{aligned}
% \label{eq:multi_loss}
% \end{equation}
% where $\pi_i^{1,a}$, $\pi_i^{1,0}$ are similarly defined as in Proposition \ref{prop:adloss}. $M_i^{1,a}=max\big(\lceil \hat{k}_i^{1,a}+\varepsilon n_i^{1,a}\rceil, n_i^{1,a}+1\big)$, $m_i^{1,a}=min\big(\lceil \hat{k}_i^{1,a}-\varepsilon n_i^{1,a}\rceil,0\big)$, $M_i^{1,0}$ and $m_i^{1,0}$ are similarly defined. 
% \end{proposition}
% Proposition \ref{prop:deoom} can be regarded as a direct corollary of Proposition \ref{prop:adloss}. 
Moreover, similar to Theorem \ref{thm:acc}, we have
% \begin{proof}[Proof of Proposition ]
    
% \end{proof}
% \begin{algorithm}[h]
% \caption{FedFaiREE for Multi-Groups}
% \textbf{Input:} Train dataset $D_i = D_i^{0,0} \cup D_i^{0,1} \cup D_i^{1,0} \cup D_i^{1,1}$; pre-trained classifier $\phi_0$ with function f; fairness constraint parameter $\alpha$ ; Confidence level parameter $\beta$; Weights of different clients $\pi$ 
% \textbf{Output:} classifier $\hat{\phi}(x,a)= \1 
% \{ f(x,a) >  t_{(k^{1,a})}^{1,a}\}$
% \begin{algorithmic}[1]
% \STATE{\bf Client Side:}
% \FOR{i=1,2,..,$S$}
% \STATE{Score on train data points in $D_i$ and get $T_i^{y,a}=\{t_{i,1}^{y,a},t_{i,2}^{y,a}, \cdots, t_{i,n_i^{y,a}}^{y,a}\}$ }
% \STATE{Sort $T_i^{y,a}$}
% \STATE{Calculate q-digest of $T_i^{y,a}$ on client $i$}
% \STATE{Update digest to server}
% \ENDFOR
% \STATE{\bf Server Side:}
% %\STATE{Combine q-digests }%and get the sketch of sorted $T^{y,a}$
% \STATE{Construct $K$ by $K=\{(\vk^{1,0}, \vk^{1,1}, \cdots, \vk^{1,A_0})| L< 1-\beta \}$, where L is defined by the right-hand side of Inequality \ref{eq:multi_loss}}
% \STATE{Select optimal $(\vk^{1,0}, \vk^{1,1}, \cdots, \vk^{1,A_0})$ by minimizing equation \ref{eq:error_deoom} using estimated values $\hat{p}^i_{a}$ and $\hat{p}^i_{Y,a}$}
% \end{algorithmic}
% \end{algorithm}

\begin{theorem} \label{thm:acc_deoom}
Under Assumptions \ref{ass:mix} and \ref{ass:quan}, given $\alpha^{\prime}<\alpha$. Suppose $\hat{\phi}$ is the final output of FedFaiREE. We then have:

(1) $|D E O O M(\hat{\phi})|<\alpha$ with probability $(1-\delta)^{N}$, where $N$ is the size of the candidate set.

(2) Suppose the density distribution functions of $f^{*}$ under $A=a, Y=1$ are continuous. When the input classifier $f$ satisfies $\left|f(x, a)-f^{*}(x, a)\right| \leq \epsilon_{0}$, for any $\epsilon>0$ such that $F_{(+)}^{*}(\epsilon+\gamma \varepsilon) \leq$ $\frac{\alpha-\alpha^{\prime}}{2}-F_{(+)}^{*}\left(2 \epsilon_{0}\right)$, we have
\begin{equation}\small
\begin{aligned}
        &\mathbb{P}(\hat{\phi}(x, a) \neq Y)-\mathbb{P}\left(\phi_{\alpha^{\prime}}^{*}(x, a) \neq Y\right) \\
        &\leq  2 F_{(+)}^{*}\left(2 \epsilon_{0}\right)+2 F_{(+)}^{*}(\epsilon+\gamma \varepsilon)+ 2\theta+O(\epsilon)
\end{aligned}
    %     &+\sum_{i=1}^S 2\pi_i \left[e_i^{0,0} p^i_{0}\left(1-p^i_{Y, 0}\right)
    % +e_i^{0,1} p^i_{0}p^i_{Y, 0} +e_i^{1,0} p^i_{1}\left(1-p^i_{Y, 1}\right)
    % +e_i^{1,1} p^i_{1}p^i_{Y, 1}\right] 
\end{equation}
with probability
\begin{small}
    $1-4\sum_{a=0}^{A_0}\sum_{i=1}^{S} e^{-2 n_i^{0,a} \epsilon^{2}}-\sum_{a=0}^{A_0}\prod_{i=1}^S\big(1-F_{i(-)}^{1,a}(2 \epsilon)\big)^{n_i^{1,a}}-\delta,$
\end{small}
where  $\delta$ is similarly to,  $F_{(+)}$ and $F_{(-)}$ are same with Theorem \ref{thm:acc}, and $\theta$ is defined in Proposition \ref{prop:error_deoom}.
% where \begin{small}
%     $\delta=\sum_{a=0}^{A_0}\delta^{1,a}(n^{1,a})$
% \end{small}, $\theta$ is defined in Proposition \ref{prop:error_deoom} and the definition of $F_{(+)}$ and $F_{(-)}$ are shown in Lemma \ref{le:dis}. 
\end{theorem}

Theorem \ref{thm:acc_deoom} offers guarantees for FedFaiREE in multi-group scenarios. Additionally, we investigate multi-label cases and the application of additional fairness notions in the Appendix. These findings demonstrate the adaptability of FedFaiREE to a wide range of scenarios.

\begin{table*}[ht]
\small
\centering
\caption{\textbf{Results on Adult and Compas dataset. } We conducted 100 experimental repetitions for each model on both datasets and compared the accuracy and fairness indicators of different models. The FedFaiREE and $\alpha$ columns indicate whether FedFaiREE was used or not and the fairness constraint. Confidence level $\beta$ is set to be 95\% throughout the experiments. $\overline{ACC}$ and $\overline{|DEOO|}$ represent the averages of accuracy and DEOO (defined in Equation \ref{deoo}). $|DEOO|_{95}$ represents the 95\% quantile of DEOO since we set the confidence level of FedFaiREE to 95\% in our experiments. }\label{tab:result}\resizebox{0.8\linewidth}{!}{
\begin{tabular}{cccccccccc}
\toprule
& & \multicolumn{4}{c}{\textbf{Adult}} &\multicolumn{4}{c}{\textbf{Compas}}\\
\cmidrule(r){3-6} \cmidrule(r){7-10}
{Model} & {\textbf{FedFaiREE}} & {$\alpha$} & {$\overline{ACC}$} & {$\overline{|DEOO|}$} & {$|DEOO|_{95}$} & {$\alpha$} & {$\overline{ACC}$} & {$\overline{|DEOO|}$} & {$|DEOO|_{95}$}\\
\midrule
\multirow{2}{*}{\textbf{FedAvg}} & {\ding{55}} & {/} & {0.844} & {0.131} & {0.178} & {/} & {0.662} & {0.126} & {0.223}\\
 & {\ding{51}} & {0.10} & {0.843} & {\textbf{0.038}} & {\textbf{0.083}} & {0.15} & {0.659} & {\textbf{0.051}} & {\textbf{0.137}}\\
\midrule
\multirow{2}{*}{\textbf{FedFB}} & {\ding{55}} & {/} & {0.850} & {0.057} & {0.117} & {/} & {0.642} & {0.107} & {0.174}\\
 & {\ding{51}} & {0.10} & {0.850} & {\textbf{0.036}} & {\textbf{0.083}}  & {0.15} & {0.641} & {\textbf{0.062}} & {\textbf{0.125}}\\
\midrule
\multirow{2}{*}{\textbf{FairFed}} & {\ding{55}} & {/} & {0.842} & {0.069} & {0.118} & {/} & {0.648} & {0.097} & {0.166}\\
 & {\ding{51}} & {0.10} & {0.841} & {\textbf{0.037}} & {\textbf{0.081}}  & {0.15} & {0.645} & {\textbf{0.047}} & {\textbf{0.114}}\\
\bottomrule
\end{tabular}}
\end{table*}

% Theorem \ref{thm:acc} guarantees accuracy and fairness under the DEO metric. Additionally, it is noteworthy that we consider more scenarios, such as multi-group and multi-label situations, in the Appendix \ref{app:multi}, which FaiREE cannot handle. The application in these scenarios highlights the stronger adaptability of FedFaiREE compared to FaiREE.

\section{Experiments}\label{sec:experiment}
In this section, we study the performance of FedFaiREE on real datasets, including Adult \citep{dua2017uci} and Compas \citep{dieterich2016compas}. In particular, we employed FedFaiREE on FedAvg \citep{Fedavg}, 
%AFL\citep{AFL}, 
FedFB \citep{FedFB}, and FairFed \citep{FairFed}. We train all algorithms using two layers of neural networks. See Appendix \ref{appendix:experiment} for details of the experimental set-up.%, including hyperparameter range, detailed model information, and metrics variances.
%\paragraph{Dataset and Setup.} 

%We compare our method and  the following the datasets.  
\textbf{Dataset.} Adult dataset \citep{dua2017uci}, which is employed for the prediction task that determines whether an individual's income exceeds \$50,000, comprises 45,222 samples, featuring various attributes including age, education, and more.
%\textbf{Compas.} 
Compas dataset \citep{dieterich2016compas}, whose task is to predict whether a person will conduct crime in the future, comprises 7214 samples. Gender is chosen as the sensitive feature for both datasets.

\textbf{Data Processing.} To replicate the decentralized conditions and account for heterogeneity across clients, we adopted the approach introduced by \citet{FairFed}. Specifically, we initiated the process by randomly sampling proportions for various sensitive attributes within each client, using the Dirichlet distribution. Subsequently, we partitioned the dataset into client-specific subsets based on these proportions. Within each of these subsets, we performed an 80-20 split, allocating 80\% of the data as the local client training set and reserving the remaining 20\% for the test set. For the numerical experiments, we repeated this procedure 100 times on both Adult and Compas datasets.

\textbf{Result and Analysis.} 
Table \ref{tab:result} presents the results from experiments conducted on the Adult and Compas datasets. These results showcase that FedFaiREE achieved desirable performance across both datasets. The ``FedFaiREE" column indicates whether FedFaiREE was used, and the ``$\alpha$" columns specify the fairness constraint. Our findings demonstrate that FedFaiREE, with its unique, distribution-free approach to fairness constraints under finite samples, consistently outperforms the original models in controlling DEOO while maintaining relatively high accuracy.
It is worth noting that FedFaiREE achieves desirable performance even when applied to FedAvg, the most fundamental model. This indicates the wide applicability and potential of FedFaiREE across various settings.
Moreover, FedFaiREE was employed with a confidence level of $\beta=0.95$ throughout the experiments, and it successfully controlled the 95th percentile of DEOO, showcasing its robustness. For a comprehensive understanding of FedFaiREE's variance and behavior with varying values of $\alpha$ and $\beta$, please refer to Appendix \ref{app:more_detail} for additional experimental details.

\textbf{Case Study}
To validate the effectiveness of FedFaiREE in scenarios with naturally heterogeneous distributions, we further consider the ACSIncome dataset\citep{ding2021retiring}. 
%and CelebA dataset\citep{liu2015deep}. 
In the ACSIncome dataset, the task is to predict whether an individual's income is above \$50,000, with the sensitive label being Race (white/non-white), and the data partitioned across 50 states. Table 2 presents the results for DEOO and Accuracy. It can be observed that after applying FedFaiREE, we significantly improved DEOO performance while maintaining a high level of accuracy.
%On the other hand, the CelebA dataset involves the task of predicting whether a celebrity photo is attractive, with the sensitive label being Gender. The data is partitioned such that every 20 celebrities constitute a client.

\begin{table}[ht]
\small
\centering
\caption{\textbf{Results on ACSIncome dataset. } See Appendix \ref{app:hyper} for further details. }\label{tab:result_case}\resizebox{0.8\linewidth}{!}{
\begin{tabular}{ccccc}
\toprule
& & \multicolumn{3}{c}{\textbf{ACSIncome}} \\
\cmidrule(r){3-5} 
{Model} & {\textbf{FedFaiREE}} & {$\alpha$} & {$\overline{ACC}$} & {$\overline{|DEOO|}$}   \\
\midrule
\multirow{2}{*}{\textbf{FedAvg}} & {\ding{55}} & {/} & {0.808} & {0.126 }\\
 & {\ding{51}} & {0.10} & {0.806} & {\textbf{0.041}}\\
\midrule
\multirow{2}{*}{\textbf{FairFed}} & {\ding{55}} & {/} & {0.773} & {0.092} \\
 & {\ding{51}} & {0.10} & {0.771} & {\textbf{0.044}} \\
\bottomrule
\end{tabular}}
\end{table}
\vspace{-0.1cm}
\section{Conclusion}
% In this paper, we proposed FedFaiREE, a novel approach designed to ensure fairness constraints under small sample sizes and distribution-free scenarios, while considering issues such as client heterogeneity, communication costs in real-application. We demonstrate that FedFaiREE can be extended to a wide range of group fairness notion and different scenarios like label shift. Future works may include extending FedFaiREE to tasks beyond predicting binary labels; study on more suitable distributed quantile algorithm for rank and quantile calculation.
In this paper, we introduce FedFaiREE, a finite-sample and distribution-free approach to guarantee fairness constraints under the federated learning setting. FedFaiREE addresses concerns that commonly exist in federated learning, such as client heterogeneity, small samples, and limited communication costs. The FedFaiREE framework can be applied to a wide range of group fairness notions and various scenarios, including label shifts, multi-group, and multi-label settings. %Our experiments provide further validation for its practical value. 

For future work, an exploration of more efficient distributed quantile algorithms for rank and quantile calculations within the FedFaiREE framework could significantly enhance its scalability and performance. Moreover, exploring a broader range of application scenarios and assessing its performance in conjunction with in-processing fair federated learning frameworks could yield valuable insights.

\section*{Acknowledgement}
The research of L. Zhang is supported in part by NSF DMS-2340241. We thank Google Cloud Research Credits program for supporting our computing needs.

%\section*{Impact Statement}
%This paper introduces FedFaiREE, a finite-sample and distribution-free approach to guarantee fairness constraints under the federated learning setting.  It is important to acknowledge that the technical solutions offered in this work have inherent limitations and do not resolve challenges related to the collection of training data, the selection of problems to address, or the misuse of automated decision-making technologies. The methodologies developed in this study should be employed with a comprehensive awareness of their capabilities and constraints.
% \section*{Reproducibility Statement}
% The code and dataset for our work can be found in the supplemental materials. To ensure reproducibility, we would like to note that we set random seeds in the range of 0 to 99 for our experiments on Compas dataset. Given that our splitting method allows for potential heterogeneity and varying dataset sizes, which might result in empty datasets, performing "split failed" in our code, we used random seeds in the range of 0 to 111 for the adult dataset when the parameter for the Dirichlet distribution was set to 1. For specific hyperparameter selections, please refer to Table \ref{tab:hyperparameters}.

% In the unusual situation where you want a paper to appear in the
% references without citing it in the main text, use \nocite
\nocite{langley00}

\bibliography{ref}
\bibliographystyle{icml2024}

%%%%%%%%%%%%%%%%%%%%%%%%%%%%%%%%%%%%%%%%%%%%%%%%%%%%%%%%%%%%%%%%%%%%%%%%%%%%%%%
%%%%%%%%%%%%%%%%%%%%%%%%%%%%%%%%%%%%%%%%%%%%%%%%%%%%%%%%%%%%%%%%%%%%%%%%%%%%%%%
% APPENDIX
%%%%%%%%%%%%%%%%%%%%%%%%%%%%%%%%%%%%%%%%%%%%%%%%%%%%%%%%%%%%%%%%%%%%%%%%%%%%%%%
%%%%%%%%%%%%%%%%%%%%%%%%%%%%%%%%%%%%%%%%%%%%%%%%%%%%%%%%%%%%%%%%%%%%%%%%%%%%%%%
\newpage
\appendix
\onecolumn
\section{Proofs}
\subsection{Proof for Proposition \ref{prop:deoo} and \ref{prop:adloss}}
We first introduce following lemma 
\begin{lemma}
    If $t_i^{y,a}$ is variable with continuous density function, we have   
     $$F_i^{y,a}\left(t_{i,\left(k_i^{y,a}\right)}^{y,a}\right) \sim \operatorname{Beta}\left(k_i^{y,a}, n_i^{y,a}-k_i^{y,a}+1\right)$$.
    \label{le:beta}
\end{lemma}

\begin{proof}[Proof of Lemma \ref{le:beta}]
    $F_i^{y,a}$ represents the continuous cumulative distribution functions of $t_i^{y,a}$, and thus we have $F_i^{y,a}\left(t_i^{y,a}\right) \sim U(0,1)$. Furthermore, as $F_i^{y,a}\left(t_{i,\left(k_i^{y,a}\right)}^{y,a}\right)$ denotes the $k_i^{y,a}$-th order statistic of $n_i^{y,a}$ i.i.d samples from $U(0,1)$, we can conclude that $F^{y,a}\left(t_{i,\left(k_i^{y,a}\right)}^{y,a}\right) \sim \operatorname{Beta}\left(k_i^{y,a}, n^{y,a}-k_i^{y,a}+1\right)$
\end{proof}

Back to proof of the Proposition \ref{prop:deoo}, the classifier is
\begin{proof}[Proof of Proposition \ref{prop:deoo}]
 $$
\phi=\left\{\begin{array}{l}
\1\left\{f(x, 0)>t_{\left(k^{1,0}\right)}^{1,0}\right\}, a=0 \\
\1\left\{f(x, 1)>t_{\left(k^{1,1}\right)}^{1,1}\right\}, a=1
\end{array}\right.
$$

we have:

$$
\begin{aligned}
&\mathbb{P}(|D E O O(\phi)|>\alpha) \\
&=\bP\big(|F^{1,1}(t_{\left(k^{1,1}\right)}^{1,1})-F^{1,0}(t_{\left(k^{1,0}\right)}^{1,0})|>\alpha\big) \\
& =\mathbb{P}\big(\sum_{i=1}^S \pi_i^{1,1} F_i^{1,1}(t_{\left(k^{1,1}\right)}^{1,1})-\sum_{i=1}^S \pi_i^{1,0} F_i^{1,0}(t_{\left(k^{1,0}\right)}^{1,0})>\alpha\big) \\
&+\mathbb{P}\big(\sum_{i=1}^S \pi_i^{1,1} F_i^{1,1}(t_{\left(k^{1,1}\right)}^{1,1})- \sum_{i=1}^S \pi_i^{1,0} F_i^{1,0}(t_{\left(k^{1,0}\right)}^{1,0})<-\alpha\big) \\
& \triangleq A+B
\end{aligned}
$$

So we only need to calculate $A$ and $B$ and It is easy to prove that we only need to consider the continuous density function case.
$$
\begin{aligned}
A&  =\bP\big(\sum_{i=1}^S \pi_i^{1,1} F_i^{1,1}(t_{\left(k^{1,1}\right)}^{1,1})-\sum_{i=1}^S \pi_i^{1,0} F_i^{1,0}(t_{\left(k^{1,0}\right)}^{1,0})>\alpha\big) \\
\leq& \bP\big(\sum_{i=1}^S \pi_i^{1,1} F_i^{1,1}(t_{i,(k_i^{1,1}+1)}^{1,1})- \sum_{i=1}^S \pi_i^{1,0} F_i^{1,0}(t_{i,(k_i^{1,0})}^{1,0})>\alpha \big) \\
\end{aligned}
$$
Considering lemma \ref{le:beta} and similar result for B, we complete the proof.   
\end{proof}

For the proof of Proposition \ref{prop:adloss}, we can adjust the estimation of A by introducing the error generated in rank calculation. Specifically, we show that

\begin{proof}[Sketch proof of Proposition \ref{prop:adloss}]
$$
\begin{aligned}
A & =\bP\big(\sum_{i=1}^S \pi_i^{1,1} F_i^{1,1}(t_{\left(k^{1,1}\right)}^{1,1})-\sum_{i=1}^S \pi_i^{1,0} F_i^{1,0}(t_{\left(k^{1,0}\right)}^{1,0})>\alpha\big) \\
& \leq \bP \big(\sum_{i=1}^S \pi_i^{1,1} F_i^{1,1}(t_{i,\left(k_i^{1,1}+\lfloor \varepsilon n_i^{1,1}\rfloor \right)}^{1,1}) - \sum_{i=1}^S \pi_i^{1,0} F_i^{1,0}(t_{i,\left(k_i^{1,0}-\lfloor \varepsilon n_i^{1,0}\rfloor\right)}^{1,0})>\alpha\big)
\end{aligned}
$$
    
\end{proof}

\subsection{Proof for Proposition \ref{prop:error}}
\begin{proof}[Proof for Proposition \ref{prop:error}]

Note the classifier is
$$
\phi=\left\{\begin{array}{l}
\1\left\{f(x, 0)>\hat{t}_{\left(k^{1,0}\right)}^{1,0}\right\}, a=0 \\
\1\left\{f(x, 1)>\hat{t}_{\left(k^{1,1}\right)}^{1,1}\right\}, a=1
\end{array}\right.
$$
So we can calculate the mis-classification error:

\begin{equation}\label{eq:prop-error}
\begin{aligned}
& \mathbb{P}(Y \neq \hat{Y})=\mathbb{P}(Y=1, \hat{Y}=0)+\mathbb{P}(Y=0, \hat{Y}=1) \\
& =\mathbb{P}(Y=1, \hat{Y}=0, A=0)+\mathbb{P}(Y=1, \hat{Y}=0, A=1) \\
& +\mathbb{P}(Y=0, \hat{Y}=1, A=0)+\mathbb{P}(Y=0, \hat{Y}=1, A=1) \\
& =\sum_{i=1}^S \pi_i \big[\mathbb{P}_i(Y=1, \hat{Y}=0, A=0)+\mathbb{P}_i(Y=1, \hat{Y}=0, A=1) \\& +
 \mathbb{P}_i(Y=0, \hat{Y}=1, A=0) + \mathbb{P}_i(Y=0, \hat{Y}=1, A=1) \big]
\end{aligned}   
\end{equation}

For ecah specific i, we have
$$
\begin{aligned}
&\mathbb{P}_i(Y=1, \hat{Y}=0, A=0) \\
& =\mathbb{P}_i(\hat{Y}=1 \mid Y=0, A=0) \mathbb{P}_i(Y=, A=0) \\
& =\mathbb{E}\left[\mathbb{P}_i(f(x, 0) \leq \hat{t}_{\left(k^{1,0}\right)}^{1,0} \mid Y=1, A=0) \mid \hat{t}_{\left(k^{1,0}\right)}^{1,0}\right] p^i_{0} p^i_{Y, 0} \\
& \leq \mathbb{E}\Big[\mathbb{P}_i(f(x, 0) \leq t_{i,\left(\hat{k}_i^{1,0}+\lfloor \varepsilon n_i^{1,0} \rfloor+1\right)}^{1,0} \mid Y=1, A=0) \mid t_{i,\big(\hat{k}_i^{1,0}+\lfloor \varepsilon n_i^{1,0} \rfloor+1\big)}^{1,0}\Big] p^i_{0} p^i_{Y, 0} \\
& =\mathbb{E}\left[F_i^{1,0}\big(t_{i,\left(\hat{k}_i^{1,0}+\lfloor \varepsilon n_i^{1,0} \rfloor+1\right)}^{1,0}\big) \mid t_{i,\left(\hat{k}_i^{1,0}+\lfloor \varepsilon n_i^{1,0} \rfloor+1\right)}^{1,0}\right] p^i_{0} p^i_{Y, 0} \\
& =\frac{\hat{k}_{i}^{1,0}+\lfloor \varepsilon n_i^{1,0} \rfloor+1}{n_i^{1,0}+1} p^i_{0} p^i_{Y, 0}
\end{aligned}
$$
By the similar reasoning, we point out that 
$$
  \mathbb{P}_i(Y=1, \hat{Y}=0, A=0) \geq \frac{\hat{k}_{i}^{1,0}-\lfloor \varepsilon n_i^{1,0} \rfloor}{n_i^{1,0}+1} p^i_{0} p^i_{Y, 0}  
$$
and thus we have
\begin{equation}\label{ineq:10}
\begin{aligned}
    &\big| \mathbb{P}_i(Y=1, \hat{Y}=0, A=0)-\frac{\hat{k}_{i}^{1,0}+0.5}{n_i^{1,0}+1} p^i_{0} p^i_{Y, 0} \big|
    \leq \frac{\lfloor \varepsilon n_i^{1,0} \rfloor+0.5}{n_i^{1,0}+1} p^i_{0} p^i_{Y, 0} 
\end{aligned}
\end{equation}

Moreover, we have

$$
\begin{aligned}
&\mathbb{P}_i(Y=0, \hat{Y}=1, A=0) \\
& =\mathbb{P}_i(\hat{Y}=1 \mid Y=0, A=0) \mathbb{P}_i(Y=0, A=0) \\
& =\mathbb{E}\left[\mathbb{P}_i\left(f(x, 0) \geq \hat{t}_{\left(k^{1,0}\right)}^{1,0} \mid Y=1, A=0\right) \mid \hat{t}_{\left(k^{1,0}\right)}^{1,0}\right] p^i_{0} q^i_{Y, 0} \\
& \geq \mathbb{E}\Big[\mathbb{P}_i\Big(f(x, 0) \geq t_{i,\left(\hat{k}_i^{0,0}+\lfloor \varepsilon n_i^{0,0} \rfloor+1\right)}^{0,0} \mid Y=1, A=0\Big) \mid t_{i,\left(\hat{k}_i^{0,0}+\lfloor \varepsilon n_i^{0,0} \rfloor+1\right)}^{0,0}\Big] p^i_{0} (1-p^i_{Y, 0}) \\
& =\mathbb{E}\big[1-F_i^{0,0}(t_{i,\left(\hat{k}_i^{0,0}+\lfloor \varepsilon n_i^{0,0} \rfloor+1\right)}^{0,0}) \mid t_{i,\left(\hat{k}_i^{0,0}+\lfloor \varepsilon n_i^{0,0} \rfloor+1\right)}^{0,0}\big] p^i_{0} q^i_{Y, 0} \\
& =\frac{n_i^{0,0}-\hat{k}_{i}^{0,0}-\lfloor \varepsilon n_i^{0,0} \rfloor}{n_i^{0,0}+1} p^i_{0} (1-p^i_{Y, 0})
\end{aligned}
$$
Similar, we have 
$$
\begin{aligned}
  &\mathbb{P}_i(Y=0, \hat{Y}=1, A=0) \leq \frac{n_i^{0,0}-\hat{k}_{i}^{0,0}+\lfloor \varepsilon n_i^{0,0} \rfloor+1}{n_i^{0,0}+1} p^i_{0} (1-p^i_{Y, 0}),  
\end{aligned}
$$
and combining these two result, we get
\begin{equation}\label{ineq:00}
\begin{aligned}
      &\big|\mathbb{P}_i(Y=0, \hat{Y}=1, A=0)-\frac{n_i^{0,0}-\hat{k}_{i}^{0,0}+0.5}{n_i^{0,0}+1} p^i_{0} q^i_{Y, 0} \big| \leq \frac{\lfloor \varepsilon n_i^{0,0} \rfloor+0.5}{n_i^{0,0}+1} p^i_{0} (1-p^i_{Y, 0}) 
\end{aligned}
\end{equation}
Following similar process of inequality \ref{ineq:10} and \ref{ineq:00}, we can also show that 
\begin{equation}\label{ineq:11}
\begin{aligned}
    &\big| \mathbb{P}_i(Y=1, \hat{Y}=0, A=1)-\frac{\hat{k}_{i}^{1,1}+0.5 }{n_i^{1,1}+1} p^i_{1} p^i_{Y, 1} \big|
    \leq \frac{\lfloor \varepsilon n_i^{1,1} \rfloor+0.5}{n_i^{1,1}+1} p^i_{1} p^i_{Y, 1} 
\end{aligned}
\end{equation}
\begin{equation}\label{ineq:01}
\begin{aligned}
  &\big|\mathbb{P}_i(Y=0, \hat{Y}=1, A=1)-\frac{n_i^{0,1}-\hat{k}_{i}^{0,1}+0.5}{n_i^{0,1}+1} p^i_{1} (1-p^i_{Y, 1}) \big| \leq \frac{\lfloor \varepsilon n_i^{0,1} \rfloor+0.5}{n_i^{0,1}+1} p^i_{1} (1-p^i_{Y, 1})   
\end{aligned}
\end{equation}

Combining Inequality \ref{ineq:10}-\ref{ineq:01} into Equation \ref{eq:prop-error}, we complete our proof. 
\end{proof}

%Proof of Accuracy theorem

\subsection{Proof for Theorem \ref{thm:acc}}\label{appendix:thm_acc}
% \begin{lemma}\label{le:size}
% If $\min\{n^{1,0}, n^{1,1}\} \geq \lceil \frac{\log \frac{\delta}{2}} {\log (1-\alpha)}\rceil$, for each $i\in\{1,\ldots,M\}$ in the candidate set, we have $|DEOO(\hat\phi_i)| < \alpha$ with probability $1-\delta$.  
% \end{lemma}

% \begin{proof}[Proof of Lemma \ref{le:size}]
% Consider that we set $k^{1,0}=n^{1,0}$ and $k^{1,a}=n^{1,a}$, 
% \end{proof}
To begin with, the Fair Bayes-optimal Classifiers under Equality of Opportunity is defined by following lemma, wherein $\eta_{a}(x):=\bP(Y=1 \mid A=a, X=x)$ stands for the proportion of group $Y=1$ conditioned on $A$ and $X$.
\begin{lemma}[Theorem E.4 in \citep{zeng2022bayes}]\label{bayes-optimal}
Let $E^{\star}=\mathrm{DEOO}\left(f^{\star}\right)$. For any $\alpha>0$, all fair Bayes-optimal classifiers $f_{E, \alpha}^{\star}$ under the fairness constraint $|\mathrm{DEOO}(f)| \leq \alpha$ are given as follows:\\
- When $\left|E^{\star}\right| \leq \alpha, f_{E, \alpha}^{\star}=f^{\star}$\\
- When $\left|E^{\star}\right|>\alpha$, suppose $\bP_{X \mid A=1, Y=1}(\eta_{1}(X)=\frac{p_{1} p_{Y, 1}}{2\left(p_{1} p_{Y, 1}-t_{E, \alpha}^{\star}\right)})=0$, then for all $x \in \mathcal{X}$ and $a \in \mathcal{A}$,
$$
f_{E, \alpha}^{\star}(x, a)=I(\eta_{a}(x)>\frac{p_{a} p_{Y, a}}{2 p_{a} p_{Y, a}+(1-2 a) t_{E, \alpha}^{\star}})
$$
where $t_{E, \alpha}^{\star}$ is defined as
$$
\begin{aligned}
  t_{E, \alpha}^{\star}&=\sup \Big\{t: \bP_{Y \mid A=1, Y=1}\left(\eta_{1}(X)>\frac{p_{1} p_{Y, 1}}{2 p_{1} p_{Y, 1}-t}\right) > \bP_{Y \mid A=0, Y=1}\left(\eta_{0}(X)>\frac{p_{0} p_{Y, 0}}{2 p_{0} p_{Y, 0}+t}\right)+\frac{E^{\star}}{\left|E^{\star}\right|} \alpha\Big\}.  
\end{aligned}
$$
\end{lemma}

\begin{lemma}[Hoeffding's inequality]\label{le:hfd}
Let $X_{1}, \ldots, X_{n}$ be independent random variables. Assume that $X_{i} \in\left[m_{i}, M_{i}\right]$ for every $i$. Then, for any $t>0$, we have
$$
\mathbb{P}\left\{\sum_{i=1}^{n}\left(X_{i}-\mathbb{E} X_{i}\right) \geq t\right\} \leq e^{-\frac{2 t^{2}}{\sum_{i=1}^{n}\left(M_{i}-m_{i}\right)^{2}}}
$$
\end{lemma}

Then, we introduce several lemma to prove Theorem \ref{thm:acc}.
\begin{lemma}\label{le:dis}
For a distribution $F$ with a continuous density function, suppose $q(x)$ denotes the quantile of $x$ under $F$, then for $x > y$, we have $F_{(-)}(x - y) \leq q(x) - q(y) \leq F_{(+)}(x-y)$, where $F_{(-)}(x)$ and $F_{(+)}(x)$ are two monotonically increasing functions, $F_{(-)}(\epsilon) > 0, F_{(+)}(\epsilon)>0$ for any $\epsilon > 0$ and $\mathop{lim}\limits_{\epsilon \rightarrow 0} F_{(-)}(\epsilon) = \mathop{lim}\limits_{\epsilon \rightarrow 0} F_{(+)}(\epsilon)=0$.
\end{lemma}

\begin{proof}[Proof of Lemma \ref{le:dis}]
Since the domain of $q(x)$ is a closed set and $q(x)$ is continuous, we know that $q(x)$ is uniformly continuous. Thus we can easily find $F_{(+)}$ to satisfy the RHS. For $F_{(-)}$, we simply define $F_{(-)}(t) = \inf\limits_{x}\{q(x+t) - q(t)\}$. Since $q(x+t) - q(t) > 0$ for $t>0$ and the domain of $x$ is a closed set, we have $F_{(-)}(\epsilon) > 0$ for $\epsilon>0$ and $\mathop{lim}\limits_{\epsilon \rightarrow 0} F_{(-)}(\epsilon) = 0$. Now we complete the proof.
\end{proof}

\begin{proof}[Proof for theorem \ref{thm:acc}]
In fact, (1) of the theorem is a direct application of Proposition \ref{prop:adloss}, so we only need to prove (2). In partcular, the main idea of our proof is to find a bridge between fair Bayes optimal classifier and our output classifier. 

To begin with, we show that there exist a classifier in our set which is quite similar with fair Bayes optimal classifier. Suppose the fair Bayes optimal classifier has the form $\phi_{\alpha^{\prime}}^{*}(x, a)=\mathbb{I}\left\{f^{*}(x, a)>\lambda_{a}^{*}\right\}$ and our output classifier is of the form $\hat{\phi}(x, a)=\1\left\{f(x, a)>\lambda_{a}\right\}$. 

For any $\epsilon>0$, by Lemma \ref{le:dis}, we know that above than a positive probability $F_{i,(-)}^{1, a}(2 \epsilon)$, $t_i^{1, a}$ would fall in the interval $\left[\lambda_{a}^{*}-\epsilon, \lambda_{a}^{*}+\epsilon\right]$ for each client $i$. Therefore, by the definition of $\varepsilon$-approximate quantile, we have at most with probability $\prod_{i=1}^S\left(1-F_{i,(-)}^{1,0}(2 \epsilon)\right)^{n_i^{1,0}}+\prod_{i=1}^S\left(1-F_{i,(-)}^{1,1}(2 \epsilon)\right)^{n_i^{1,1}}$, there exists $a \in\{0,1\}$ such that all $t_{i,(k)}^{1, a}$ fall out of $\left[\lambda_{a}^{*}-\epsilon, \lambda_{a}^{*}+\epsilon\right]$. Thus, with probability $1-\prod_{i=1}^S\left(1-F_{i(-)}^{1,0}(2 \epsilon)\right)^{n_i^{1,0}}-\prod_{i=1}^S\left(1-F_{i(-)}^{1,1}(2 \epsilon)\right)^{n_i^{1,1}}$, for $a \in \{0,1\}$, there would exist i such that there exists at least one $t_i^{1, a}$ in $\left[\lambda_{a}^{*}-\epsilon, \lambda_{a}^{*}+\epsilon\right]$. So with $1-\prod_{i=1}^S\left(1-F_{i(-)}^{1,0}(2 \epsilon)\right)^{n_i^{1,0}}-\prod_{i=1}^S\left(1-F_{i(-)}^{1,1}(2 \epsilon)\right)^{n_i^{1,1}}-\delta(n^{1,0})-\delta(n^{1,1})$, there exist a classifier $\phi_{0}(x, a)=\1\left\{f(x, a)>\hat{t}_{*}^{1, a}\right\}$ such that $\hat{t}_{*}^{1, a} \in \left[\lambda_{a}^{*}-\epsilon-\gamma \varepsilon, \lambda_{a}^{*}+\epsilon+ \gamma \varepsilon \right]$. We also denote $\phi_{0}^{*}(x, a)=\1\left\{f^{*}(x, a)>t_{*}^{1, a}\right\}$. Given the threshold is quite close, we further prove that the accuracy is quite close with a high probability. Actually, we have
\begin{equation}\label{eq:part1}
\begin{aligned}
& \left|\mathbb{P}\left(\phi_{0}(x, a) \neq Y\right)-\mathbb{P}\left(\phi_{\alpha^{\prime}}^{*}(x, a) \neq Y\right)\right| \\
\leq & \left|\mathbb{P}\left(\phi_{0}(x, a) \neq Y\right)-\mathbb{P}\left(\phi_{0}^{*}(x, a) \neq Y\right)\right| +\left|\mathbb{P}\left(\phi_{0}^{*}(x, a) \neq Y\right)-\mathbb{P}\left(\phi_{\alpha^{\prime}}^{*}(x, a) \neq Y\right)\right| \\
\leq & \mathbb{P}\left(t_{*}^{1, a}-\epsilon_{0} \leq f^{*}(x, a) \leq t_{*}^{1, a}+\epsilon_{0}\right) +\mathbb{P}\left(\min \left\{t_{*}^{1, a}, \lambda_{a}^{*}\right\} \leq f^{*}(x, a) \leq \max \left\{t_{*}^{1, a}, \lambda_{a}^{*}\right\}\right) \\
\leq & F_{(+)}^{*}\left(2 \epsilon_{0}\right)+F_{(+)}^{*}\left(\max \left\{t_{*}^{1, a}, \lambda_{a}^{*}\right\}-\min \left\{t_{*}^{1, a}, \lambda_{a}^{*}\right\}\right) \\
\leq & F_{(+)}^{*}\left(2 \epsilon_{0}\right)+2 F_{(+)}^{*}(\epsilon+ \gamma \varepsilon)
\end{aligned}
\end{equation}

with probability $1-\prod_{i=1}^S\left(1-F_{i,(-)}^{1,0}(2 \epsilon)\right)^{n_i^{1,0}}-\prod_{i=1}^S\left(1-F_{i,(-)}^{1,1}(2 \epsilon)\right)^{n_i^{1,1}}-\delta(n^{1,0})-\delta(n^{1,1})$.

Further we point out that

$$\small
\begin{aligned}
& \left| |\operatorname{DEOO}\left(\phi_{0}\right)|-|\operatorname{DEOO}\left(\phi_{\alpha^{\prime}}^{*}\right)| \right| \\
& \leq \left||\operatorname{DEOO}\left(\phi_{0}\right)|-| \operatorname{DEOO}\left(\phi_{0}^{*}\right)| +| D E O O\left(\phi_{0}^{*}\right)| -| \operatorname{DEOO}\left(\phi_{\alpha^{\prime}}^{*}\right)|\right| \\
&=\big|| \mathbb{P}(f>t_{*}^{1,0} \mid Y=1, A=0)-\mathbb{P}(f>t_{*}^{1,1} \mid Y=1, A=1) | \\
&-| \mathbb{P}\left(f^{*}>t_{*}^{1,0} \mid Y=1, A=0\right)-\mathbb{P}\left(f^{*}>t_{*}^{1,1}\mid Y=1, A=1\right) |\big| \\ &+\big| | \mathbb{P}\left(f^{*}>t_{*}^{1,0} \mid Y=1, A=0\right)- \mathbb{P}\left(f^{*}>t_{*}^{1,1} \mid Y=1, A=1\right) | \\ 
&-| \mathbb{P}\left(f^{*}>\lambda_{0}^{*} \mid Y=1, A=0\right)-\mathbb{P}\left(f^{*}>\lambda_{1}^{*} \mid Y=1, A=1\right) |\big| \\
&\leq \big| \mathbb{P}\left(f>t_{*}^{1,0} \mid Y=1, A=0\right)-\mathbb{P}\left(f^{*}>t_{*}^{1,0} \mid Y=1,A=0\right)\big| \\
&+\big|\mathbb{P}\left(f>t_{*}^{1,1} \mid Y=1, A=1\right)-\mathbb{P}\left(f^{*}>t_{*}^{1,1}\mid Y=1, A=1\right)\big| \\
&+\big||\mathbb{P}\left(f^{*}>t_{*}^{1,0} \mid Y=1, A=0\right) -\mathbb{P}\left(f^{*}>t_{*}^{1,1} \mid Y=1, A=1\right) | \\
&-| \mathbb{P}\left(f^{*}>\lambda_{0}^{*} \mid Y=1, A=0\right)-\mathbb{P}\left(f^{*}>\lambda_{1}^{*} \mid Y=1, A=1\right) |\big| \\
\end{aligned}
$$
$$
\begin{aligned}\small
& \leq \mathbb{P}\left(t_{*}^{1,0}-\epsilon_{0} \leq f^{*}(x, a) \leq t_{*}^{1,0}+\epsilon_{0}\right)
+\mathbb{P}\left(t_{*}^{1,1}-\epsilon_{0} \leq f^{*}(x, a) \leq t_{*}^{1,1}+\epsilon_{0}\right) \\
&+|\mathbb{P}\left(f^{*}>t_{*}^{1,0} \mid Y=1, A=0\right) -\mathbb{P}\left(f^{*}>t_{*}^{1,1} \mid Y=1, A=1\right) \\
&-\mathbb{P}\left(f^{*}>\lambda_{0}^{*} \mid Y=1, A=0\right) +\mathbb{P}\left(f^{*}>\lambda_{1}^{*} \mid Y=1, A=1\right) | \\
& \leq 2 F_{(+)}^{*}\left(2 \epsilon_{0}\right) +\mathbb{P}\left(\min \left\{t_{*}^{1, a}, \lambda_{a}^{*}\right\} \leq f^{*}(x, a) \leq \max \left\{t_{*}^{1, a}, \lambda_{a}^{*}\right\}\right) \\
& \leq 2 F_{(+)}^{*}\left(2 \epsilon_{0}\right) +F_{(+)}^{*}\left(\max \left\{t_{*}^{1, a}, \lambda_{a}^{*}\right\}-\min \left\{t_{*}^{1, a}, \lambda_{a}^{*}\right\}\right) \\
& \leq 2 F_{(+)}^{*}\left(2 \epsilon_{0}\right) +2 F_{(+)}^{*}(\epsilon+\gamma \varepsilon) \\
\end{aligned}
$$
% Note the only relax in the proof of \ref{prop:deoo} when f has a continuous density function is the relax in local rank. Therefore, similar to proposition \ref{prop:deoo}, one can prove that 
% $$
% \begin{aligned}
%     \mathbb{P}(|D E O O(\phi)|>\alpha) \geq &\mathbb{P}\left(\sum_{i=1}^S \pi_i^{1,1} Q\left(m_i^{1,1}, n_i^{1,1}+1-m_i^{1,1}\right) - \sum_{i=1}^S \pi_i^{1,0} Q\left(M_i^{1,0}, n_i^{1,0}+1-M_i^{1,0}\right) > \alpha\right) \\
%     &+ \mathbb{P}\left(\sum_{i=1}^S \pi_i^{1,1} Q\left(M_i^{1,1}, n_i^{1,1}+1-M_i^{1,1} \right) - \sum_{i=1}^S \pi_i^{1,0} Q\left(m_i^{1,0}, n_i^{1,0}+1-m_i^{1,0}\right) < -\alpha\right) \\
% \end{aligned}
% $$
% Thus, by the definition of Total Variation distance, defining $ \rh(\va,\vb)=\sum_{i=1}^S \pi_i^{1,1} Q\left(a_i^{1,1}, n_i^{1,1}+1-a_i^{1,1}\right) - \sum_{i=1}^S \pi_i^{1,0} Q\left(b_i^{1,0}, n_i^{1,0}+1-b_i^{1,0}\right)
% $
Thus, we know that 
$$
\begin{aligned}
    &\left|\operatorname{DEOO}\left(\phi_{0}\right)\right| \leq|D E O O\left(\phi_{\alpha^{\prime}}^{*}\right)|+2 F_{(+)}^{*}\left(2 \epsilon_{0}\right) +2 F_{(+)}^{*}(\epsilon+\gamma \varepsilon)
   =\alpha^{\prime}+2 F_{(+)}^{*}\left(2 \epsilon_{0}\right)+2 F_{(+)}^{*}(\epsilon+\gamma \varepsilon)
\end{aligned}
$$
% $$
% \begin{aligned}
%    L(\vk_*^{1,0},\vk*^{1,1}) &\leq \left|\operatorname{DEOO}\left(\phi_{0}\right)\right| + 2d_{TV}(\rh(\bm{M},\vm),\rh(\vm,\bm{M}) )\\
%    &\leq|D E O O\left(\phi_{\alpha^{\prime}}^{*}\right)|+2 F_{(+)}^{*}\left(2 \epsilon_{0}\right)+2 F_{(+)}^{*}(\epsilon+\gamma \varepsilon)+ 2d_{TV}(\rh(\bm{M},\vm),\rh(\vm,\bm{M}) )\\
%    &=\alpha^{\prime}+2 F_{(+)}^{*}\left(2 \epsilon_{0}\right)+2 F_{(+)}^{*}(\epsilon+\gamma \varepsilon)+ 2d_{TV}(\rh(\bm{M},\vm),\rh(\vm,\bm{M}) )
% \end{aligned}
% $$
If $F_{(+)}^{*}(\epsilon+\gamma \varepsilon) \leq \frac{\alpha-\alpha^{\prime}}{2}-F_{(+)}^{*}\left(2 \epsilon_{0}\right)$, then there will exist at least one feasible classifier in the candidate set.

On the other hand, we could prove that the output classifier is quite similar with $\phi_{0}$ we mentioned above.

By Proposition \ref{prop:error}, for any $\phi \in K$, $\hat{q}^i_{Y, a}=1-\hat{p}^i_{Y, a}$, we have
\begin{equation}\label{eq:prob_estimation}\small
    \begin{aligned}
        &\Bigg| \mathbb{P}\left(\phi(x, a) \neq Y\right)-\sum_{i=1}^S \pi_i \Big[\frac{\hat{k}_{i}^{1,0}+0.5}{n_i^{1,0}+1} p^i_{0} p^i_{Y, 0} +\frac{\hat{k}_{i}^{1,1}+0.5}{n_i^{1,1}+1} p^i_{1} p^i_{Y, 1}+ \frac{n_i^{0,0}+0.5-\hat{k}_{i}^{0,0}}{n_i^{0,0}+1} p^i_{0}q^i_{Y, 0} + \frac{n_i^{0,1}+0.5-\hat{k}_{i}^{0,1}}{n_i^{0,1}+1} p^i_{1}q^i_{Y, 1}
    \Big]    \Bigg| \leq \theta
    % \sum_{i=1}^S \pi_i \left[e_i^{0,0} p^i_{0}\left(1-p^i_{Y, 0}\right)
    % +e_i^{0,1} p^i_{0}p^i_{Y, 0} +e_i^{1,0} p^i_{1}\left(1-p^i_{Y, 1}\right)
    % +e_i^{1,1} p^i_{1}p^i_{Y, 1}\right],
    \end{aligned}
    \end{equation} 
%where $e_i^{y,a}=\frac{(2\lfloor \varepsilon n_i^{y,a} \rfloor +1)}{2\left(n_i^{y,a}+1\right)}$. 
Therefore, we only need to check the influence induced by using $\hat{p}^i_{a}$ and $\hat{p}^i_{Y, a}$, instead of $p^i_{0}$ and $p^i_{Y, 0}$. In detail, we point out this influence can be estimated by Hoeffding's inequality as follow:

Since $\hat{p}^i_{a}=\frac{n_i^{1, a}+n_i^{0, a}}{n_i}$ and $\hat{p}^i_{Y, a}=\frac{n_i^{1, a}}{n_i^{0,a}+n_i^{1,a}}$, we have 
$\frac{n_i^{1, a}+n_i^{0, a}}{n_i}=\frac{\sum_{j=1}^{n_i} \1\left\{Z_{j}^{a}=1\right\}}{n}$ and 
$\frac{n_i^{1, a}}{n_i^{0,a}+n_i^{1,a}}=\frac{\sum_{j=1}^{n_i^{0,a}+n_i^{1,a}} \1\left\{Z_{j}^{Y, a}=1\right\}}{n_i^{0,a}+n_i^{1,a}}$
, where $Z_{j}^{a} \sim B\left(1, p^i_{a}\right)$ and 
$Z_{j}^{Y, a} \sim B\left(1, p^i_{Y, a}\right)$. 

Thus, from Hoeffding's inequality, we have 

$$
\mathbb{P}\left(\left|\hat{p}^i_{a}-p^i_{a}\right| \geq \sqrt{\frac{n_i^{0, a}}{n_i}} \epsilon\right) \leq 2 e^{-2 n_i^{0, a} \epsilon^{2}}
$$

For the same reason, we have we have

$$
\mathbb{P}\left(\left|\hat{p}^i_{Y, a}-p^i_{Y, a}\right| \geq \sqrt{\frac{n_i^{0, a}}{n_i}} \epsilon\right) \leq 2 e^{-2 n_i^{0,a} \epsilon^{2}}
$$

So, we have with probability $1- 4\sum_{i=1}^{S} e^{-2 n_i^{0,a} \epsilon^{2}}$%-2 \sum_{i=1}^{S}e^{-2 n_i^{0, a} \epsilon^{2}}$
$$
\left\{\begin{aligned}
\left|\hat{p}^i_{a}-p^i_{a}\right| & \leq \sqrt{\frac{n_i^{0, a}}{n_i}} \epsilon \\
\left|\hat{p}^i_{Y, a}-p^i_{Y, a}\right| & \leq \sqrt{\frac{n_i^{0, a}}{n_i^{*,a}}} \epsilon 
\end{aligned}\right.,
$$
where $n_i^{*,a}=(n_i^{0,a}+n_i^{1,a})$.

Thus, with probability $1- 4\sum_{a=0}^{1}\sum_{i=1}^{S} e^{-2 n_i^{0,a} \epsilon^{2}}$, %-M \sum_{i=1}^{S}e^{-2 n_i^{0, a} \epsilon^{2}}$
\begin{equation} \label{eq:part2}\small
\begin{aligned}
& \left|\mathbb{P}\left(\hat{\phi}_{i}(x, a) \neq Y\right)-\hat{\mathbb{P}}\left(\hat{\phi}_{i}(x, a) \neq Y\right)\right| \\
&\leq  \Bigg|\sum_{i=1}^S \pi^i\Big[\frac{\hat{k}_{i}^{1,0}+0.5}{n_i^{1,0}+1} p^i_{0} p^i_{Y, 0}+\frac{\hat{k}_{i}^{1,1}+0.5}{n_i^{1,1}+1} p^i_{1} p^i_{Y, 1}+\frac{n_i^{0,0}+0.5-\hat{k}_{i}^{0,0}}{n_i^{0,0}+1} p^i_{0}q^i_{Y, 0}+\frac{n_i^{0,1}+0.5-\hat{k}_{i}^{0,1}}{n_i^{0,1}+1} p^i_{1}q^i_{Y, 1} \Big] \\
&-\sum_{i=1}^S \pi^i\Big[\frac{\hat{k}_{i}^{1,0}+0.5}{n_i^{1,0}+1} \hat{p}^i_{0} \hat{p}^i_{Y, 0}+\frac{\hat{k}_{i}^{1,1}+0.5}{n_i^{1,1}+1} \hat{p}^i_{1} \hat{p}^i_{Y, 1}+ \frac{n_i^{0,0}+0.5-\hat{k}_{i}^{0,0}}{n_i^{0,0}+1} \hat{p}^i_{0}\hat{q}^i_{Y, 0}+\frac{n_i^{0,1}+0.5-\hat{k}_{i}^{0,1}}{n_i^{0,1}+1} \hat{p}^i_{1}\hat{q}^i_{Y, 1}\Big] \Bigg| \\
& +\sum_{i=1}^S \pi_i \left[e_i^{0,0} p^i_{0}q^i_{Y, 0}
    +e_i^{0,1} p^i_{0}p^i_{Y, 0} +e_i^{1,0} p^i_{1}q^i_{Y, 1}
    +e_i^{1,1} p^i_{1}p^i_{Y, 1}\right] \\
&=\sum_{i=1}^S \pi_i \left[e_i^{0,0} p^i_{0}q^i_{Y, 0}
    +e_i^{0,1} p^i_{0}p^i_{Y, 0} +e_i^{1,0} p^i_{1}q^i_{Y, 1}
    +e_i^{1,1} p^i_{1}p^i_{Y, 1}\right] +|\sum_{i=1}^S \pi^i (A_i-\hat{A_i})| \\
\end{aligned}
\end{equation}
 For $A_i-\hat{A_i}$, we have
 
\begin{equation} \label{eq:part3}\small
\begin{aligned}
&A_i-\hat{A_i} \leq \epsilon\left[\sqrt{\frac{n_{i}^{0,0}}{n^{*,0}_{i}}} \frac{\hat{k}_{i}^{1,0}+0.5}{n^{1,0}+1}\left(p^i_{0}+p^i_{Y, 0}\right) + \sqrt{\frac{n_{i}^{0,1}}{n^{*,1}_{i}}} \frac{\hat{k}_{i}^{1,1}+0.5}{n^{1,1}+1}\left(p^i_{1}+p^i_{Y, 1}\right)\right] \\ &+\epsilon^{2}\left(\frac{n_{i}^{0,0}}{n^{*,0}_{i}} \frac{\hat{k}_{i}^{1,0}+0.5}{n^{1,0}+1}+\frac{n_{i}^{0,1}}{n^{*,1}_{i}} \frac{\hat{k}_{i}^{1,1}+0.5}{n^{1,1}+1}\right)  +\frac{n^{0,0}+0.5-\hat{k}_{i}^{0,0}}{n^{0,0}+1} \sqrt{\frac{n_{i}^{0,0}}{n^{*,0}_{i}}} \epsilon\left[\sqrt{\frac{n_{i}^{0,0}}{n^{*,0}_{i}}} \epsilon+p^i_{0}+p^i_{Y, 0}+1\right] \\
& +\frac{n^{0,1}+0.5-\hat{k}_{i}^{0,1}}{n^{0,1}+1} \sqrt{\frac{n_{i}^{0,1}}{n^{*,1}_{i}}} \epsilon\left[\sqrt{\frac{n_{i}^{0,1}}{n^{*,1}_{i}}} \epsilon+p^i_{1}+p^i_{Y, 1}+1\right] \\
&\leq \epsilon\left[\sqrt{\frac{n_{i}^{0,0}}{n^{*,0}_{i}}}\left(p^i_{0}+p^i_{Y, 0}\right)+\sqrt{\frac{n_{i}^{0,1}}{n^{*,1}_{i}}}\left(p^i_{1}+p^i_{Y, 1}\right)\right] +\epsilon^{2}\left(\frac{n_{i}^{0,0}}{n^{*,0}_{i}}+\frac{n_{i}^{0,1}}{n^{*,1}_{i}}\right) +\sqrt{\frac{n_{i}^{0,0}}{n^{*,0}_{i}}} \epsilon\left[\sqrt{\frac{n_{i}^{0,0}}{n^{*,0}_{i}}} \epsilon+p^i_{0}+p^i_{Y, 0}+1\right] \\
& +\sqrt{\frac{n_{i}^{0,1}}{n^{*,1}_{i}}} \epsilon\left[\sqrt{\frac{n_{i}^{0,1}}{n^{*,1}_{i}}} \epsilon+p^i_{1}+p^i_{Y, 1}+1\right] \\
 &\leq 4 \epsilon+2\epsilon^{2}+2\epsilon^{2}+6\epsilon  \\
&=  4\epsilon^{2}+10 \epsilon
\end{aligned}
\end{equation}
Combining Inequality \ref{eq:part1}-\ref{eq:part3}, we complete the proof.
\end{proof}

\subsection{Detailed Theory for Label Shift Case}\label{appendix:shift}
\begin{proposition}\label{prop:error_shift}
    Under Assumption \ref{ass:shift}, the misclassification error can be estimated by
    \begin{equation}
    \begin{aligned}
    &\hat{\mathbb{P}}\left(\hat{\phi}(x, a) \neq Y\right)= \sum_{i=1}^S \pi_i \Big[\frac{\hat{k}_{i}^{1,0}+0.5}{n_i^{1,0}+1} p^i_{0} p^i_{Y, 0}w^{1,0} \\
    &+\frac{\hat{k}_{i}^{1,1}+0.5}{n_i^{1,1}+1} p^i_{1} p^i_{Y, 1}w^{1,1}+\frac{n_i^{0,0}+0.5-\hat{k}_{i}^{0,0}}{n_i^{0,0}+1} p^i_{0}q^i_{Y, 0}w^{0,0} \\ &
     +\frac{n_i^{0,1}+0.5-\hat{k}_{i}^{0,1}}{n_i^{0,1}+1} p^i_{1}q^i_{Y, 1}w^{0,1}
    \Big],    
    \end{aligned} %\label{eq:error_shift}
    \end{equation}
    where $w^{y,a}= \frac{p^{S+1}_a p^{S+1}_{Y,a}}{p_a p_{Y,a}}$. 
    Further, discrepancy between empirical error and true error is limited by following inequality:
    \begin{equation}
    \begin{aligned}
        \left| \mathbb{P}\left(\hat{\phi}(x, a) \neq Y\right)-\hat{\mathbb{P}}\left(\hat{\phi}(x, a) \neq Y\right) \right| \leq \theta^\prime
    %& 
    % \sum_{i=1}^S \pi_i \left[e_i^{0,0} p^i_{0}\left(1-p^i_{Y, 0}\right)w^{0,0}
    % +e_i^{0,1} p^i_{0}p^i_{Y, 0}w^{0,1} \right.\\
    % &\left.+e_i^{1,0} p^i_{1}\left(1-p^i_{Y, 1}\right)w^{1,0}
    % +e_i^{1,1} p^i_{1}p^i_{Y, 1}w^{1,1}\right], 
    \end{aligned}
    \end{equation} 
    where $e_i^{y,a}=\frac{2\lfloor \varepsilon n_i^{y,a} \rfloor +1}{2\left(n_i^{y,a}+1\right)}$ and 
    $\theta^\prime=\sum_{i=1}^S \pi_i \big[e_i^{0,0} p^i_{0}q^i_{Y, 0}w^{0,0}
     +e_i^{0,1}w^{0,1} p^i_{0}p^i_{Y, 0}
  +e_i^{1,0}w^{1,0} p^i_{1}q^i_{Y, 1}
     $
     $+e_i^{1,1}w^{1,1} p^i_{1}p^i_{Y, 1}\big]$.
\end{proposition}%prop:error

\begin{proof}[Proof for Proposition \ref{prop:error_shift}]
Note the classifier is
$$
\phi=\left\{\begin{array}{l}
\1\left\{f(x, 0)>\hat{t}_{\left(k^{1,0}\right)}^{1,0}\right\}, a=0 \\
\1\left\{f(x, 1)>\hat{t}_{\left(k^{1,1}\right)}^{1,1}\right\}, a=1
\end{array}\right.
$$
So we can calculate the mis-classification error in $P_{S+1}$. Denoted $\bP_{S+1}$ the probability measure under the $P_{S+1}$ distribution, we have:

\begin{equation}\small%\label{eq:prop-error}
\begin{aligned}
& \mathbb{P}_{S+1}(Y \neq \hat{Y})=\mathbb{P}_{S+1}(Y=1, \hat{Y}=0)+\mathbb{P}_{S+1}(Y=0, \hat{Y}=1) \\
& =\mathbb{P}_{S+1}(Y=1, \hat{Y}=0, A=0)+\mathbb{P}_{S+1}(Y=1, \hat{Y}=0, A=1) \\
& +\mathbb{P}_{S+1}(Y=0, \hat{Y}=1, A=0)+\mathbb{P}_{S+1}(Y=0, \hat{Y}=1, A=1) \\
&= \mathbb{P}(Y=1, \hat{Y}=0, A=0 \mid (X,Y,A) \sim P_{S+1}) +\mathbb{P}(Y=1, \hat{Y}=0, A=1\mid (X,Y,A) \sim P_{S+1}) \\
& +\mathbb{P}(Y=0, \hat{Y}=1, A=0\mid (X,Y,A) \sim P_{S+1})  +\mathbb{P}(Y=0, \hat{Y}=1, A=1\mid (X,Y,A) \sim P_{S+1}) \\
&= \mathbb{P}( \hat{Y}=0 \mid Y=1,A=0 )p^{S+1}_0 p^{S+1}_{Y,0}  +\mathbb{P}(\hat{Y}=0 \mid Y=1, A=1)p^{S+1}_1 p^{S+1}_{Y,1} \\
& +\mathbb{P}(\hat{Y}=1 \mid Y=0, A=0)p^{S+1}_0 (1-p^{S+1}_{Y,0}) +\mathbb{P}(\hat{Y}=1 \mid Y=0,  A=1)p^{S+1}_1 (1-p^{S+1}_{Y,1}) \\
&=\sum_{i=1}^S \pi^{1,0}_i \mathbb{P}_i( \hat{Y}=0 \mid Y=1,A=0 )p^{S+1}_0 p^{S+1}_{Y,0} +\sum_{i=1}^S \pi^{1,1}_i \mathbb{P}(\hat{Y}=0 \mid Y=1, A=1)p^{S+1}_1 p^{S+1}_{Y,1} \\
&+\sum_{i=1}^S \pi^{0,0}_i \mathbb{P}(\hat{Y}=1 \mid Y=0, A=0)p^{S+1}_0 (1-p^{S+1}_{Y,0}) +\sum_{i=1}^S \pi^{0,1}_i \mathbb{P}(\hat{Y}=1 \mid Y=0,  A=1)p^{S+1}_1 (1-p^{S+1}_{Y,1}) \\
& =\sum_{i=1}^S \pi_i \big[w^{0,0}\mathbb{P}_i(Y=1, \hat{Y}=0, A=0)+w^{0,1}\mathbb{P}_i(Y=1, \hat{Y}=0, A=1) \\
&+w^{1,0}\mathbb{P}_i(Y=0, \hat{Y}=1, A=0) + w^{1,1}\mathbb{P}_i(Y=0, \hat{Y}=1, A=1) \big]
\end{aligned}   
\end{equation}

Then, since estimating $\mathbb{P}_i(Y=0, \hat{Y}=y, A=a)$ shares similarities with the approach outlined in Proposition \ref{prop:error}. This similarity in the estimation process allows us to successfully complete our proof.
\end{proof}

Given proof for Proposition \ref{prop:error_shift}, proof for Theorem \ref{thm:label} is similar to Proof for Theorem \ref{thm:acc}
\begin{algorithm}[t]
\caption{FedFaiREE for label shift case}\label{alg:label_shift}
\textbf{Input:} Train dataset $D_i = D_i^{0,0} \cup D_i^{0,1} \cup D_i^{1,0} \cup D_i^{1,1}$; pre-trained classifier $\phi_0$ with function f; fainess constraint parameter $\alpha$ ; Confidence level parameter $\beta$; Weights of different clients $\pi$ 
\textbf{Output:} classifier $\hat{\phi}(x,a)= \1 
\{ f(x,a) >  t_{(k^{1,a})}^{1,a}\}$
\begin{algorithmic}[1]
\STATE{\bf Client Side:}
\FOR{i=1,2,..,$S$}
\STATE{Score on train data points in $D_i$ and get $T_i^{y,a}=\{t_{i,1}^{y,a},t_{i,2}^{y,a}, \cdots, t_{i,n_i^{y,a}}^{y,a}\}$ }
\STATE{Sort $T_i^{y,a}$}
\STATE{Calculate q-digest of $T_i^{y,a}$ on client $i$}
\STATE{Update digest to server}
\ENDFOR
\STATE{\bf Server Side:}
%\STATE{Combine q-digests }%and get the sketch of sorted $T^{y,a}$
\STATE{Construct $K$ by $K=\{(k^{1,0}, k^{1,1})| L(\vk^{1,0}, \vk^{1,1})< 1-\beta \}$}
\STATE{Select optimal $(k_0,k_1)$ by minimizing equation \ref{eq:error_shift} using estimated values $\hat{p}^i_{a} = \frac{n_i^{0,a}+n_i^{1,a}}{n_i^{0,0}+n_i^{0,1}+n_i^{1,0}+n_i^{1,1}}$ and $\hat{p}^i_{Y,a} = \frac{n_i^{1,a}}{n_i^{0,a}+n_i^{1,a}}$}
\end{algorithmic}
\end{algorithm}
\subsection{Detailed Theory for DEO}\label{appendix:deo}
\begin{proposition}\label{prop:deo}
Under Assumption \ref{ass:mix}, for $a \in \{0,1\}$, consider $k^{1, a} \in \{1, \ldots, n^{1, a}\}$, the corresponding $\hat{k}_i^{1,a}$ for $i \in [S]$ which are $\varepsilon$-approximate ranks and the score-based classifier $\phi(x, a) = \1\{f(x, a) > t_{(k^{1, a})}^{1, a}\}$ . 
Define 
\begin{small}
$$\begin{aligned}
 &h_{y,a}(\vu,\vv)=\mathbb{P}\big(\sum_{i=1}^S \pi_i^{y,a} Q\left(u_i, n_i^{y,a}+1-u_i\right) - \sum_{i=1}^S \pi_i^{y,1-a} Q\left(v_i, n_i^{y,1-a}+1-v_i\right) \geq \alpha\big).   
\end{aligned}$$\end{small}
Then we have:
\begin{equation}\small
\begin{aligned}
    \mathbb{P}(|DEO(\phi)| \preceq (\alpha,\alpha)) \geq& 1-h^*_{1,1}-h^*_{1,0}-h^*_{0,1}-h^*_{0,0}
\end{aligned}
\label{eq:adloss_DEO}
\end{equation}
where the definitions of $M_i^{y,a}$, $m_i^{y,a}$, $\pi_i^{y,a}$, $Q(A,B)$ are similar to Proposition \ref{prop:adloss},$h^*_{1,1}=h_{y,a}(\bm{M}^{y,a},\vm^{y,a})$
%where $M_i^{y,a}=\lceil \hat{k}_i^{y,a}+\varepsilon n_i^{y,a}\rceil$, $m_i^{y,a}=\lceil \hat{k}_i^{y,a}-\varepsilon n_i^{y,a}\rceil$, $\pi_i^{y,a} = \mathbb{P}(\text{sampling } x \text{ from client } i \mid \text{ sampling } x \text{ with label } Y=y \text{ and } A=a)$, and $Q(A,B)$ are independent random variables following Beta distribution, $Q(A,B) \sim \text{Beta}(A, B)$. 
\end{proposition}
\begin{proof}[Proof of Proposition \ref{prop:deo}]
 Note the output classifier is
 $$
\phi=\left\{\begin{array}{l}
\1\left\{f(x, 0)>\hat{t}_{\left(k^{1,0}\right)}^{1,0}\right\}, a=0 \\
\1\left\{f(x, 1)>\hat{t}_{\left(k^{1,1}\right)}^{1,1}\right\}, a=1
\end{array}\right.
$$

we have:

$$
\begin{aligned}\small
&\mathbb{P}(|D E O(\phi)| \preceq (\alpha,\alpha) ) \\
&\geq 1-\bP\left(\left|F^{1,1}\left(t_{\left(k^{1,1}\right)}^{1,1}\right)-F^{1,0}\left(t_{\left(k^{1,0}\right)}^{1,0}\right)\right|>\alpha\right) -\bP\left(\left|F^{0,1}\left(t_{\left(k^{1,1}\right)}^{1,1}\right)-F^{0,0}\left(t_{\left(k^{1,0}\right)}^{1,0}\right)\right|>\alpha\right) \\
& =1 -\mathbb{P}\left(\sum_{i=1}^S \pi_i^{1,1} F_i^{1,1}\left(t_{\left(k^{1,1}\right)}^{1,1}\right)-\sum_{i=1}^S \pi_i^{1,0} F_i^{1,0}\left(t_{\left(k^{1,0}\right)}^{1,0}\right)>\alpha\right) \\
& -\mathbb{P}\left(\sum_{i=1}^S \pi_i^{1,1} F_i^{1,1}\left(t_{\left(k^{1,1}\right)}^{1,1}\right)- \sum_{i=1}^S \pi_i^{1,0} F_i^{1,0}\left(t_{\left(k^{1,0}\right)}^{1,0}\right)<-\alpha\right) \\
& -\mathbb{P}\left(\sum_{i=1}^S \pi_i^{0,1} F_i^{0,1}\left(t_{\left(k^{1,1}\right)}^{1,1}\right)-\sum_{i=1}^S \pi_i^{0,0} F_i^{0,0}\left(t_{\left(k^{1,0}\right)}^{1,0}\right)>\alpha\right) \\
& -\mathbb{P}\left(\sum_{i=1}^S \pi_i^{0,1} F_i^{0,1}\left(t_{\left(k^{1,1}\right)}^{1,1}\right)- \sum_{i=1}^S \pi_i^{0,0} F_i^{0,0}\left(t_{\left(k^{1,0}\right)}^{1,0}\right)<-\alpha\right) \\
\end{aligned}
$$

The remainder of the proof is similar to the proof for Proposition \ref{prop:deoo} 
    
\end{proof}
Building upon Proposition \ref{prop:deo}, we can further prove Theorem \ref{thm:acc_eo} using a similar approach as in Theorem \ref{thm:acc}.

\begin{algorithm}[t]
\caption{FedFaiREE for DEO}\label{alg:deo}
\textbf{Input:} Train dataset $D_i = D_i^{0,0} \cup D_i^{0,1} \cup D_i^{1,0} \cup D_i^{1,1}$; pre-trained classifier $\phi_0$ with function f; fairness constraint parameter $\alpha$ ; Confidence level parameter $\beta$; Weights of different clients $\pi$ 
\textbf{Output:} classifier $\hat{\phi}(x,a)= \1 
\{ f(x,a) >  t_{(k^{1,a})}^{1,a}\}$
\begin{algorithmic}[1]
\STATE{\bf Client Side:}
\FOR{i=1,2,..,$S$}
\STATE{Score on train data points in $D_i$ and get $T_i^{y,a}=\{t_{i,1}^{y,a},t_{i,2}^{y,a}, \cdots, t_{i,n_i^{y,a}}^{y,a}\}$ }
\STATE{Sort $T_i^{y,a}$}
\STATE{Calculate q-digest of $T_i^{y,a}$ on client $i$}
\STATE{Update digest to server}
\ENDFOR
\STATE{\bf Server Side:}
%\STATE{Combine q-digests }%and get the sketch of sorted $T^{y,a}$
\STATE{Construct $K$ by $K=\{(k^{1,0}, k^{1,1})| L(\vk^{1,0}, \vk^{1,1})< 1-\beta \}$, where L is defined in Equation \ref{eq:construct K deo}}
\STATE{Select optimal $(k_0,k_1)$ by minimizing equation \ref{eq:error_p} using estimated values $\hat{p}^i_{a} = \frac{n_i^{0,a}+n_i^{1,a}}{n_i^{0,0}+n_i^{0,1}+n_i^{1,0}+n_i^{1,1}}$ and $\hat{p}^i_{Y,a} = \frac{n_i^{1,a}}{n_i^{0,a}+n_i^{1,a}}$}
\end{algorithmic}
\end{algorithm}

\subsection{Detailed Theory for Multi-Groups Case}\label{appendix:multigroup}
\begin{definition}(Equality of Opportunity, Multiple Groups)
A classifier satisfies Equality of Opportunity if it satisfies the same true positive rate among protected groups:
$$\bP_{X \mid A=0, Y=1}(\widehat{Y}=1)=\bP_{X \mid A=a, Y=1}(\widehat{Y}=1),$$
where $a$ belongs to a protected class $\mathcal{A} = \{1,\cdots,A_0\}$
\end{definition}
% \begin{definition} 
% To be specific, given a classifier $\phi$, the $\alpha$ difference tolerance in Equality of Opportunity can be defined as:
% \begin{equation}
% |DEOO|\leq \alpha, 
% \label{deoo}
% \end{equation}
% where $DEOO = \bP_{X \mid A=1, Y=1}(\widehat{Y}=1)-\bP_{X \mid A=0, Y=1}(\widehat{Y}=1)$. 
% \end{definition}
Similar to $DEOO$, we define metric for Equality of Opportunity under Multiple Groups as:
$$
\begin{aligned}
DEOOM = &\max_{a}\{|\bP_{X \mid A=a, Y=1}(\widehat{Y}=1)- \bP_{X \mid A=0, Y=1}(\widehat{Y}=1)|\}    
\end{aligned}
$$

Therefore, inspired by Proposition \ref{prop:adloss}, we have

\begin{proposition}
\label{prop:deoom}
Under Assumption \ref{ass:mix}, for $a \in \{0,1,\cdots,A_0\}$, consider $k^{1, a} \in \{1, \ldots, n^{1, a}\}$, the corresponding $\hat{k}_i^{1,a}$ for $i \in [S]$ which are $\varepsilon$-approximate ranks and the score-based classifier $\phi(x, a) = \1\{f(x, a) > t_{(k^{1, a})}^{1, a}\}$ . Define 
\begin{small}
$$\begin{aligned}
    h^*_{y,a}=&\mathbb{P}\left(\sum_{i=1}^S \pi_i^{y,a} Q\left(M^{1,a}_i, n_i^{y,a}+1-M^{1,a}_i\right)  -\sum_{i=1}^S \pi_i^{y,0} Q\left(m^{1,0}_i, n_i^{y,0}+1-m^{1,0}_i\right) \geq \alpha\right) \\
    &+\mathbb{P}\left(\sum_{i=1}^S \pi_i^{y,0} Q\left(M^{1,0}_i, n_i^{y,0}+1-M^{1,0}_i\right)- \sum_{i=1}^S \pi_i^{y,a} Q\left(m^{1,a}_i, n_i^{y,a}+1-m^{1,a}_i\right) \geq \alpha\right)
\end{aligned}.$$
\end{small} Then we have:
\begin{equation}
\small
\begin{aligned}
    \mathbb{P}(|DEOOM(\phi)| > \alpha) \leq \sum_{a=1}^{A_0} h^*_{1,a}
\end{aligned}
\label{eq:multi_loss}
\end{equation}
where $\pi_i^{1,a}$, $\pi_i^{1,0}$ are similarly defined as in Proposition \ref{prop:adloss}. $M_i^{1,a}=max\big(\lceil \hat{k}_i^{1,a}+\varepsilon n_i^{1,a}\rceil, n_i^{1,a}+1\big)$, $m_i^{1,a}=min\big(\lceil \hat{k}_i^{1,a}-\varepsilon n_i^{1,a}\rceil,0\big)$, $M_i^{1,0}$ and $m_i^{1,0}$ are similarly defined. $Q(\alpha,\beta)$ are independent random variables and $Q(\alpha,\beta) \sim \text{Beta}(\alpha, \beta)$. Especially, we define $Q(0,\beta)=0$ and $Q(\alpha,0)=1$ for $\alpha,\beta \neq 0$. 
\end{proposition}

Proposition \ref{prop:deoom} can be regarded as a direct corollary of Proposition \ref{prop:adloss}. Moveover, similar to Proposition \ref{prop:error}, we have

\begin{proposition}\label{prop:error_deoom}
    Under Assumption \ref{ass:mix}, the misclassification error can be estimated by
    \begin{equation}\label{eq:error_deoom}
    \small
    \begin{aligned}    \hat{\mathbb{P}}\left(\hat{\phi}(x, a) \neq Y\right)=&\sum_{i=1}^S \Big[\pi_i \sum_{a=0}^{A_0}\big(\frac{\hat{k}_{i}^{1,a}+0.5}{n_i^{1,a}+1} p^i_{a} p^i_{Y, a} + \frac{n_i^{0,a}+0.5-\hat{k}_{i}^{0,a}}{n_i^{0,a}+1} p^i_{a}q^i_{Y, a} \big) \Big]
    \end{aligned}  
    \end{equation}
    Further, the discrepancy between empirical error and true error is upper bounded by the following:
    \begin{equation}\small
    \begin{aligned}
        \left| \mathbb{P}\left(\hat{\phi}(x, a) \neq Y\right)-\hat{\mathbb{P}}\left(\hat{\phi}(x, a) \neq Y\right) \right| \leq \theta,
    \end{aligned}
    \end{equation} 
    where \begin{small}
        $\theta=\sum_{i=1}^S \left[\pi_i \sum_{a=0}^{A_0}\left(e_i^{0,a} p^i_{a}q^i_{Y, a}+e_i^{1,a} p^i_{1}q^i_{Y, a}\right)\right]$, $e_i^{y,a}=\frac{2\lfloor \varepsilon n_i^{y,a} \rfloor +1}{2\left(n_i^{y,a}+1\right)}$.
    \end{small}
\end{proposition}
% \begin{proof}[Proof of Proposition ]
    
% \end{proof}
\begin{algorithm}[h]
\caption{FedFaiREE for Multi-Groups}
\textbf{Input:} Train dataset $D_i = D_i^{0,0} \cup D_i^{0,1} \cup D_i^{1,0} \cup D_i^{1,1}$; pre-trained classifier $\phi_0$ with function f; fairness constraint parameter $\alpha$ ; Confidence level parameter $\beta$; Weights of different clients $\pi$ 
\textbf{Output:} classifier $\hat{\phi}(x,a)= \1 
\{ f(x,a) >  t_{(k^{1,a})}^{1,a}\}$
\begin{algorithmic}[1]
\STATE{\bf Client Side:}
\FOR{i=1,2,..,$S$}
\STATE{Score on train data points in $D_i$ and get $T_i^{y,a}=\{t_{i,1}^{y,a},t_{i,2}^{y,a}, \cdots, t_{i,n_i^{y,a}}^{y,a}\}$ }
\STATE{Sort $T_i^{y,a}$}
\STATE{Calculate q-digest of $T_i^{y,a}$ on client $i$}
\STATE{Update digest to server}
\ENDFOR
\STATE{\bf Server Side:}
%\STATE{Combine q-digests }%and get the sketch of sorted $T^{y,a}$
\STATE{Construct $K$ by $K=\{(k^{1,0}, k^{1,1}, \cdots, k^{1,A_0})| L< 1-\beta \}$, where L is defined by the right-hand side of Inequality \ref{eq:multi_loss}}
\STATE{Select optimal $(k^{1,0}, k^{1,1}, \cdots, k^{1,A_0})$ by minimizing equation \ref{eq:error_deoom} using estimated values $\hat{p}^i_{a}$ and $\hat{p}^i_{Y,a}$}
\end{algorithmic}
\end{algorithm}

\begin{theorem}
Under Assumption \ref{ass:mix} and \ref{ass:quan}, given $\alpha^{\prime}<\alpha$. Suppose $\hat{\phi}$ is the final output of FedFaiREE, we have:

(1) $|D E O O M(\hat{\phi})|<\alpha$ with probability $(1-\delta)^{N}$, where $N$ is the size of the candidate set.

(2) Suppose the density distribution functions of $f^{*}$ under $A=a, Y=1$ are continuous. When the input classifier $f$ satisfies $\left|f(x, a)-f^{*}(x, a)\right| \leq \epsilon_{0}$, for any $\epsilon>0$ such that $F_{(+)}^{*}(\epsilon+\gamma \varepsilon) \leq$ $\frac{\alpha-\alpha^{\prime}}{2}-F_{(+)}^{*}\left(2 \epsilon_{0}\right)$, we have
\begin{equation}\small
\begin{aligned}
        &\mathbb{P}(\hat{\phi}(x, a) \neq Y)-\mathbb{P}\left(\phi_{\alpha^{\prime}}^{*}(x, a) \neq Y\right) \leq  2 F_{(+)}^{*}\left(2 \epsilon_{0}\right)+2 F_{(+)}^{*}(\epsilon+\gamma \varepsilon)+ 2\theta+O(\epsilon)
\end{aligned}
    %     &+\sum_{i=1}^S 2\pi_i \left[e_i^{0,0} p^i_{0}\left(1-p^i_{Y, 0}\right)
    % +e_i^{0,1} p^i_{0}p^i_{Y, 0} +e_i^{1,0} p^i_{1}\left(1-p^i_{Y, 1}\right)
    % +e_i^{1,1} p^i_{1}p^i_{Y, 1}\right] 
\end{equation}
with probability $
1-4\sum_{a=0}^{A_0}\sum_{i=1}^{S} e^{-2 n_i^{0,a} \epsilon^{2}}-\sum_{a=0}^{A_0}\prod_{i=1}^S\big(1-F_{i(-)}^{1,a}(2 \epsilon)\big)^{n_i^{1,a}}-\delta, $
where $\delta=\sum_{a=0}^{A_0}\delta^{1,a}(n^{1,a})$, $\theta$ is defined in Proposition \ref{prop:error_deoom} and the definition of $F_{(+)}$ and $F_{(-)}$ are shown in Lemma \ref{le:dis} 
\end{theorem}

\subsection{Detailed Theory for Multi-Labels Case}
\begin{definition}(Equality of Opportunity, Multiple labels\citep{liu2023simfair})
A classifier satisfies Equality of Opportunity if it satisfies :
$$\hat{\bm{Y}} \perp A \mid \bm{Y}=\bm{y}_{a d v},$$
where $\bm{Y} \in \{0,1\}^m$ and $\bm{y}_{adv}$ denotes some advantaged label where only favorable outcomes.
\end{definition}
\begin{definition}(Multi-label Score-based Classifier)
        A Multi-label score-based classifier is an element-wise indication function, where the j-th component of $\hat{\bm{Y}}$ satisfies $\hat Y_j=\phi_j(x,a) = \1\{f_j(x,a) > c_j\}$ for a measurable score function $f: \mathcal{X} \times\{0,1\} \rightarrow[0,1]$ and a constant threshold $c_j>0$. 
\end{definition}
% $$\mathbb{E}[\widehat{\mathbf{Y}}\mid A=0, \mathbf{Y}=\mathbf{Y}_{adv}]=\mathbf{E}[\widehat{\mathbf{Y}}\mid A=1, \mathbf{Y}=\mathbf{Y}_{adv}],$$
% where $\mathbf{Y} \in \{0,1\}^d$ and $\mathbf{Y}_{adv}$ stands for a prefered label.

% Inspired by \citep{liu2023simfair}, we consider fairness indicatior
% $$DEOOML=\big|\mathbb{E}[\widehat{\mathbf{Y}}\mid A=0, \mathbf{Y}=\mathbf{Y}_{adv}]-\mathbf{E}[\widehat{\mathbf{Y}}\mid A=1, \mathbf{Y}=\mathbf{Y}_{adv}]\big|,$$
% and fairness constraint 
% $$DEOOML \preceq \bm{\alpha} ,$$
% Where $|\cdot|$ represents element-wise absolute value and $\preceq$  represent element-wise comparison between vectors.

Considering relaxing the aforementioned Equality of Opportunity constraint, we introduce a fairness indicator as follow:
$$
\begin{aligned}
   DEOOM_{\bm{y}} (\phi)= &\big|\mathbb{P}[\widehat{\mathbf{Y}}=\bm{y} \mid A=0, \mathbf{Y}=\mathbf{Y}_{adv}] -\mathbf{P}[\widehat{\mathbf{Y}}=\bm{y} \mid A=1, \mathbf{Y}=\mathbf{Y}_{adv}]\big|,
\end{aligned}
$$
where $\bm{y}$ can be considered as either certain advantageous labels or as a collection of advantageous labels (at this point, '=' is replaced by '$\in$').
% and fairness constraint 
% $$DEOOML \preceq \bm{\alpha} ,$$
% Where $|\cdot|$ represents element-wise absolute value and $\preceq$  represent element-wise comparison between vectors.

% We consider a iterative version of Q-digest that at each client, we first construct a Q-digest for the first component of score f(x) and then at each leaf node, we add a Q-digest for score of the second component of f(x) whose first component belongs to the leaf node. Repeat this operation, we get a sketch for multidimensional score function f(x).
% Assuming we set parameter in order to get $\varepsilon_i$-approximate quantile and rank for the i-th component, we have following result.

Additionally, we consider an iterative Q-digest approach. At each client, our process involves constructing a Q-digest initially for the first component of the score $f(x)$. Subsequently, at each leaf node, we include a Q-digest for the second component of score $f(x)$ associated with the leaf node's first component. Repeating this procedure iteratively allows us to generate a sketch for the multidimensional score function $f(x)$. Assuming the parameter is appropriately set to achieve an $\varepsilon_j$-approximate quantile and rank for the $j$-th component, we arrive at the following result.

\begin{proposition}
Under Assumption \ref{ass:mix}, for $a \in \{0,1\}$, consider $\vq^{\bm{y}_{adv} ,a}=(q_1^{\bm{y}_{adv} ,a},q_2^{\bm{y}_{adv} ,a},...,q_m^{\bm{y}_{adv} ,a}) \in [0,1]^m$, $n_{i,(j)}^{\bm{y}_{adv} ,a}$ is the estimation of $|N_{i,(j)}^{\bm{y}_{adv} ,a}|$, $N_{i,(j)}^{\bm{y}_{adv} ,a}=\{f_j(x) \mid x \text{ belongs to Client i, }, Y=\bm{y}_{adv},A=a, (f_l(x)-t_l)y^*_l \geq 0, l=1,\cdots,j-1  \}|$ and $t^{\bm{y}_{adv}}_j$ is estimation of $q_j$ quantile of $N_{*,(j)}^{\bm{y}_{adv} ,a}$ (the union of $N_{i,(j)}^{\bm{y}_{adv} ,a}$), where estimations with subscript $(j)$ are $\varepsilon$-approximate ranks and quantiles, $\hat{k}^{\bm{y}_{adv},a}_{i,(j)}$ represent the estimation local rank of $t^{\bm{y}_{adv}}_j$ in $N_{i,(j)}^{\bm{y}_{adv} ,a}$, the score-based classifier $\phi(x, a) = \1\{f(x, a) > t_j^{\bm{y}_{adv}, a}\}$. Define 
\begin{small}
$$
\small
\begin{aligned}
    &h_{\bm{y}_{adv},a}= \mathbb{P}\left(\sum_{i=1}^S \pi_i^{\bm{y}_{adv},a}\prod_{j=1}^{m} g_j\left(Q\left(u^{\bm{y}_{adv},a}_{i,(j)}, l_j n_{i,(j)}^{\bm{y}_{adv},a}+1-u^{\bm{y}_{adv},a}_{i,(j)}\right)\right) \right.\\
    & \left.- \sum_{i=1}^S \pi_i^{\bm{y}_{adv},1-a} \prod_{j=1}^{m} g_j\left(Q\left(v^{\bm{y}_{adv},1-a}_{i,(j)}, (2-l_j)n_{i,(j)}^{\bm{y}_{adv},1-a}+1-v^{\bm{y}_{adv},1-a}_{i,(j)}\right)\right) \geq \alpha\right),
\end{aligned}
$$
\end{small}Then we have:
\begin{equation}
\small
\begin{aligned}
    \mathbb{P}(|DEOOM_{\bm{y}^*} (\phi)| > \alpha) \leq h_{\bm{y}_{adv},0}+h_{\bm{y}_{adv},1},
    %&\mathbb{P}\left(\sum_{i=1}^S \pi_i^{1,1} Q\left(M_i^{1,1}, n_i^{1,1}+1-M_i^{1,1}\right) - \sum_{i=1}^S \pi_i^{1,0} Q\left(m_i^{1,0}, n_i^{1,0}+1-m_i^{1,0}\right) > \alpha\right) \\
    %&+ \mathbb{P}\left(\sum_{i=1}^S \pi_i^{1,1} Q\left(m_i^{1,1}, n_i^{1,1}+1-m_i^{1,1} \right) - \sum_{i=1}^S \pi_i^{1,0} Q\left(M_i^{1,0}, n_i^{1,0}+1-M_i^{1,0}\right) < -\alpha\right).
\end{aligned}
\end{equation}
where $\pi_i^{\bm{y}_{adv},a}$ is similarly defined as in Proposition \ref{prop:deoo},
$l_j=(1-2y^*_j)\varepsilon_{j-1}$,$g_j(Q)=(1-2y^*_j)Q+y^*_j$,$u^{\bm{y}_{adv},a}_{i,(j)}=y^*_j m^{\bm{y}_{adv},a}_{i,(j)}+(1-y^*_j) M^{\bm{y}_{adv},a}_{i,(j)} $, $v^{\bm{y}_{adv},a}_{i,(j)}=y^*_j M^{\bm{y}_{adv},a}_{i,(j)}+(1-y^*_j) m^{\bm{y}_{adv},a}_{i,(j)} $, $M^{\bm{y}_{adv},a}_{i,(j)}=max\big(\lceil \hat{k}^{\bm{y}_{adv},a}_{i,(j)}+\varepsilon_j  n^{\bm{y}_{adv},a}_{i,(j)}\rceil,  n^{\bm{y}_{adv},a}_{i,(j)}+1\big)$, 
$m^{\bm{y}_{adv},a}_{i,(j)}=min\big(\lceil \hat{k}^{\bm{y}_{adv},a}_{i,(j)}-\varepsilon_j n^{\bm{y}_{adv},a}_{i,(j)}\rceil,0\big)$, and $Q(\alpha,\beta)$ are independent random variables and $Q(\alpha,\beta) \sim \text{Beta}(\alpha, \beta)$. Especially, we define $Q(0,\beta)=0$ and $Q(\alpha,0)=1$ for $\alpha,\beta \neq 0$. 
\end{proposition}
The proposition above can be proved using Lemma \ref{le:beta} and conditional probability. It is important to note that $\bm{y}$ and $\bm{y}_{adv}$ are not necessarily single labels; they can also represent a set of labels with constraints on specific components where values are restricted to 0 or 1 (for $j$ where $y^*_j$ does not have constraint, $t_j$ is set to 0.5, and it is excluded from the construction of $N$ and calculation of $h$). And similarly, the selection can be conducted by minimizing empirical misclassification error.

Considering a high-dimensional extension of Lemma \ref{le:dis}, we have
\begin{lemma}\label{le:dis2}
For a distribution $F$ with a continuous density function, suppose $q(x)$ denotes the probability of $X \preceq x$ where X is a random variable under $F$, then for $y \preceq x$, we have $F_{(-)}(||x - y||_2) \leq q(x) - q(y) \leq F_{(+)}(||x-y||_2)$, where $F_{(-)}(x)$ and $F_{(+)}(x)$ are two monotonically increasing functions, $F_{(-)}(\epsilon) > 0, F_{(+)}(\epsilon)>0$ for any $\epsilon > 0$ and $\mathop{lim}\limits_{\epsilon \rightarrow 0} F_{(-)}(\epsilon) = \mathop{lim}\limits_{\epsilon \rightarrow 0} F_{(+)}(\epsilon)=0$.
\end{lemma}

Therefore, similarly, we have 
\begin{theorem} \label{thm:acc_deoom2}
Under Assumption \ref{ass:mix} and \ref{ass:quan}, given $\alpha^{\prime}<\alpha$. Suppose $\hat{\phi}$ is the final output of FedFaiREE, we have:

(1) $|D E O O M_{\bm{y}^*}(\hat{\phi})|<\alpha$ with probability $(1-\delta)^{N}$, where $N$ is the size of the candidate set.

(2) Suppose the density distribution functions of $f^{*}$ under $A=a, Y=1$ are continuous. When the input classifier $f$ satisfies $||f(x, a)-f^{*}(x, a)||_2 \leq \epsilon_{0}$, for any $\epsilon>0$ such that $M_{(+)}^{*}(\epsilon+\gamma \varepsilon) \leq$ $\frac{\alpha-\alpha^{\prime}}{2m}-M_{(+)}^{*}\left(2 \epsilon_{0}\right)$, we have
\begin{equation}\small
\begin{aligned}
     &\mathbb{P}(\hat{\phi}(x, a) \neq Y)-\mathbb{P}\left(\phi_{\alpha^{\prime}}^{*}(x, a)\neq Y\right) \leq  2m M_{(+)}^{*}\left(2 \epsilon_{0}\right)+2m M_{(+)}^{*}(\epsilon+\gamma \varepsilon_m)+ 2\theta+O(\epsilon)
\end{aligned}
    %     &+\sum_{i=1}^S 2\pi_i \left[e_i^{0,0} p^i_{0}\left(1-p^i_{Y, 0}\right)
    % +e_i^{0,1} p^i_{0}p^i_{Y, 0} +e_i^{1,0} p^i_{1}\left(1-p^i_{Y, 1}\right)
    % +e_i^{1,1} p^i_{1}p^i_{Y, 1}\right] 
\end{equation}
with probability $
1-(2^{m+1}+2)\sum_{a=0}^{1}\sum_{i=1}^{S} e^{-2 n_i^{0,a} \epsilon^{2}}-\prod_{i=1}^S\big(1-F_{i(-)}^{\bm{y}_{adv},0}(2 \epsilon)\big)^{n_i^{\bm{y}_{adv},0}}-\prod_{i=1}^S\big(1-F_{i(-)}^{\bm{y}_{adv},1}(2 \epsilon)\big)^{n_i^{\bm{y}_{adv},1}}-\delta, $
where $\delta=\sum_{a=0}^{1}\delta^{\bm{y}_{adv},a}(n^{\bm{y}_{adv},a})$, 
$\theta=\sum_{i=1}^S \left[\pi_i \sum_{a=0}^{1}\sum_{\bm{y}}e_i^{\bm{y},a} p^i_{a}p^i_{\bm{y}, a}\right]$
, 
$e_i^{\bm{y},a}=\frac{2\lfloor \varepsilon_m n_i^{\bm{y},a} \rfloor +1}{2\left(n_i^{\bm{y},a}+1\right)}$
, $M_{(+)}^{*}$ corresponds to the maximum of $F_{(+)}$ associated with $f_j^*$, and the definition of $F_{(+)}$ and $F_{(-)}$ are shown in Lemma \ref{le:dis2}.
\end{theorem}

\section{Application on Further Notions}\label{appendix:application}
In this section, we delve into the application of FedFaiREE on additional fairness concepts.
\subsection{Definition}
To begin with, we introduce the definitions of various fairness concepts.
\begin{definition}[Demographic Parity]
    A classifier satisfies Demographic Parity if its prediction $\widehat{Y}$ is statistically independent of the sensitive attribute $A$ :
$$
\bP(\widehat{Y}=1 \mid A=1)=\bP(\widehat{Y}=1 \mid A=0)
$$
\end{definition}

\begin{definition}[Predictive Equality]
A classifier satisfies Predictive Equality if it achieves the same TNR (or FPR) among protected groups:
$$
\bP_{X \mid A=1, Y=0}(\widehat{Y}=1)=\bP_{X \mid A=0, Y=0}(\widehat{Y}=1)
$$
\end{definition}

\begin{definition}[Equalized Accuracy]
A classifier satisfies Equalized Accuracy if its mis-classification error is statistically independent of the sensitive attribute $A$:
$$
\bP(\widehat{Y} \neq Y \mid A=1)=\bP(\widehat{Y} \neq Y \mid A=0)
$$
\end{definition}

Similar to $DEOO$ and $DEO$, we define the following indicators:
\begin{align}
\mathrm{DDP} &=\bP_{X \mid A=1}(\widehat{Y}=1)-\bP_{X \mid A=0}(\widehat{Y}=1) \label{ddp}\\
\mathrm{DPE} &=\bP_{X \mid A=1, Y=0}(\widehat{Y}=1)-\bP_{X \mid A=0, Y=0}(\widehat{Y}=1)\label{dpe}\\
\mathrm{DEA} &= \bP(\widehat{Y} \neq Y \mid A=1)-\bP(\widehat{Y} \neq Y \mid A=0).\label{dea}
\end{align}

\subsection{Theory and Algorithm}
Similar to $DEO$ and $DEOO$, we To be concise, we denote $n_i^{*,a}$ as denotes the size of subset of dataset $D_i$ that satisfies $A=a$. Similar explanations apply to $k^{*,a}$.
\subsubsection{FedFaiREE for DDP}
\begin{proposition}\label{prop:ddp}
Under Assumption \ref{ass:mix}, for $a \in \{0,1\}$, consider $k^{*, a} \in \{1, \ldots, n^{*, a}\}$, the corresponding $\hat{k}_i^{*,a}$ for $i \in [S]$ which are $\varepsilon$-approximate ranks and the score-based classifier $\phi(x, a) = \1\{f(x, a) > t_{(k^{*, a})}^{*, a}\}$ . Define 
$$\begin{aligned}
    &h_{*,a}(\vu,\vv)=\mathbb{P}\big(\sum_{i=1}^S \pi_i^{*,a} Q(u_i, n_i^{*,a}+1-u_i) - \sum_{i=1}^S \pi_i^{*,1-a} Q(v_i, n_i^{*,1-a}+1-v_i) \geq \alpha\big).
\end{aligned}$$ Then we have:
\begin{equation}
\begin{aligned}
    \mathbb{P}(|DDP(\phi)| > \alpha) \leq h_{*,0}(\bm{M}^{*,0},\vm^{*,1})+h_{*,1}(\bm{M}^{*,1},\vm^{*,0})
    %&\mathbb{P}\left(\sum_{i=1}^S \pi_i^{1,1} Q\left(M_i^{1,1}, n_i^{1,1}+1-M_i^{1,1}\right) - \sum_{i=1}^S \pi_i^{1,0} Q\left(m_i^{1,0}, n_i^{1,0}+1-m_i^{1,0}\right) > \alpha\right) \\
    %&+ \mathbb{P}\left(\sum_{i=1}^S \pi_i^{1,1} Q\left(m_i^{1,1}, n_i^{1,1}+1-m_i^{1,1} \right) - \sum_{i=1}^S \pi_i^{1,0} Q\left(M_i^{1,0}, n_i^{1,0}+1-M_i^{1,0}\right) < -\alpha\right).
\end{aligned}\label{eq:loss_ddp}
\end{equation}

Where $\pi_i^{*,a} = \mathbb{P}(\text{sampling } x \text{ from client } i \mid \text{ sampling } x \text{ with sensitive attribute} A=a)$, $M_i^{*,a}=max\left(\lceil \hat{k}_i^{*,a}+\varepsilon n_i^{*,a}\rceil, n_i^{*,a}+1\right)$, $m_i^{*,a}=min\left(\lceil \hat{k}_i^{*,a}-\varepsilon n_i^{*,a}\rceil,0\right)$, and $Q(A,B)$ are independent random variables following Beta distribution, $Q(A,B) \sim \text{Beta}(A, B)$. Especially, we define $Q(0,B)=0$ and $Q(A,0)=1$ for $A,B \neq 0$. 
\end{proposition}

\begin{algorithm}[h]
\caption{FedFaiREE for DDP}\label{alg:ddp}
\textbf{Input:} Train dataset $D_i = D_i^{0,0} \cup D_i^{0,1} \cup D_i^{1,0} \cup D_i^{1,1}$; pre-trained classifier $\phi_0$ with function f; fairness constraint parameter $\alpha$ ; Confidence level parameter $\beta$; Weights of different clients $\pi$ 
\textbf{Output:} classifier $\hat{\phi}(x,a)= \1 
\{ f(x,a) >  t_{(k^{1,a})}^{1,a}\}$
\begin{algorithmic}[1]
\STATE{\bf Client Side:}
\FOR{i=1,2,..,$S$}
\STATE{Score on train data points in $D_i$ and get $T_i^{y,a}=\{t_{i,1}^{y,a},t_{i,2}^{y,a}, \cdots, t_{i,n_i^{y,a}}^{y,a}\}$ }
\STATE{Sort $T_i^{y,a}$}
\STATE{Calculate q-digest of $T_i^{y,a}$ on client $i$}
\STATE{Update digest to server}
\ENDFOR
\STATE{\bf Server Side:}
%\STATE{Combine q-digests }%and get the sketch of sorted $T^{y,a}$
\STATE{Construct $K$ by $K=\{(k^{1,0}, k^{1,1})| L(\vk^{1,0}, \vk^{1,1})< 1-\beta \}$, where L is defined by the right-hand side of Inequality \ref{eq:loss_ddp}}
\STATE{Select optimal $(k_0,k_1)$ by minimizing equation \ref{eq:error_p} using estimated values $\hat{p}^i_{a} = \frac{n_i^{0,a}+n_i^{1,a}}{n_i^{0,0}+n_i^{0,1}+n_i^{1,0}+n_i^{1,1}}$ and $\hat{p}^i_{Y,a} = \frac{n_i^{1,a}}{n_i^{0,a}+n_i^{1,a}}$}
\end{algorithmic}
\end{algorithm}

\begin{theorem} \label{thm:ddp}
Under Assumption \ref{ass:mix} and \ref{ass:quan}, given $\alpha^{\prime}<\alpha$. Suppose $\hat{\phi}$ is the final output of FedFaiREE, we have:

(1) $|D D P(\hat{\phi})|<\alpha$ with probability $(1-\delta)^{N}$, where $N$ is the size of the candidate set.

(2) Suppose the density distribution functions of $f^{*}$ under $A=a, Y=1$ are continuous. When the input classifier $f$ satisfies $\left|f(x, a)-f^{*}(x, a)\right| \leq \epsilon_{0}$, for any $\epsilon>0$ such that $F_{(+)}^{*}(\epsilon+\gamma \varepsilon) \leq$ $\frac{\alpha-\alpha^{\prime}}{2}-F_{(+)}^{*}\left(2 \epsilon_{0}\right)$, we have
\begin{equation}
    \begin{aligned}
        &\mathbb{P}(\hat{\phi}(x, a) \neq Y)-\mathbb{P}\left(\phi_{\alpha^{\prime}}^{*}(x, a) \neq Y\right) \leq  2 F_{(+)}^{*}\left(2 \epsilon_{0}\right)+2 F_{(+)}^{*}(\epsilon+\gamma \varepsilon)+8 \epsilon^{2}+20 \epsilon+ 2\theta
    %     &+\sum_{i=1}^S 2\pi_i \left[e_i^{0,0} p^i_{0}\left(1-p^i_{Y, 0}\right)
    % +e_i^{0,1} p^i_{0}p^i_{Y, 0} +e_i^{1,0} p^i_{1}\left(1-p^i_{Y, 1}\right)
    % +e_i^{1,1} p^i_{1}p^i_{Y, 1}\right] 
    \end{aligned}
\end{equation}
with probability $
1-4\sum_{a=1}^{1}\sum_{i=1}^{S} e^{-2 n_i^{0,a} \epsilon^{2}}-\prod_{i=1}^S\left(1-F_{i(-)}^{1,0}(2 \epsilon)\right)^{n_i^{1,0}}-\prod_{i=1}^S\left(1-F_{i(-)}^{1,1}(2 \epsilon)\right)^{n_i^{1,1}}-\delta, $
where $\delta=\delta^{1,0}(n^{1,0})+\delta^{1,1}(n^{1,1})$, $\theta$ is defined in Proposition\ref{prop:error} and the definition of $F_{(+)}$ and $F_{(-)}$ are shown in Lemma \ref{le:dis} 
\end{theorem}

\subsubsection{FedFaiREE for DPE}
\begin{proposition}\label{prop:dpe}
Under Assumption \ref{ass:mix}, for $a \in \{0,1\}$, consider $k^{0, a} \in \{1, \ldots, n^{0, a}\}$, the corresponding $\hat{k}_i^{0,a}$ for $i \in [S]$ which are $\varepsilon$-approximate ranks and the score-based classifier $\phi(x, a) = \1\{f(x, a) > t_{(k^{0, a})}^{0, a}\}$ . 
Define 
$$\begin{aligned}
    &h_{y,a}(\vu,\vv)=\mathbb{P}\big(\sum_{i=1}^S \pi_i^{y,a} Q(u_i, n_i^{y,a}+1-u_i) - \sum_{i=1}^S \pi_i^{y,1-a} Q(v_i, n_i^{y,1-a}+1-v_i) \geq \alpha\big).
\end{aligned}$$
Then we have:
\begin{equation}
\begin{aligned}
    \mathbb{P}(|DPE(\phi)| > \alpha) \leq h_{0,1}(\bm{M}^{0,1},\vm^{0,0})+h_{0,0}(\bm{M}^{0,0},\vm^{0,0})
\end{aligned}\label{eq:loss_dpe}
\end{equation}
where $M_i^{0,a}=\lceil \hat{k}_i^{0,a}+\varepsilon n_i^{0,a}\rceil$, $m_i^{0,a}=\lceil \hat{k}_i^{0,a}-\varepsilon n_i^{0,a}\rceil$, $\pi_i^{y,a} = \mathbb{P}(\text{sampling } x \text{ from client } i \mid \text{ sampling } x \text{ with label } Y=y \text{ and } A=a)$, and $Q(A,B)$ are independent random variables following Beta distribution, $Q(A,B) \sim \text{Beta}(A, B)$. 
\end{proposition}

\begin{algorithm}[h]
\caption{FedFaiREE for DPE}\label{alg:dpe}
\textbf{Input:} Train dataset $D_i = D_i^{0,0} \cup D_i^{0,1} \cup D_i^{1,0} \cup D_i^{1,1}$; pre-trained classifier $\phi_0$ with function f; fairness constraint parameter $\alpha$ ; Confidence level parameter $\beta$; Weights of different clients $\pi$ 
\textbf{Output:} classifier $\hat{\phi}(x,a)= \1 
\{ f(x,a) >  t_{(k^{1,a})}^{1,a}\}$
\begin{algorithmic}[1]
\STATE{\bf Client Side:}
\FOR{i=1,2,..,$S$}
\STATE{Score on train data points in $D_i$ and get $T_i^{y,a}=\{t_{i,1}^{y,a},t_{i,2}^{y,a}, \cdots, t_{i,n_i^{y,a}}^{y,a}\}$ }
\STATE{Sort $T_i^{y,a}$}
\STATE{Calculate q-digest of $T_i^{y,a}$ on client $i$}
\STATE{Update digest to server}
\ENDFOR
\STATE{\bf Server Side:}
%\STATE{Combine q-digests }%and get the sketch of sorted $T^{y,a}$
\STATE{Construct $K$ by $K=\{(k^{1,0}, k^{1,1})| L(\vk^{1,0}, \vk^{1,1})< 1-\beta \}$, where L is defined by the right-hand side of Inequality \ref{eq:loss_dpe}}
\STATE{Select optimal $(k_0,k_1)$ by minimizing equation \ref{eq:error_p} using estimated values $\hat{p}^i_{a} = \frac{n_i^{0,a}+n_i^{1,a}}{n_i^{0,0}+n_i^{0,1}+n_i^{1,0}+n_i^{1,1}}$ and $\hat{p}^i_{Y,a} = \frac{n_i^{1,a}}{n_i^{0,a}+n_i^{1,a}}$}
\end{algorithmic}
\end{algorithm}

\begin{theorem} \label{thm:dpe}
Under Assumption \ref{ass:mix} and \ref{ass:quan}, given $\alpha^{\prime}<\alpha$. Suppose $\hat{\phi}$ is the final output of FedFaiREE, we have:

(1) $|D P E(\hat{\phi})|<\alpha$ with probability $(1-\delta)^{N}$, where $N$ is the size of the candidate set.

(2) Suppose the density distribution functions of $f^{*}$ under $A=a, Y=1$ are continuous. When the input classifier $f$ satisfies $\left|f(x, a)-f^{*}(x, a)\right| \leq \epsilon_{0}$, for any $\epsilon>0$ such that $F_{(+)}^{*}(\epsilon+\gamma \varepsilon) \leq$ $\frac{\alpha-\alpha^{\prime}}{2}-F_{(+)}^{*}\left(2 \epsilon_{0}\right)$, we have
\begin{equation}
    \begin{aligned}
        &\mathbb{P}(\hat{\phi}(x, a) \neq Y)-\mathbb{P}\left(\phi_{\alpha^{\prime}}^{*}(x, a) \neq Y\right) \leq  2 F_{(+)}^{*}\left(2 \epsilon_{0}\right)+2 F_{(+)}^{*}(\epsilon+\gamma \varepsilon)+8 \epsilon^{2}+20 \epsilon+ 2\theta
    %     &+\sum_{i=1}^S 2\pi_i \left[e_i^{0,0} p^i_{0}\left(1-p^i_{Y, 0}\right)
    % +e_i^{0,1} p^i_{0}p^i_{Y, 0} +e_i^{1,0} p^i_{1}\left(1-p^i_{Y, 1}\right)
    % +e_i^{1,1} p^i_{1}p^i_{Y, 1}\right] 
    \end{aligned}
\end{equation}
with probability $
1-4\sum_{a=1}^{1}\sum_{i=1}^{S} e^{-2 n_i^{0,a} \epsilon^{2}}-\prod_{i=1}^S\left(1-F_{i(-)}^{1,0}(2 \epsilon)\right)^{n_i^{1,0}}-\prod_{i=1}^S\left(1-F_{i(-)}^{1,1}(2 \epsilon)\right)^{n_i^{1,1}}-\delta, $
where $\delta=\delta^{1,0}(n^{1,0})+\delta^{1,1}(n^{1,1})$, $\theta$ is defined in Proposition\ref{prop:error} and the definition of $F_{(+)}$ and $F_{(-)}$ are shown in Lemma \ref{le:dis} 
\end{theorem}
\subsubsection{FedFaiREE for DEA}
\begin{proposition}\label{prop:dea}
Under Assumption \ref{ass:mix}, for $a \in \{0,1\}$, consider $k^{y, a} \in \{1, \ldots, n^{y, a}\}$, the corresponding $\hat{k}_i^{y,a}$ for $i \in [S]$ which are $\varepsilon$-approximate ranks and the score-based classifier $\phi(x, a) = \1\{f(x, a) > t_{(k^{1, a})}^{1, a}\}$ . 
Define 
$$
\begin{aligned}
    &h_{*,a}(\vu^1,\vu^0,\vv^1,\vv^0) =\mathbb{P}\bigg(p_{y,a}-p_{y,1-a}-p_{y,a}\sum_{i=1}^S \pi_i^{1,a}  Q\left(u^1_i, n_i^{1,a}+1-u^1_i\right) +(1-p_{y,a})\sum_{i=1}^S \pi_i^{0,a}  Q\left(u^0_i, n_i^{0,a}+1-u^0_i\right) \\
    &+p_{y,1-a}\sum_{i=1}^S \pi_i^{1,1-a} Q\left(v^1_i, n_i^{1,1-a}+1-v^1_i\right) - (1-p_{y,1-a})\sum_{i=1}^S \pi_i^{0,1-a} Q\left(v^0_i, n_i^{0,1-a}+1-v^0_i\right)\geq \alpha\bigg).
\end{aligned}
$$
Then we have:
\begin{equation}
\begin{aligned}
    &\mathbb{P}(|DPE(\phi)| > \alpha) \leq h_{*,1}(\vm^{1,1},\bm{M}^{0,1}, \bm{M}^{1,0},\vm^{0,0}) +h_{*,0}(\vm^{1,0},\bm{M}^{0,0}, \bm{M}^{1,1},\vm^{0,1})
\end{aligned}\label{eq:loss_dea}
\end{equation}
where $M_i^{0,a}=\lceil \hat{k}_i^{0,a}+\varepsilon n_i^{0,a}\rceil$, $m_i^{0,a}=\lceil \hat{k}_i^{0,a}-\varepsilon n_i^{0,a}\rceil$, $\pi_i^{y,a} = \mathbb{P}(\text{sampling } x \text{ from client } i \mid \text{ sampling } x \text{ with label } Y=y \text{ and } A=a)$, and $Q(A,B)$ are independent random variables following Beta distribution, $Q(A,B) \sim \text{Beta}(A, B)$. 
\end{proposition}
\begin{algorithm}[h]
\caption{FedFaiREE for DEA}\label{alg:dea}
\textbf{Input:} Train dataset $D_i = D_i^{0,0} \cup D_i^{0,1} \cup D_i^{1,0} \cup D_i^{1,1}$; pre-trained classifier $\phi_0$ with function f; fairness constraint parameter $\alpha$ ; Confidence level parameter $\beta$; Weights of different clients $\pi$ 
\textbf{Output:} classifier $\hat{\phi}(x,a)= \1 
\{ f(x,a) >  t_{(k^{1,a})}^{1,a}\}$
\begin{algorithmic}[1]
\STATE{\bf Client Side:}
\FOR{i=1,2,..,$S$}
\STATE{Score on train data points in $D_i$ and get $T_i^{y,a}=\{t_{i,1}^{y,a},t_{i,2}^{y,a}, \cdots, t_{i,n_i^{y,a}}^{y,a}\}$ }
\STATE{Sort $T_i^{y,a}$}
\STATE{Calculate q-digest of $T_i^{y,a}$ on client $i$}
\STATE{Update digest to server}
\ENDFOR
\STATE{\bf Server Side:}
%\STATE{Combine q-digests }%and get the sketch of sorted $T^{y,a}$
\STATE{Construct $K$ by $K=\{(k^{1,0}, k^{1,1})| L(\vk^{1,0}, \vk^{1,1})< 1-\beta \}$, where L is defined by the right-hand side of Inequality \ref{eq:loss_dea}}
\STATE{Select optimal $(k_0,k_1)$ by minimizing equation \ref{eq:error_p} using estimated values $\hat{p}^i_{a} = \frac{n_i^{0,a}+n_i^{1,a}}{n_i^{0,0}+n_i^{0,1}+n_i^{1,0}+n_i^{1,1}}$ and $\hat{p}^i_{Y,a} = \frac{n_i^{1,a}}{n_i^{0,a}+n_i^{1,a}}$}
\end{algorithmic}
\end{algorithm}

\begin{theorem} \label{thm:dea}
Under Assumption \ref{ass:mix} and \ref{ass:quan}, given $\alpha^{\prime}<\alpha$. Suppose $\hat{\phi}$ is the final output of FedFaiREE, we have:

(1) $|D E A(\hat{\phi})|<\alpha$ with probability $(1-\delta)^{N}$, where $N$ is the size of the candidate set.

(2) Suppose the density distribution functions of $f^{*}$ under $A=a, Y=1$ are continuous. When the input classifier $f$ satisfies $\left|f(x, a)-f^{*}(x, a)\right| \leq \epsilon_{0}$, for any $\epsilon>0$ such that $F_{(+)}^{*}(\epsilon+\gamma \varepsilon) \leq$ $\frac{\alpha-\alpha^{\prime}}{2}-F_{(+)}^{*}\left(2 \epsilon_{0}\right)$, we have
\begin{equation}
    \begin{aligned}
        &\mathbb{P}(\hat{\phi}(x, a) \neq Y)-\mathbb{P}\left(\phi_{\alpha^{\prime}}^{*}(x, a)\neq Y\right) \leq  2 F_{(+)}^{*}\left(2 \epsilon_{0}\right)+2 F_{(+)}^{*}(\epsilon+\gamma \varepsilon)+8 \epsilon^{2}+20 \epsilon+ 2\theta
    %     &+\sum_{i=1}^S 2\pi_i \left[e_i^{0,0} p^i_{0}\left(1-p^i_{Y, 0}\right)
    % +e_i^{0,1} p^i_{0}p^i_{Y, 0} +e_i^{1,0} p^i_{1}\left(1-p^i_{Y, 1}\right)
    % +e_i^{1,1} p^i_{1}p^i_{Y, 1}\right] 
    \end{aligned}
\end{equation}
with probability $
1-4\sum_{a=1}^{1}\sum_{i=1}^{S} e^{-2 n_i^{0,a} \epsilon^{2}}-\prod_{i=1}^S\left(1-F_{i(-)}^{1,0}(2 \epsilon)\right)^{n_i^{1,0}}-\prod_{i=1}^S\left(1-F_{i(-)}^{1,1}(2 \epsilon)\right)^{n_i^{1,1}}-\delta, $
where $\delta=\delta^{1,0}(n^{1,0})+\delta^{1,1}(n^{1,1})$, $\theta$ is defined in Proposition\ref{prop:error} and the definition of $F_{(+)}$ and $F_{(-)}$ are shown in Lemma \ref{le:dis} 
\end{theorem}

\subsection{Connection with Fairness Metrics in \citep{hu2022fair} and \citep{papadaki2022minimax} }
\citet{hu2022fair} introduces several group fairness metrics as follow:
\begin{definition}
A classifier h satisfies Bounded Group Loss (BGL) at level $\zeta$ under distribution $\mathcal{D}$ if for all $a \in A$, we have $\mathbb{E}[l(h(x), y) \mid A=a] \leq \zeta$.
\end{definition}
\begin{definition}
    A classifier $h$ satisfies Conditional Bounded Group Loss (CBGL) for $y \in Y$ at level $\zeta_y$ under distribution $\mathcal{D}$ if for all $a \in A$, we have $\mathbb{E}[l(h(x), y) \mid A=a, Y=y] \leq \zeta_y$.
\end{definition}
When considering y as a binary variable and the loss function l being the 0-1 loss function, BGL is equivalent to
$$\mathbb{P}[\hat{y}\neq y |\mid A=a] \leq \zeta,$$
holding for any a, whereas Demographic Parity refers to
$$\mathbb{P}[\hat{y}\neq y |\mid A=0]=\mathbb{P}[\hat{y}\neq y |\mid A=1].$$
In this context, BGL can be understood as a relaxation of Demographic Parity.

Similarly, when considering y as a binary variable and the loss function l being the 0-1 loss function, CBGL is equivalent to
$$\mathbb{P}[\hat{y}\neq y |\mid A=a, Y=y] \leq \zeta_y,$$
holding for any a, whereas Equalized Odds refers to
$$\mathbb{P}[\hat{y}\neq y |\mid A=0,Y=y]=\mathbb{P}[\hat{y}\neq y |\mid A=1,Y=y].$$
In this context, CBGL can be understood as a relaxation of Equalized Odds.

According to \citep{hu2022fair}, the metric that \citet{papadaki2022minimax} considers is equivalent to
\begin{definition}
    FedMinMax\citep{papadaki2022minimax} aims to solve for the following objective: $\min _h \max _{\boldsymbol{\lambda} \in \mathbb{R}_{+}^{|A|},\|\boldsymbol{\lambda}\|_1=1} \sum_{a \in A} \boldsymbol{\lambda}_a \mathbf{r}_a(h)$, where
    $\mathbf{r}_a(h):=\sum_{k=1}^K \mathbf{r}_{a, k}(h)=\sum_{k=1}^K\left(1 / m_a \sum_{a_{k, i}=a} l\left(h\left(x_{k, i}\right), y_{k, i}\right)\right)$, $K$ stands for client number and $m_a$ stands for numbers of points with attribute $a$. 
\end{definition}
Similarly, this can be understood as a relaxation of Demographic Parity in the context of considering y as a binary variable and the loss function l being the 0-1 loss function.

%\section{Extension to Multi-Groups and Multi-Labels Fairness}\label{app:multi}

\section{Experiment Details}\label{appendix:experiment}
\subsection{Further selection in candidate set construction}\label{appendix:search}
To further simplify the candidate set selection, similar to FaiREE\citep{fairee}, we note that, by Lemma \ref{bayes-optimal}, if we assume our input classifier $f$ is similar to $f^*$, we have 
\begin{equation}\label{trans}
   t_a=\frac{p_{a} p_{Y, a}}{2 p_{a} p_{Y, a}+(1-2 a) t_{E, \alpha}^{\star}},
\end{equation}
which means 
\begin{equation}\label{trans1}
  t_{E, \alpha}^{\star}=\frac{p_{a} p_{Y, a}-2 p_{a} p_{Y, a} t_a}{(1-2 a)t_a}  
\end{equation}
Therefore, bringing Equation \ref{trans1} ($a=0$) into Equation \ref{trans} ($a=1$), we have 
\begin{equation}
    t_0=\frac{p_{0} p_{Y, 0}}{2 p_{0} p_{Y, 0}+2 p_{1} p_{Y, 1}-p_{1} p_{Y, 1}/t_1}
\end{equation}
This inspired us that we could further simplify the construction of candidate set K by replacing Equation \ref{eq:construct K} with
 \begin{equation}\label{eq:construct K ad}
     K=\{(k^{1,0}, k^{1,1})| L(\vk^{1,0}, \vk^{1,1})< 1-\beta, k^{1,0}=\mu(k^{1,1})\},
 \end{equation}
 Where $\mu({k_1})=\argmin_{k_0} \frac{p_{0} p_{Y, 0}}{2 p_{0} p_{Y, 0}+2 p_{1} p_{Y, 1}-p_{1} p_{Y, 1}/\hat{t}_{k_1}}$

 \subsection{Model Details and Hyperparameter Selection}\label{app:hyper}
We employed several existing Federated Learning models in the experiment, and their detailed information is listed as follows:

\begin{enumerate}
    \item FedAvg\citep{Fedavg}: FedAvg is a fundamental Federated Learning model that serves as the foundational baseline for our experiments. It operates by computing model updates on each client's local data and then aggregates these updates on a central server through averaging. FedAvg doesn't specifically address fairness concerns but is crucial for benchmarking purposes.
    \item FedFB\citep{FedFB}: FedFB is a novel framework designed for fairness-aware Federated Learning. Drawing inspiration from FairBatch, a fairness algorithm for centralized data, FedFB extends this concept to the Federated Learning setting. It incorporates both local debiasing and global reweighting for each client within the framework to achieve fairness objectives.
    \item FairFed\citep{FairFed}: FairFed is another innovative framework for fairness-aware Federated Learning. It employs a unique approach to improving fairness by reweighting clients based on updated local fairness indicators during each epoch. This allows FairFed to combine multiple local debiasing methods effectively.
\end{enumerate}

To compare performance in terms of DEOO, we selected FedFB with respect to Equal Opportunity (EO) as presented in \citet{FedFB}, and FairFed-FB-EO from FairFed as introduced in \citet{FairFed}. These are specific models within the FedFB and FairFed frameworks that are designed for DEOO.

We also note that there are concerns raised by the fairness community regarding the COMPAS dataset underscore crucial complexities within algorithmic fairness research\citep{bao2021s}. While Risk Assessment Instrument (RAI) datasets like COMPAS serve as prevalent benchmarks, their oversimplification of the intricate dynamics within real-world criminal justice processes poses significant challenges. Measurement biases and errors inherent in pretrial RAI datasets limit the direct translation of fairness claims to actual outcomes within the criminal justice system. Additionally, the technical focus on these data as a benchmark sometimes ignores the contextual grounding necessary for working with RAI datasets. Ethical reflection within socio-technical systems further highlights the necessity of acknowledging and grappling with the limitations and complexities inherent in RAI datasets. 

Additionally, the hyperparameter selection ranges for each model are shown in Table \ref{tab:hyperparameters}.
\begin{table*}[htbp]
  \centering
  \captionsetup{font=small, labelfont=bf}
  \caption{Hyperparameter Selection Ranges}
  \label{tab:hyperparameters}
  \renewcommand{\arraystretch}{1.5} 
  \setlength{\tabcolsep}{10pt} 
  \begin{tabular}{p{2cm}p{5cm}p{5cm}}
    \toprule
    \textbf{Model} & \textbf{Hyperparameter} & \textbf{Ranges} \\
    \midrule
    \multirow{9}{*}{General} & Learning rate & \{0.001, 0.005, 0.01\} \\
    & Global round & \{5, 10, 20, 30, 40, 50, 80\} \\
    & Local round & \{5, 10\} \\
    & Local batch size & \{16, 32, 64, 128\} \\
    & Hidden layer & \{5, 10, 50\}\\
    & Optimizer & \{Adam, Sgd\} \\
    & Fraction & \{1\} \\
    & Parameter for Dirichlet distribution &\{1\} for Adult, \{10\} for Compas \\
    & Number of Clients &\{100\} for Adult, \{10\} for Compas, \{50\} for ACSIncome  \\ 
    & Sensitive Group &Female for Adult and Compas, Non-white for ACSIncome\\
    \midrule
    \multirow{1}{*}{FedFaiREE} &Confidence level &\{95\%\} \\
    \midrule
    \multirow{2}{*}{Qdigest} &Accuracy &\{1/$2^7$\} for Adult and  ACSIncome, \{1/$2^{10}$\} for Compas \\
    &Compression factor & \{300\} for Adult and ACSIncome, \{150\} for Compas \\
    \midrule
    \multirow{1}{*}{FedFB} & Step size ($\alpha$) & \{0.005, 0.01, 0.05\} \\
    \midrule
    \multirow{2}{*}{FairFed} & Global step size ($\beta$) & \{0.005, 0.01, 0.05\} \\
    & Local debiasing step size ($\alpha$) & \{0.005, 0.01, 0.05\} \\
    \bottomrule
  \end{tabular}
\end{table*}

We further present a data split sample in Table \ref{split}, where random seed was set to be 0. 
\begin{table*}[h]
\centering
\caption{\textbf{Heterogeneous data distribution on the sensitive attribute.} The client index is sorted by number of Male.}\label{split}
\begin{tabular}{c|c|c|c|c|c}
\toprule
\multicolumn{3}{c}{Minimum ten clients} &\multicolumn{3}{c}{Maximum ten clients} \\
\cmidrule(r){1-3} \cmidrule(r){4-6}
Client id & Male & Female & Client id & Male & Female \\
\midrule
1 & 6 & 41 & 91 & 738 & 118  \\
2 & 6 & 117 & 92 & 863 & 49  \\
3 & 6 & 297 & 93 & 880 & 52  \\
4 & 13 & 35 & 94 & 956 & 147  \\
5 & 20 & 310 & 95 & 961 & 50  \\
6 & 22 & 120 & 96 & 1101 & 35  \\
7 & 24 & 234 & 97 & 1245 & 102  \\
8 & 30 & 70 & 98 & 1250 & 31  \\
9 & 32 & 124 & 99 & 1277 & 180 \\
10 & 33 & 26 & 100 & 1480 & 24  \\
\bottomrule
\end{tabular}
\end{table*}

\subsection{More detailed results}\label{app:more_detail}
In this subsection, we present a more detailed analysis of the experimental results from Section \ref{sec:experiment}. Table \ref{tab:result1} and Table \ref{tab:result2} respectively illustrate the variances in the results obtained from the Adult dataset and the Compas dataset.

\begin{table*}[htbp]
\centering
\caption{\textbf{Results with standard deviation on Adult.}}\label{tab:result1}
{
\begin{tabular}{cccccccccc}
\toprule
& & \multicolumn{4}{c}{\textbf{Adult}} \\
\cmidrule(r){3-6}
{Model} & {\textbf{FedFaiREE}} & {$\alpha$} & {$\overline{ACC}$} & {$\overline{|DEOO|}$} & {$|DEOO|_{95}$} \\
\midrule
\textbf{FedAvg} & {\ding{55}} & {/} & {0.844 (0.003)} & {0.131 (0.030)} & {0.178} \\
 & {\ding{51}} & {0.10} & {0.843 (0.003)} & {\textbf{0.038} (0.026)} & {\textbf{0.083}} \\
\midrule
\multirow{2}{*}{\textbf{FedFB}} & {\ding{55}} & {/} & {0.850 (0.003)} & {0.057 (0.034)} & {0.117} \\
 & {\ding{51}} & {0.10} & {0.850 (0.003)} & {\textbf{0.036} (0.025)} & {\textbf{0.083}}  \\
\midrule
\multirow{2}{*}{\textbf{FairFed}} & {\ding{55}} & {/} & {0.842 (0.003)} & {0.069 (0.034)} & {0.118} \\
 & {\ding{51}} & {0.10} & {0.841 (0.003)} & {\textbf{0.037} (0.026)} & {\textbf{0.081}} \\
\midrule
\end{tabular}}
\end{table*}

\begin{table*}[htbp]
\centering
\caption{\textbf{Results with standard deviation on Compas.}}\label{tab:result2}
{
\begin{tabular}{cccccccccc}
\toprule
& & \multicolumn{4}{c}{\textbf{Compas}} \\
\cmidrule(r){3-6}
{Model} & {\textbf{FedFaiREE}} & {$\alpha$} & {$\overline{ACC}$} & {$\overline{|DEOO|}$} & {$|DEOO|_{95}$} \\
\midrule
\multirow{2}{*}{\textbf{FedAvg}} & {\ding{55}} & {/} & {0.662 (0.011)} & {0.126 (0.056)} & {0.223} \\
 & {\ding{51}} & {0.15} & {0.659 (0.010)} & {\textbf{0.051} (0.044)} & {\textbf{0.137}} \\
\midrule
\multirow{2}{*}{\textbf{FedFB}} & {\ding{55}} & {/} & {0.642 (0.011)} & {0.107 (0.043)} & {0.174} \\
 & {\ding{51}} & {0.15} & {0.641 (0.010)} & {\textbf{0.062} (0.040)} & {\textbf{0.125}} \\
\midrule
\multirow{2}{*}{\textbf{FairFed}} & {\ding{55}} & {/} & {0.648 (0.012)} & {0.097 (0.047)} & {0.166} \\
 & {\ding{51}} & {0.15} & {0.645 (0.011)} & {\textbf{0.047} (0.036)} & {\textbf{0.114}} \\
\midrule
\end{tabular}}
%\vspace{-1.3em}
\end{table*}

\begin{table*}[htbp]
\centering
\caption{\textbf{Results on Adult with Parameter for Dirichlet distribution=10.}}\label{tab:result3}
{
\begin{tabular}{cccccccccc}
\toprule
& & \multicolumn{4}{c}{\textbf{Adult}} \\
\cmidrule(r){3-6}
{Model} & {\textbf{FedFaiREE}} & {$\alpha$} & {$\overline{ACC}$} & {$\overline{|DEOO|}$} & {$|DEOO|_{95}$} \\
\midrule
\multirow{2}{*}{\textbf{FedAvg}} & {\ding{55}} & {/} & {0.844 (0.004)} & {0.127 (0.032)} & {0.184} \\
 & {\ding{51}} & {0.10} & {0.843 (0.003)} & {\textbf{0.029} (0.027)} & {\textbf{0.091}} \\
\midrule
\multirow{2}{*}{\textbf{FedFB}} & {\ding{55}} & {/} & {0.845 (0.003)} & {0.057 (0.034)} & {0.117} \\
 & {\ding{51}} & {0.10} & {0.845 (0.003)} & {\textbf{0.036} (0.025)} & {\textbf{0.083}}  \\
\midrule
\multirow{2}{*}{\textbf{FairFed}} & {\ding{55}} & {/} & {0.839 (0.004)} & {0.081 (0.033)} & {0.138} \\
 & {\ding{51}} & {0.10} & {0.838 (0.004)} & {\textbf{0.027} (0.025)} & {\textbf{0.073}} \\
\midrule
\end{tabular}}
\end{table*}
Table \ref{tab:result3} shows the result on adult with parameter for Dirichlet distribution=10. Moreover, we present an analysis of the impact of parameter variations on the experimental results. We consider two parameters——the fairness constraint, $\alpha$, and the confidence coefficient, $\beta$, separately. Figure \ref{fig:abcd1} and \ref{fig:abcd2} shows the result on Adult dataset and Compas dataset, respectively.

% \begin{figure}[h]
%     \begin{minipage}[b]{0.4\textwidth}
%         \centering
%         \includegraphics[width=\textwidth]{Draft/figure/acc_alpha.png}
%         \caption{\textbf{The change of accuracy with respect to alpha on Compas.} The other parameters of the experiment are consistent with those in Table \ref{tab:result}.}
%         \label{fig:a}
%     \end{minipage}
%     \hfill
%     \begin{minipage}[b]{0.4\textwidth}
%         \centering
%         \includegraphics[width=\textwidth]{Draft/figure/deoo_alpha.png}
%         \caption{\textbf{The change of $\overline{|DEOO|}$ and $|DEOO|_{95}$ with respect to alpha on Compas.} The other parameters of the experiment are consistent with those in Table \ref{tab:result}.}
%         \label{fig:b}
%     \end{minipage}
% \end{figure}

% \begin{figure}[h]
%     \begin{minipage}[b]{0.4\textwidth}
%         \centering
%         \includegraphics[width=\textwidth]{Draft/figure/acc_beta.png}
%         \caption{\textbf{The change of accuracy with respect to beta on Compas.} The other parameters of the experiment are consistent with those in Table \ref{tab:result}.}
%         \label{fig:a}
%     \end{minipage}
%     \hfill
%     \begin{minipage}[b]{0.4\textwidth}
%         \centering
%         \includegraphics[width=\textwidth]{Draft/figure/deoo_beta.png}
%         \caption{\textbf{The change of $\overline{|DEOO|}$ and $|DEOO|_{95}$ with respect to beta on Compas.} The other parameters of the experiment are consistent with those in Table \ref{tab:result}.}
%         \label{fig:b}
%     \end{minipage}
% \end{figure}

\begin{figure*}[htbp]
    \begin{minipage}[b]{0.48\textwidth}
        \centering
        \includegraphics[width=\textwidth]{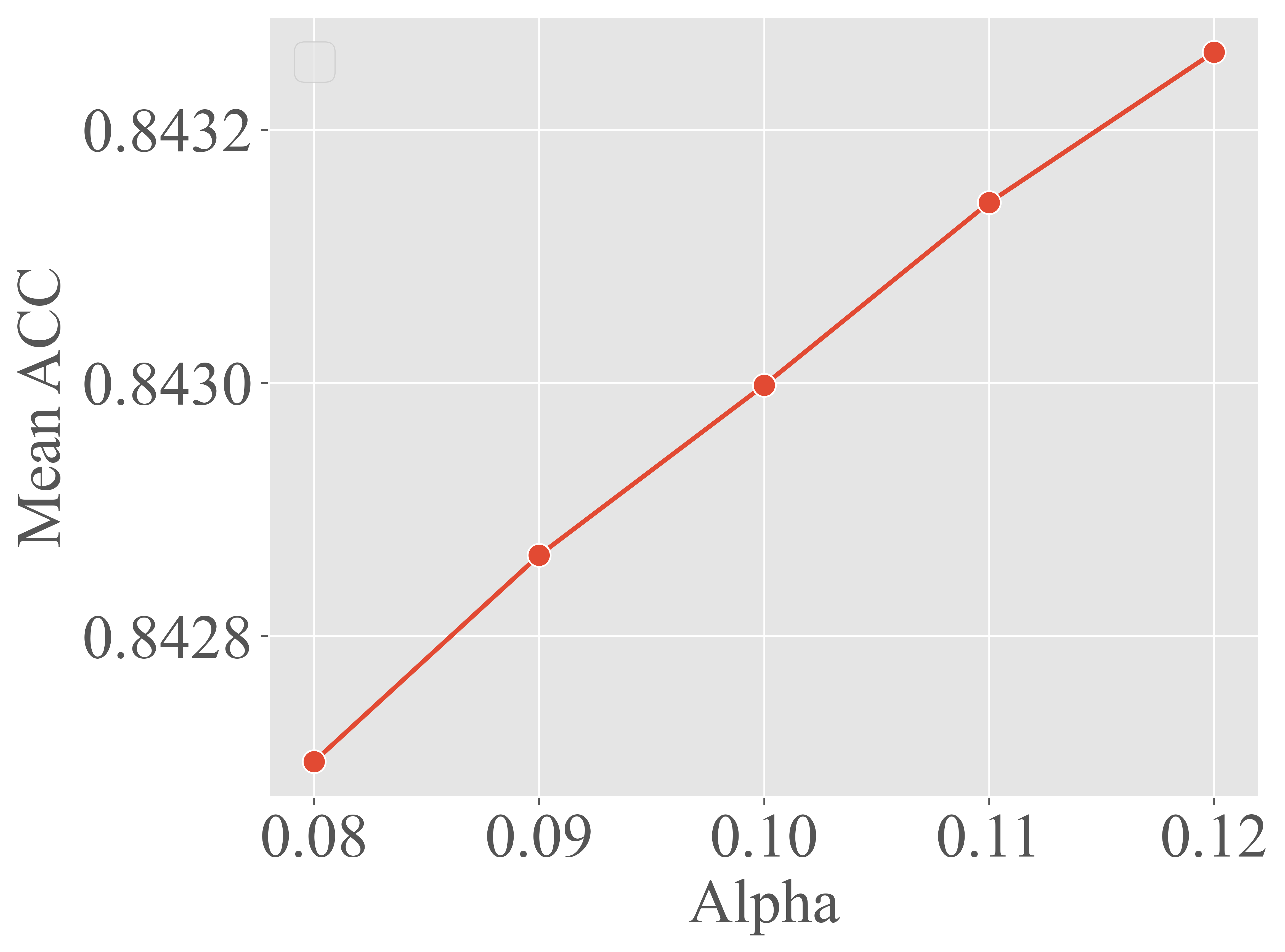}
        \caption*{\textbf{(a)}}
    \end{minipage}
    \hfill
    \begin{minipage}[b]{0.48\textwidth}
        \centering
        \includegraphics[width=\textwidth]{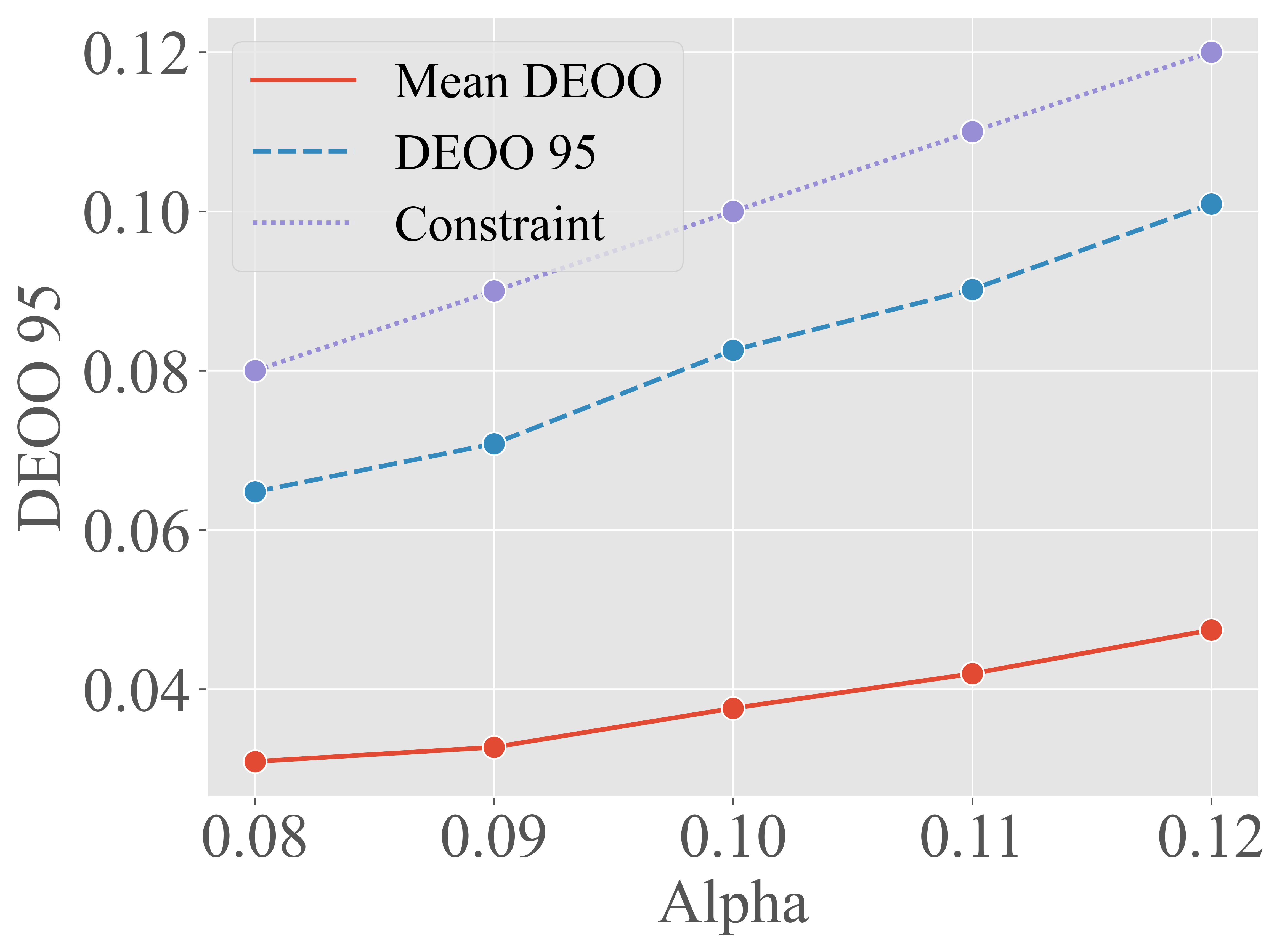}
        \caption*{\textbf{(b)}}
    \end{minipage}

    %\vspace{0.3cm} % Adjust vertical space between rows

    \begin{minipage}[b]{0.48\textwidth}
        \centering
        \includegraphics[width=\textwidth]{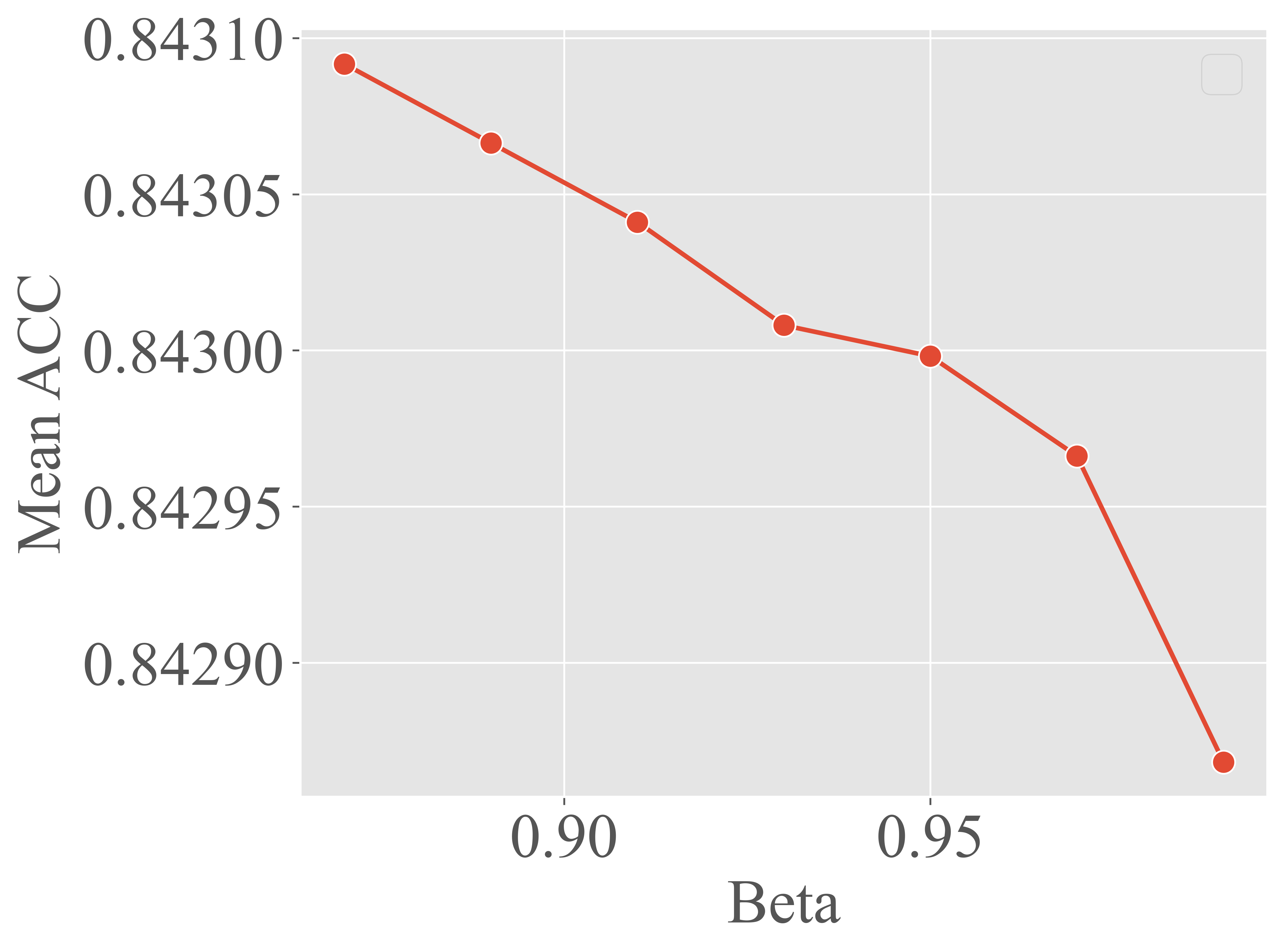}
        \caption*{\textbf{(c)}}
    \end{minipage}
    \hfill
    \begin{minipage}[b]{0.48\textwidth}
        \centering
        \includegraphics[width=\textwidth]{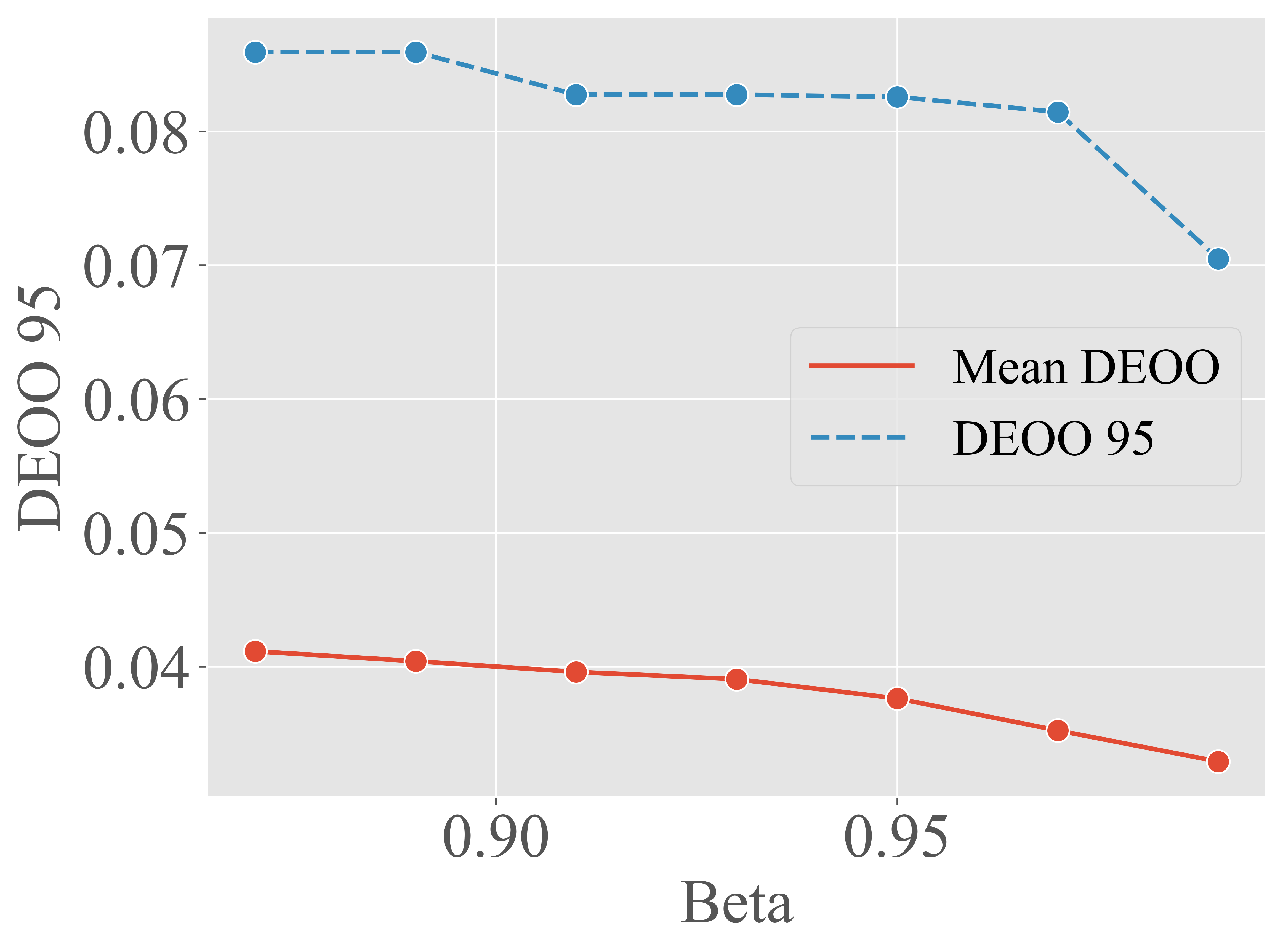}
        \caption*{\textbf{(d)}}
    \end{minipage}
    
    \caption{\textbf{The changes of accuracy, $\overline{|DEOO|}$ and $|DEOO|_{95}$ with respect to $\alpha$ and $\beta$ on Adult.} The other parameters of the experiment are consistent with those in Table \ref{tab:result}.}
    \label{fig:abcd1}
\end{figure*}

\begin{figure*}[htbp]
    \begin{minipage}[b]{0.48\textwidth}
        \centering
        \includegraphics[width=\textwidth]{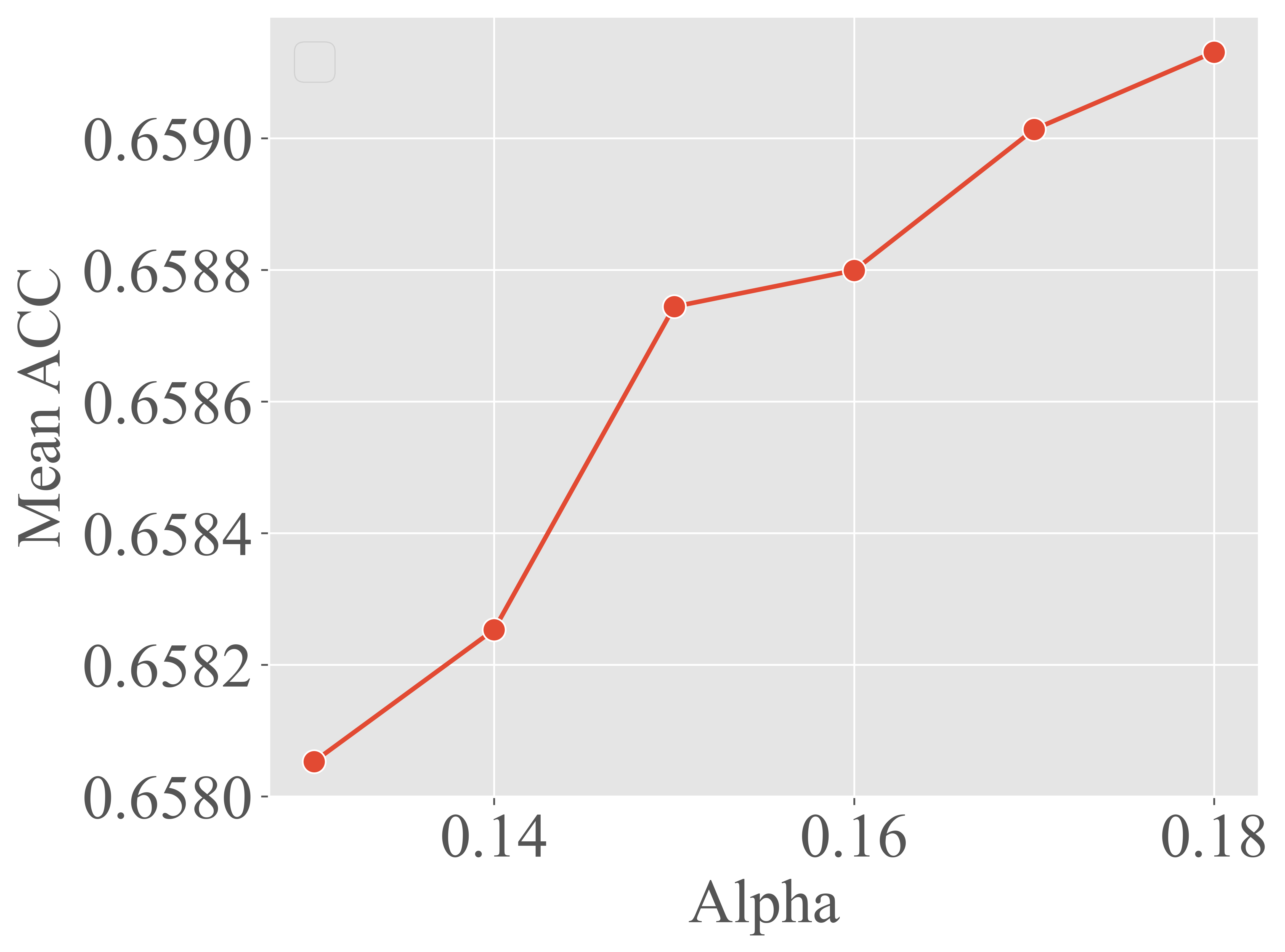}
        \caption*{\textbf{(a)}}
    \end{minipage}
    \hfill
    \begin{minipage}[b]{0.48\textwidth}
        \centering
        \includegraphics[width=\textwidth]{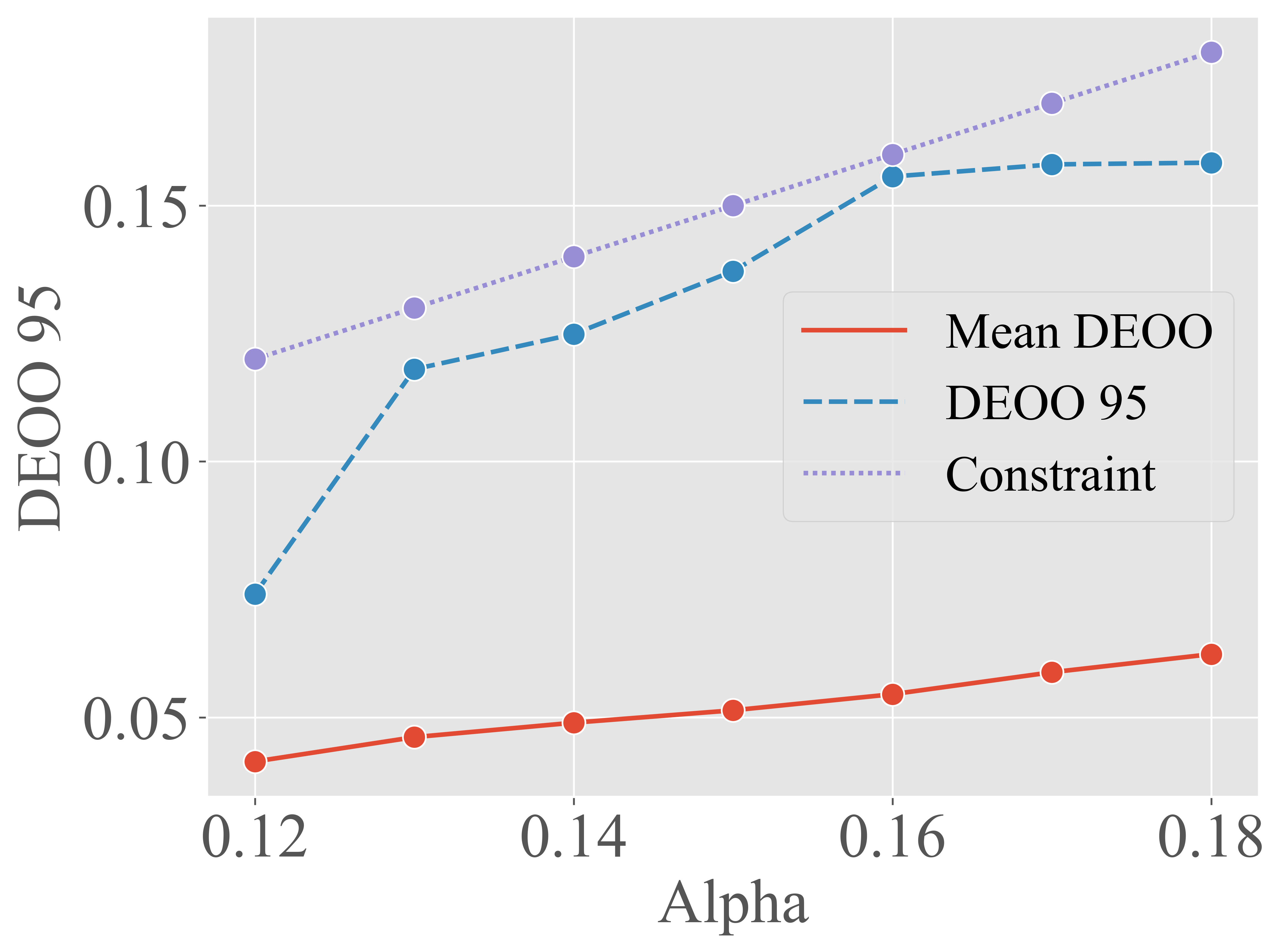}
        \caption*{\textbf{(b)}}
    \end{minipage}

    %\vspace{0.3cm} % Adjust vertical space between rows

    \begin{minipage}[b]{0.48\textwidth}
        \centering
        \includegraphics[width=\textwidth]{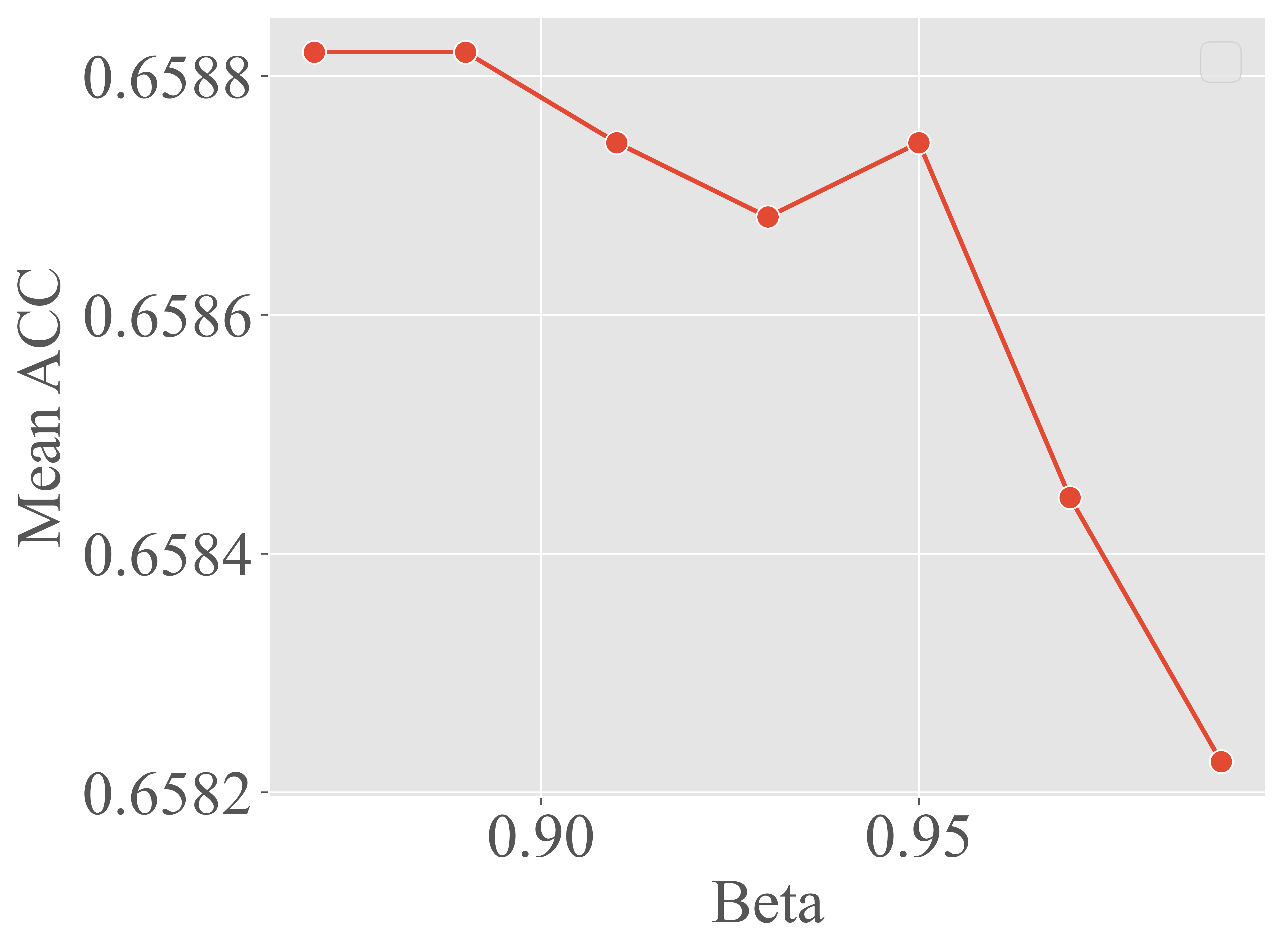}
        \caption*{\textbf{(c)}}
    \end{minipage}
    \hfill
    \begin{minipage}[b]{0.48\textwidth}
        \centering
        \includegraphics[width=\textwidth]{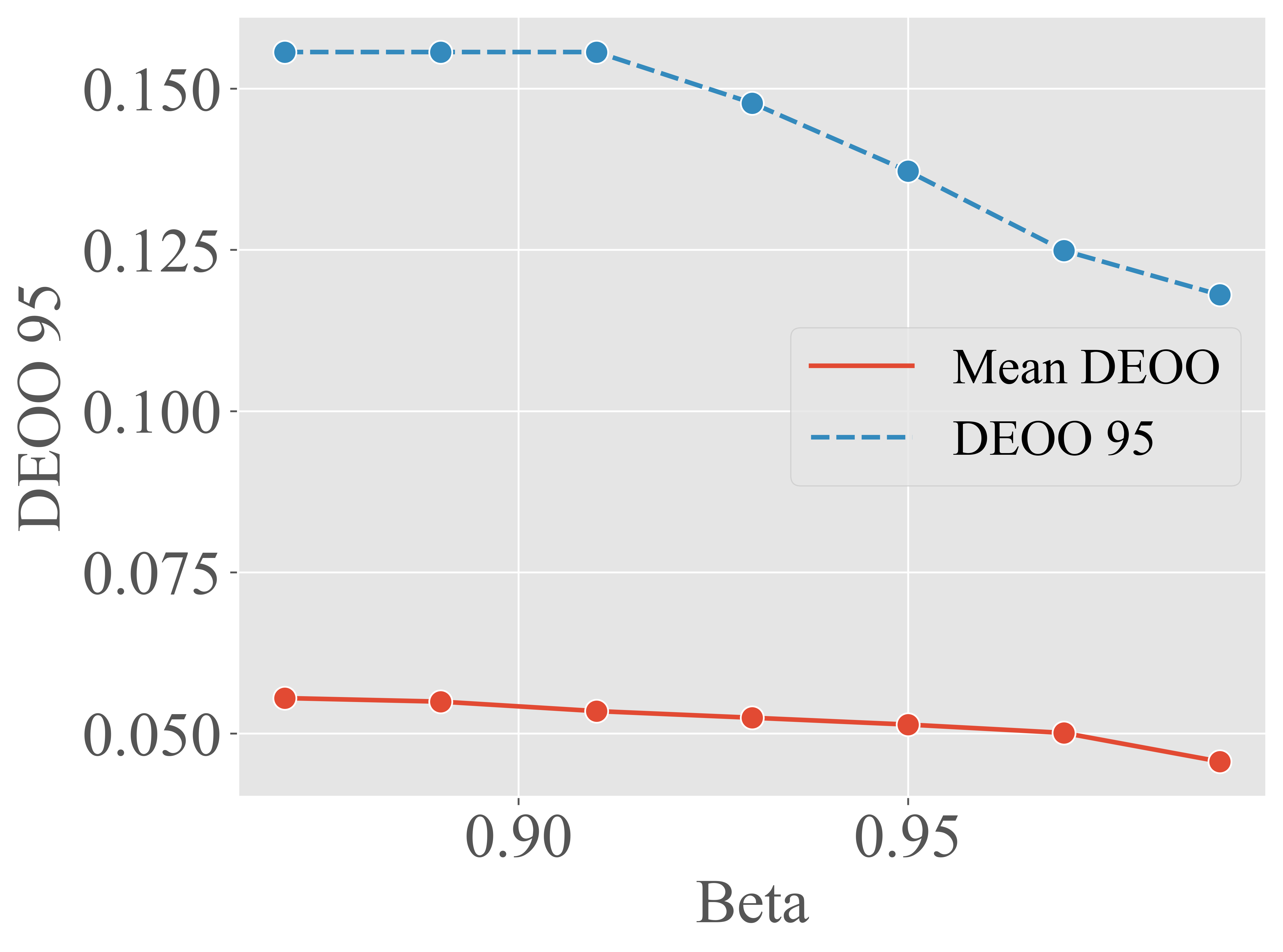}
        \caption*{\textbf{(d)}}
    \end{minipage}
    
    \caption{\textbf{The changes of accuracy, $\overline{|DEOO|}$ and $|DEOO|_{95}$ with respect to $\alpha$ and $\beta$ on Compas.} The other parameters of the experiment are consistent with those in Table \ref{tab:result}.}
    \label{fig:abcd2}
\end{figure*}

\subsection{Further results on DEO}
In this subsection, we conducted experiments using FedFaiREE for DEO, which is a specific algorithm under the FedFaiREE framework designed for DEO as mentioned in Section \ref{sec:deo}. The results are presented in Tables \ref{tab:result_deo1} and \ref{tab:result_deo2}. It's worth noting that FedFaiREE for DEO exhibited favorable performance similar to FedFaiREE for DEOO, showing significant improvements in both DEOO and DPE indicators while maintaining relatively high accuracy.
\begin{table*}[htbp]
\centering
\small
\caption{\textbf{Results of FedFaiREE for DEO on Adult dataset. }We conducted 100 experimental repetitions for each model on both datasets and compared the accuracy and fairness indicators of different models. The ``FedFaiREE" and ``$\alpha$" columns indicate whether FedFaiREE was used or not.``$\overline{ACC}$", ``$\overline{|DEOO|}$" and ``$\overline{|DPE|}$" represent the averages of accuracy, DEOO (defined in Equation \ref{deoo}) and DPE (defined in Equation \ref{dpe}), respectively. ``$|DEOO|_{95}$" and ``$|DPE|_{95}$" represent the 95\% quantile of DEOO and DPE since we set the confidence level of FedFaiREE to 95\% in our experiments. }\label{tab:result_deo1}
\begin{tabular}{cccccccccccccc}
\toprule
& & \multicolumn{6}{c}{\textbf{Adult}} \\
\cmidrule(r){3-8} 
{Model} & {FedFaiREE} & {$\alpha$} & {$\overline{ACC}$} & {$\overline{|DEOO|}$} & {$|DEOO|_{95}$} & {$\overline{|DPE|}$} & {$|DPE|_{95}$}\\
\midrule
\multirow{2}{*}{\textbf{FedAvg}} & {No} & {/} & {0.844 (0.003)} & {0.131 (0.030)} & {0.178} & {0.088 (0.005)} & {0.097} \\
 & {Yes} & {0.10} & {0.843 (0.003)} & {\textbf{0.037} (0.025)} & {\textbf{0.082}} & {\textbf{0.064} (0.007)} & {\textbf{0.075}}\\
\midrule
\multirow{2}{*}{\textbf{FedFB}} & {No} & {/} & {0.850 (0.003)} & {0.057 (0.034)} & {0.117} & {0.066 (0.007)} & {0.077} \\
 & {Yes} & {0.10} & {0.850 (0.003)} & {\textbf{0.036} (0.025)} & {\textbf{0.083}}  & {\textbf{0.061} (0.006)} & {\textbf{0.070}}\\
\midrule
\multirow{2}{*}{\textbf{FairFed}} & {No} & {/} & {0.842 (0.003)} & {0.069 (0.034)} & {0.118} & {0.072 (0.006)} & {0.083}\\
 & {Yes} & {0.10} & {0.841 (0.003)} & {\textbf{0.037} (0.026)} & {\textbf{0.081}}  & {\textbf{0.063} (0.006)} & {\textbf{0.071}} \\
\midrule
\end{tabular}
\end{table*}

\begin{table*}[htbp]
\centering
\small
\caption{\textbf{Results of FedFaiREE for DEO on Compas dataset. }}\label{tab:result_deo2}
\begin{tabular}{cccccccccccccc}
\toprule
& & \multicolumn{6}{c}{\textbf{Compas}} \\
\cmidrule(r){3-8} 
{Model} & {FedFaiREE} & {$\alpha$} & {$\overline{ACC}$} & {$\overline{|DEOO|}$} & {$|DEOO|_{95}$} & {$\overline{|DPE|}$} & {$|DPE|_{95}$}\\
\midrule
\multirow{2}{*}{\textbf{FedAvg}} & {\ding{55}} & {/} & {0.662 (0.011)} & {0.126 (0.056)} & {0.223} & {0.083 (0.032)} & {0.136} \\
 & {\ding{51}} & {0.15} & {0.652 (0.036)} & {\textbf{0.049} (0.045)} & {\textbf{0.137}} & {\textbf{0.028} (0.024)} & {\textbf{0.072}}\\
\midrule
\multirow{2}{*}{\textbf{FedFB}} & {\ding{55}} & {/} & {0.642 (0.011)} & {0.107 (0.043)} & {0.174} & {0.066 (0.028)} & {0.112} \\
 & {\ding{51}} & {0.15} & {0.642 (0.010)} & {\textbf{0.062} (0.040)} & {\textbf{0.125}}  & {\textbf{0.036} (0.024)} & {\textbf{0.081}}\\
\midrule
\multirow{2}{*}{\textbf{FairFed}} & {\ding{55}} & {/} & {0.648 (0.011)} & {0.097 (0.047)} & {0.166} & {0.087 (0.036)} & {0.148}\\
 & {\ding{51}} & {0.15} & {0.642 (0.029)} & {\textbf{0.047} (0.036)} & {\textbf{0.114}}  & {\textbf{0.037} (0.028)} & {\textbf{0.085}} \\
\midrule
\end{tabular}
\end{table*}

\section{Comparison to FaiREE \citep{fairee} and other related works}\label{sec:compare}
Regarding the differences between FedFaiREE and FaiREE, several pivotal distinctions become evident. Primarily, FedFaiREE demonstrates superior adaptability for practical applications. Notably, it incorporates mechanisms to handle label shift scenarios, ensuring model robustness within such distributions, as elucidated in Section 5.1. Furthermore, it's worth noting that FedFaiREE extends considerations to encompass multiple sensitive groups and multiple labels, aligning more closely with practical real-world application scenarios, as discussed in Appendix D.

Another critical difference lies in the setting: FaiREE operates in a centralized environment, assuming homogeneous data across all clients. In contrast, FedFaiREE is expressly tailored for decentralized settings, acknowledging client heterogeneity and effectively addressing the challenges stemming from diverse data distributions and sizes across clients. This tailored approach significantly enhances its adaptability and robustness across various scenarios.

Lastly, while FaiREE relies on specific centralized quantile estimation methods, FedFaiREE adopts approximate quantiles. This adaptation not only facilitates adaptation to distributed data but also fortifies the method's robustness and adaptability.

\subsection{Comparison to other related works}
Differences between FedFaiREE and other fair federated learning methods lie in their approach to addressing fairness concerns. Many methods, akin to this paper, extend the principles of centralized machine learning to decentralized settings, such as FedFB\citep{FedFB}, FedMinMax\citep{papadaki2022minimax}, PFFL\citep{hu2022fair}, and others. These methods primarily focus on introducing fairness penalties in the objective functions and incorporate client reweighting schemes and terms (in objective functions) reweighting schemes that consider global or local fairness. The key divergence between our approach and these methods is that the latter typically converge and provide fairness guarantees only in large-sample scenarios, lacking assurances for fairness in small-sample situations, especially under distribution-free assumptions. Empirical results from Table 1 in this paper demonstrate that compared to FedFaiREE, methods like FedFB, FairFed are not as effective in controlling fairness in small-sample scenarios. Furthermore, as these methods are predominantly in-processing techniques, while FedFaiREE falls under post-processing methods, there is a potential for further integration to achieve improved fairness guarantees as shown in our experiments. Moreover, another significant characteristic of FedFaiREE is its capability to adjust the trade-off between fairness and accuracy according to specific fairness constraints. This control capacity has been demonstrated in numerous experiments, showcasing an ability that other methods lack.
%%%%%%%%%%%%%%%%%%%%%%%%%%%%%%%%%%%%%%%%%%%%%%%%%%%%%%%%%%%%%%%%%%%%%%%%%%%%%%%
%%%%%%%%%%%%%%%%%%%%%%%%%%%%%%%%%%%%%%%%%%%%%%%%%%%%%%%%%%%%%%%%%%%%%%%%%%%%%%%

\end{document}